\documentclass[11pt]{article}

\usepackage[utf8]{inputenc} % allow utf-8 input
\usepackage[T1]{fontenc}    % use 8-bit T1 fonts
\usepackage{hyperref}       % hyperlinks
\usepackage{url}            % simple URL typesetting
\usepackage{booktabs}       % professional-quality tables
\usepackage{amsfonts}       % blackboard math symbols
\usepackage{nicefrac}       % compact symbols for 1/2, etc.
\usepackage{tablefootnote}  % for table footnotes

\usepackage[x11names]{xcolor}
\usepackage{multirow}
\usepackage{adjustbox}

\usepackage{makecell}
\usepackage{colortbl}

\usepackage{fullpage}
\usepackage{wrapfig}

\usepackage[protrusion=true,expansion=true]{microtype}

\usepackage{algorithm}
\usepackage{algorithmic}
\usepackage{amsthm} 
\usepackage{amsmath}
\usepackage{amssymb}
\usepackage{float}
\usepackage{mathtools}
\usepackage{url}
\usepackage{graphicx,times}
\usepackage{fancyheadings}
\usepackage{subfigure}
\usepackage{bm}

\usepackage{threeparttable}
\DeclareMathOperator*{\argmin}{arg\,min}

\newcommand{\norm}[1]{\left\lVert#1\right\rVert}

\usepackage{pifont}% 
\newcommand{\cmark}{\ding{51}}%
\newcommand{\xmark}{\ding{55}}%

\allowdisplaybreaks

\newtheorem{theorem}{Theorem}
\newtheorem{proposition}{Proposition} 
\newtheorem{lemma}{Lemma} 
\newtheorem{assumption}{Assumption} 
 
\newtheorem{definition}{Definition}

\usepackage{cleveref}

\title{SimFBO: Towards  
Simple, Flexible and Communication-efficient Federated Bilevel Learning}

\author{%
  Yifan Yang, Peiyao Xiao and Kaiyi Ji
  \\
  Department of Computer Science and Engineering\\
  University at Buffalo\\
  \texttt{\{yyang99, peiyaoxi, kaiyiji\}@buffalo.edu} \\
}
\date{May 30, 2023}

\begin{document}

\maketitle

\vspace{0.3cm}
% \noindent\\

\begin{abstract}
\noindent
Federated bilevel optimization (FBO) has shown great potential recently in machine learning and edge computing due to the emerging nested optimization structure in meta-learning, fine-tuning, hyperparameter tuning, etc. However, existing FBO algorithms often involve complicated computations and require multiple sub-loops per iteration, each of which contains a number of communication rounds. In this paper, we propose a simple and flexible FBO framework named SimFBO, which is easy to implement without sub-loops, and includes a generalized server-side aggregation and update for improving communication efficiency. We further propose System-level heterogeneity robust FBO (ShroFBO) as a variant of SimFBO with stronger resilience to heterogeneous local computation. We show that SimFBO and ShroFBO provably achieve a linear convergence speedup with partial client participation and client sampling without replacement, as well as improved sample and communication complexities. Experiments demonstrate the effectiveness of the proposed methods over existing FBO algorithms. 
\end{abstract}

% ~\citep{franceschi2018bilevel,bertinetto2018meta,rajeswaran2019meta}

% \red{1. Bilevel optimization. }

\section{Introduction}
Recent years have witnessed significant progress in a variety of emerging areas 
including meta-learning and fine-tuning~\cite{finn2017model, rajeswaran2019meta}, automated hyperparameter optimization~\cite{franceschi2018bilevel, feurer2019hyperparameter}, reinforcement learning~\cite{konda1999actor, hong2020two}, fair batch selection in machine learning~\cite{roh2021fairbatch}, adversarial learning~\cite{zhang2022revisiting, liu2021investigating}, AI-aware communication networks~\cite{ji2023network}, fairness-aware federated learning~\cite{zeng2021improving}, etc. These problems share a common nested optimization structure, and have inspired intensive study on the theory and algorithmic development of bilevel optimization. Prior efforts have been taken mainly on the single-machine scenario. However, in modern machine learning applications, data privacy has emerged as a critical concern in centralized training, and the data often exhibit an inherently distributed nature~\cite{xing2016strategies}. This highlights the importance of recent research and attention on federated bilevel optimization, and has inspired many emerging applications including but not limited to federated meta-learning~\cite{fallah2020personalized}, hyperparameter tuning for federated learning~\cite{huang2022federated}, resource allocation over communication networks~\cite{ji2023network} and graph-aided federated learning~\cite{xing2022big}, adversarial robustness on edge computing~\cite{manoharan2022svm}, etc. In general, the federated bilevel optimization problem takes the following mathematical formulation.    
\begin{align}\label{eq:intro}
&\min_{x\in\mathbb{R}^{p}} \Phi(x)=F\big(x, y^*(x)\big) : =\sum_{i=1}^{n} p_if_i(x,y^*(x)) = \sum_{i=1}^{n} p_i\mathbb{E}_{\xi}\Big[f_i\big(x,y^*(x);\xi_i\big)\Big] \nonumber \\ 
&\;\;\mbox{s.t.} \; y^*(x)= \argmin_{y\in\mathbb{R}^q} G(x, y) : =\sum_{i=1}^{n} p_ig_i(x,y) = \sum_{i=1}^{n} p_i\mathbb{E}_{\zeta}\big[g_i(x,y;\zeta_i)\big]
% &\min_{x\in\mathbb{R}^{p}} \Phi(x)=f(x, y^*(x)) : =
% \mathbb{E}_{\xi}  p_if(x,y^*(x);\xi)  \nonumber \\
% &\;\;\mbox{s.t.} \; y^*(x)= \argmin_{y\in\mathbb{R}^q} g(x,y):=\mathbb{E}_{\zeta} \left[g(x,y^*(x);\zeta)\right] 
\end{align}
where $n$ is the total number of clients, the outer- and inner-functions $f_i(x,y)$ and $g_i(x,y)$ for each client $i$ take the expectation forms w.r.t.~the random variables $\xi_i$ and $\zeta_i$, and are jointly continuously differentiable. However, efficiently solving the federated problem in \cref{eq:intro} suffers from several main challenges posed by the federated hypergradient (i.e., $\nabla\Phi(x)$) computation that contains the second-order global Hessian-inverse matrix, the lower- and upper-level data and system-level heterogeneity, and the nested optimization structure. To address these issues, \cite{huang2022federated, tarzanagh2022fednest,huang2023achieving,huang2022fast} proposed approximate implicit differentiation (AID)-based federated bilevel algorithms, which applied the idea of non-federated AID-based estimate in \cite{ghadimi2018approximation} to the federated setting, and involve two sub-loops for estimating the global lower-level solution $y^*(x)$ and the Hessian-inverse-vector product, respectively. \cite{xiao2023communication} then 
proposed AggITD by leveraging the idea of iterative differentiation, which improved the communication efficiency of AID-based approaches by synthesizing the lower-level optimization and the hypergradient computation into the same communication sub-loop. However, some limitations still remain in these approaches. 
\begin{list}{$\bullet$}{\topsep=0.2ex \leftmargin=0.2in \rightmargin=0.in \itemsep =0.01in}
\item First, the sub-loops, each with a large number of communication rounds, often compute products of a series of matrix-vector products, and hence can complicate the implementation and increase the communication cost. 
\item Second, the practical client sampling {\bf without} replacement has not been studied in these methods due to challenges posed by the nested structure of AID- and ITD-based federated hypergradient estimators. 
\item Third, as observed in the single-level federated learning~\cite{wang2020tackling}, in the presence of heterogeneous system capabilities such as diverse computing power and storage, clients can take a variable number of local updates or use different local optimizers, which may make these FBO algorithms converge to the stationary point of a different objective. 
\end{list}

% In particular, it has been shown~\cite{tarzanagh2022fednest} that the federated hypergradient $\nabla\Phi(x)$ contains produc 

% The biggest challenge to solving the nested problem in \cref{eq:intro} is that the gradient of outer-objective function i.e. hypergradient $\nabla \Phi(x)$ is that global Hessian inverse matrices are included, which also suffers client heterogeneity like the data heterogeneity~\cite{karimireddy2020scaffold, hsu2019measuring}. 
% To address these issues, prior methods focused on AID-based federated hypergradient estimation~\cite{huang2022federated, tarzanagh2022fednest} and ITD-based federated hypergradient estimation~\cite{xiao2023communication}. However, unlike traditional BO, separate loops at each outer iteration will result in a vast number of communication rounds to minimize the inner-objective and construct the federated hypergradient estimate in FBO. As shown in \Cref{fig:threecompare}, sub-loops exacerbate the implementation complexity and incur higher communication costs. 
\begin{figure}[t]
        \vspace{-2mm}
	\centering    
	{\label{fig1:threecompare}\includegraphics[width=159mm]{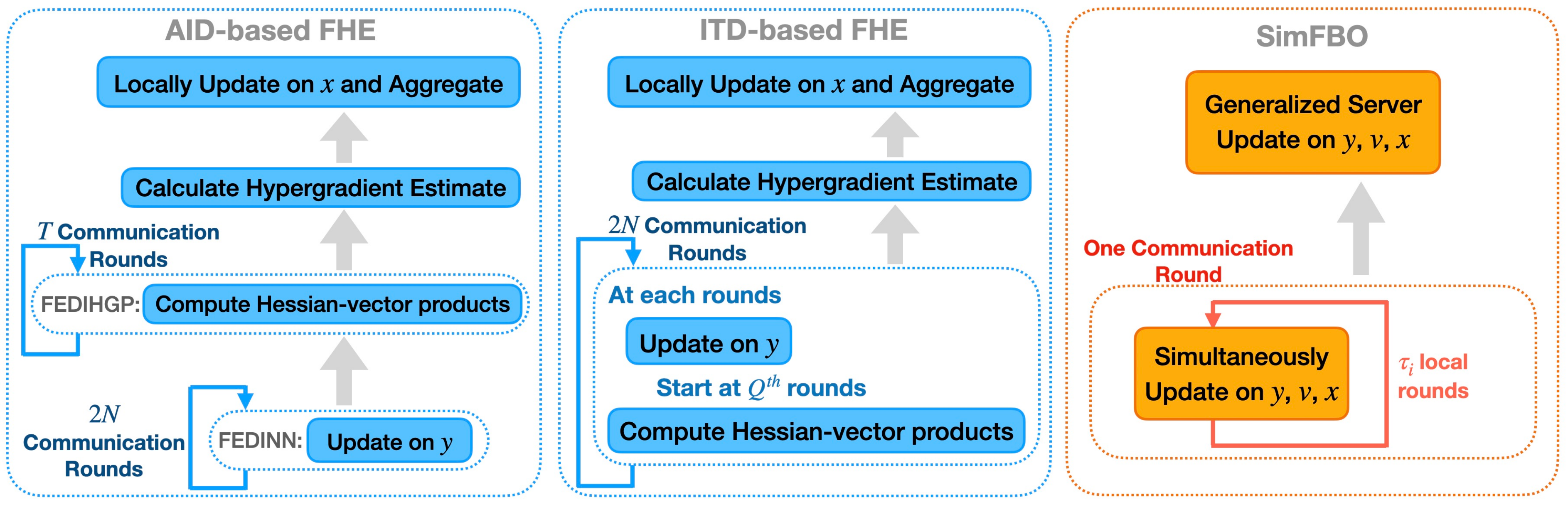}}
         \vspace{-0.2cm}
	\caption{Comparison of AID-based federated hypergradient estimation (FHE) in FedNest~\cite{tarzanagh2022fednest} (left), ITD-based FHE in AggITD~\cite{xiao2023communication} (middle) and our proposed SimFBO (right) at each iteration.}\label{fig:threecompare}
	  \vspace{-0.3cm}
\end{figure}

\subsection{Our Contributions}
In this paper, we propose a communication-efficient federated bilevel method named SimFBO, which is simple to implement without sub-loops, flexible with a generalized server-side update, and resilient to system-level heterogeneity. Our specific contributions are summarized below.   
% a simple and flexible framework without sub-loops, which refer to as SimFBO. As shown in~\Cref{fig:threecompare}, our SimFBO introduces a new variable $v$ for hypergradient estimation and updates all variables simultaneously. Our framework admits a simpler implementation and much fewer communication rounds. Moreover, we proposed a robust version named ShroFBO to resist system heterogeneity. 

\noindent
{\bf A simple and flexible implementation.} As illustrated in \Cref{fig:threecompare}, differently from AID- and ITD-based approaches that contain multiple sub-loops of communication rounds at each iteration, our proposed SimFBO is simpler with a single communication round per iteration, in which three variables $y,x$ and $v$ are updated simultaneously for optimizing the lower- and upper-level problems, and approximating the Hessian-inverse-vector product. SimFBO also includes a generalized server-side update on $x,y,v$, which accommodates the client sampling without replacement, and allows for a flexible aggregation to improve the communication efficiency.

% In existing works, AID-based bilevel approaches~\cite{franceschi2018bilevel, grazzi2020iteration, ji2021bilevel} required two loops for the hypergradient estimation relying on inner updates of $y$, which evolves to two giant communication rounds for the federated hypergradient estimation~\cite{tarzanagh2022fednest}. ITD-based approaches~\cite{ji2021bilevel} reduced implementation but the remaining sub-loop is still communication costing for federated hypergradient estimation~\cite{xiao2023communication}.
% In contrast, we propose a new iteration and update process which avoids separate loops by solving a linear system, which allows local updates of all parameters update at the same time, and only one communication round is required for each global iteration. 
\noindent
{\bf Resilient server-side updates to system-level heterogeneity.}
In the presence of heterogeneous local computation, we show that the naive server-side aggregation can lead to the convergence to a stationary point of a different objective. To this end, we propose System-level heterogeneity robust FBO (ShroFBO) building on a normalized version of the generalized server-side update with correction, which provably converges to a stationary point of the original objective.

% The prior works~\cite{wang2020tackling} provided insight that heterogeneous local updates result in averaging global updates convergence to a wrong stationary point. Neither FedNest~\cite{tarzanagh2022fednest} nor AggITD~\cite{xiao2023communication} considers this situation. In this work, we provide a system-level heterogeneity robust version FBO (ShroFBO) which ensures the global convergence accuracy after heterogeneous local updates via normalized estimators and converges as the same speed as SimFBO.  

\noindent
{\bf Convergence analysis and improved complexity.} As shown in \Cref{tb:FBOs}, our SimFBO and ShroFBO both achieve a sample complexity (i.e., the number of samples needed to reach an $\epsilon$-accurate stationary point) of $\mathcal{O}(\epsilon^{-2}P^{-1})$, which matches the best result obtained by FedMBO~\cite{huang2023achieving} but under a more practical client sampling without replacement. Moreover, SimFBO and ShroFBO both achieve the best communication complexity (i.e., the number of communication rounds to reach an $\epsilon$-accurate stationary point) of $\mathcal{O}(\epsilon^{-1})$, which improves  those of other methods by an order of $\epsilon^{-1/2}$. Technically, we develop novel analysis in characterizing the client drifts by the three variables, bounding the per-iteration progress in the global $y$ and $v$ updates, and proving the smoothness and bounded variance in local $v$ updates via induction, which may be of independent interest.

% We provide a novel error and convergence analysis for the proposed SimFBO estimator and ShroFBO algorithm. The linear speedup was achieved by FedMBO ~\cite{huang2023achieving} in Non-IID FBO. However, FedMBO asked clients to be selected with replacement, which may cause duplicate selection. 
% In this paper, our SimFBO achieves linear speedup without client selection replacement. The communication of SimFBO is \red{$\mathcal{O}(\epsilon^{-1})$}, which efficiently and significantly reduces the implementation time, especially when a large number of clients are involved. 

\noindent
{\bf Superior performance in practice.} In the experiments, the proposed SimFBO method significantly improves over existing strong federated bilevel baselines such as AggITD, FedNest and LFedNest in both the i.i.d.~and non-i.i.d.~settings. We also validate the better performance of ShroFBO in the presence of heterogeneous local computation due to the resilient server-side updates.

% \red{KJ: revision for related work up to here.}

% Some recent researches~\cite{gao2022convergence, li2022local} were proposed in the homogeneous setting, which developed momentum-based optimization methods with fully local hypergradient estimators. 

% In the heterogeneous setting, the AID-based FedNest~\cite{tarzanagh2022fednest} and ITD-based AggITD~\cite{xiao2023communication} provided their convergence analysis. Moreover, some other setups were studied such as decentralized bilevel optimization~\cite{chen2022decentralized, yang2022decentralized, lu2022decentralized}, asynchronous bilevel optimization over directed network~\cite{yousefian2021bilevel} and distributed bilevel network utility maximization~\cite{ji2023network}. 

% \red{there are some others.}

\begin{table*}[t]
\centering
 % \small
 % \begin{threeparttable}
 \renewcommand{\arraystretch}{1.25}
 \scalebox{0.99}{
    \begin{tabular}{l|c|c|c|c|c|c}
    \hline  \hline
{\bf Algorithm} & \thead{System-level \\heterogeneity }& \thead{Partial \\participation} & \thead{Without \\ replacement} &  \thead{Linear \\speedup} & \thead{Samples \\complexity} & \thead{Communication \\ complexity}  
\\ \hline \hline  
    FedNest~\cite{tarzanagh2022fednest} & \xmark & \xmark & \xmark & \xmark & $\mathcal{O}(\epsilon^{-2})$ & $\mathcal{O}(\epsilon^{-2})$ 
    \\ \hline 
    FBO-AggITD~\cite{xiao2023communication} & \xmark & \xmark & \xmark & \xmark & $\mathcal{O}(\epsilon^{-2})$ & $\mathcal{O}(\epsilon^{-2})$ 
     \\ \hline 
    FedBiO~\cite{li2023communication} & \xmark & \xmark & \xmark & \cmark & $\mathcal{O}(\epsilon^{-2.5}n^{-1})$ & $\mathcal{O}(\epsilon^{-1.5})$ 
    \\ \hline 
    % AdaFBiO & \xmark & \xmark& ssss & \xmark & $\mathcal{O}(\epsilon^{-1.5})$ & $\mathcal{O}(\epsilon^{-0.5})$ 
    % \\ \hline
    % LocalBSGVR & \xmark & \xmark & ssss & \cmark (?)& $\mathcal{O}(\epsilon^{-1.5}n^{-1})$ & $\mathcal{O}(\epsilon^{-1.5}n^{-1})$ 
    % \\ \hline
    FedMBO~\cite{huang2023achieving} & \xmark & \cmark & \xmark & \cmark & $\mathcal{O}(\epsilon^{-2}P^{-1})$ &  $\mathcal{O}(\epsilon^{-2})$
    \\ \hline 
    \cellcolor{blue!6}{SimFBO (this paper)} & \cellcolor{blue!6}{\xmark} & \cellcolor{blue!6}{\cmark} & \cellcolor{blue!6}{\cmark}  & \cellcolor{blue!6}{\cmark} & \cellcolor{blue!6}{$\mathcal{O}(\epsilon^{-2}P^{-1})$} & \cellcolor{blue!6}{$\mathcal{O}(\epsilon^{-1})$} 
    \\ \hline 
    \cellcolor{blue!6}{ShroFBO (this paper)} & \cellcolor{blue!6}{\cmark} & \cellcolor{blue!6}{\cmark} & \cellcolor{blue!6}{\cmark} & \cellcolor{blue!6}{\cmark} & \cellcolor{blue!6}{$\mathcal{O}(\epsilon^{-2}P^{-1})$} & \cellcolor{blue!6}{$\mathcal{O}(\epsilon^{-1})$} 
    \\ \hline 
    \end{tabular} }
    % \vspace{-0.2cm}
    \caption{Comparison of different federated bilevel algorithms in the setting with heterogeneous data. We do not include the methods with momentum-based acceleration for a fair comparison. $P\leq n$ is the number of sampled clients in each communication round. }
    % For the last column, we count the Hessian-vector-based procedure for estimating the Hessian-inverse-vector product as an optimization loop. $\mathcal{\widetilde O}(\cdot)$ hides all $\log$ factors. 
     \label{tb:FBOs}
     \vspace{-0.2cm}
\end{table*}

\section{SimFBO: A Simple and Flexible Framework}
\subsection{Preliminary: Federated Hypergradient Computation}
The biggest challenge in FBO is to compute the federated hypergradient $\nabla \Phi(x)$ due to the implicit and complex dependence of $y^*(x)$ on $x$. Under suitable assumptions and using the implicit function theorem in~\cite{griewank2008evaluating}, it has been shown that the $\nabla\Phi(x)$ takes the form of 
% \begin{align}
%     \nabla\Phi(x) &=  \sum_{i=1}^n p_i\nabla_xf_i(x,y^*) - \sum_{i=1}^n p_i\nabla_{xy}^2g_i(x,y^*)\Big[\sum_{i=1}^n p_i\nabla_{yy}^2g_i(x,y^*)\Big]^{-1}\sum_{i=1}^n p_i\nabla_yf_i(x,y^*) \nonumber
% \end{align}
\begin{align}
    \nabla\Phi(x) &=  \sum_{i=1}^n p_i\nabla_xf_i(x,y^*) - \nabla_{xy}^2G(x,y^*)\big[\nabla_{yy}^2G(x,y^*)\big]^{-1}\sum_{i=1}^n p_i\nabla_yf_i(x,y^*) \nonumber
\end{align}
which poses several computational challenges in the  federated setting. First, 
%nonlinear
% This inequality suggests that the federated hypergradient cannot be adequately expressed summations due to the nonlinearity. \neq \sum_{i=1}^n p_i\Big(\nabla_xf_i(x,y^*) - \nabla_{xy}^2g_i(x,y^*)\big[\nabla_{yy}^2g_i(x,y^*)\big]^{-1}\nabla_yf_i(x,y^*)\Big).
the second term at the right-hand side contains three global components in a nonlinear manner, and hence the direct aggregation of local hypergradients given by  
$$\sum_{i=1}^n p_i \big(\nabla_xf_i(x,y^*) - \nabla_{xy}^2g_i(x,y^*)\big[\nabla_{yy}^2g_i(x,y^*)\big]^{-1}\nabla_yf_i(x,y^*)\big )$$
is a biased estimation of $\nabla\Phi(x)$ due to the client drift. 
Second, it is infeasible to compute, store and communicate the second-order Hessian-inverse and Jacobian matrices due to the limited computing and communication resource. Although various AID- and ITD-based approaches have been proposed to address these challenges, they still suffer from several limitations (as we point out in the introduction) such as complicated implementation, high communication cost, lack of client sampling without replacement, and vulnerability to the system-level heterogeneity. To this end, we propose a simple, flexible and communication-efficient FBO framework named SimFBO in this section.  

\subsection{Federated Hypergradient Surrogate}
To estimate the federated hypergradient efficiently, we use the surrogate 
% to approximate the global Hessian-inverse-vector product:
% \begin{align*}
   $\bar{\nabla}F(x,y,v) = \nabla_x F(x,y) - \nabla_{xy}^2G(x,y)v$, 
% \end{align*}
where $v \in R^{d_y}$ is an auxiliary vector. Then, it suffices to find $y$ and $v$ as efficient estimates of the solutions to the global lower-level problem and the global linear system (LS) {\small$\nabla_{yy}^2G(x,y)v=\nabla_yF(x,y)$} that is equivalent to solving  
following quadratic programming.
\begin{align}\label{def:R}
  \min_v \;   R(x,y,v) &= \frac{1}{2}v^T\nabla_{yy}^2G(x,y)v - v^T\nabla_yF(x,y) \nonumber\\
    &= \sum_{i=1}^np_i\big (\underbrace{\frac{1}{2}v^T  \nabla_{yy}^2 g_i(x,y)v - v^T\nabla_y  f_i(x,y)}_{R_i(x,y,v)}\big),
\end{align}
where $R_i(x,y,v)$ can be regarded as the loss function of client $i$ for solving this global LS problem. Based on this surrogate, we next describe the proposed SimFBO framework.

\subsection{Simple Local and Server-side Aggregations and Updates}
% \subsubsection{Local Update}

% {\color{red}KJ: remember to mention sampling without replacements. }

\textbf{Simple local update.} Differently from FedNest~\cite{tarzanagh2022fednest} and AggITD~\cite{xiao2023communication} that perform the lower-level optimization, the federated hypergradient estimation and the upper-level update alternatively in different communication sub-loops, our SimFBO conducts the simple updates on all these three procedures simultaneously in each communication round. In specific, each communication round $t$ first selects a subset $C^{(t)}$ of participating clients without replacement. Then, each active client $i\in C^{(t)}$ updates three variables $y,v,x$ at $k^{th}$ local iteration {\bf simultaneously} as 
% \begin{align} \label{eq:localupdate}
%     y_i^{(t,k+1)} &= y_i^{(t,k)} - \eta_y^{(t)}a_i^{(t,k)} \nabla_y g_i\big(x_i^{(t,k)},y_i^{(t,k)}; \zeta_i^{(t,k)}\big) \nonumber\\
%    v_i^{(t,k+1)} &= v_i^{(t,k)} - \eta_v^{(t)} a_i^{(t,k)} \nabla_v R_i\big(x_i^{(t,k)},y_i^{(t,k)},v_i^{(t,k)};\psi_i^{(t,k)}\big) \nonumber\\
%     x_i^{(t,k+1)} &= x_i^{(t,k)} - \eta_x^{(t)} a_i^{(t,k)} \bar{\nabla}f_i\big(x_i^{(t,k)},y_i^{(t,k)},v_i^{(t,k)};\xi_i^{(t,k)}\big),
% \end{align}
% \begin{align}\label{eq:localupdate}
% \begin{pmatrix}
%     y_i^{(t,k+1)} \vspace{0.2cm}\\ 
%     v_i^{(t,k+1)}\vspace{0.2cm}\\
%     x_i^{(t,k+1)}
%     \end{pmatrix} 
%     \leftarrow \begin{pmatrix}
%     y_i^{(t,k)}\vspace{0.2cm}\\ 
%     v_i^{(t,k)}\vspace{0.2cm}\\
%     x_i^{(t,k)}
%     \end{pmatrix}- a_i^{(t,k)}\begin{pmatrix}
%     \eta_y^{(t)} \nabla_y g_i\big(x_i^{(t,k)},y_i^{(t,k)}; \zeta_i^{(t,k)}\big)\vspace{0.2cm}\\ 
%     \eta_v^{(t)} \nabla_v R_i\big(x_i^{(t,k)},y_i^{(t,k)},v_i^{(t,k)};\psi_i^{(t,k)}\big) \vspace{0.2cm}\\
%     \eta_x^{(t)} \bar{\nabla}f_i\big(x_i^{(t,k)},y_i^{(t,k)},v_i^{(t,k)};\xi_i^{(t,k)}\big)
% \end{pmatrix}
% \end{align}
\begin{align}\label{eq:localupdate}
\begin{pmatrix}
    y_i^{(t,k+1)} \vspace{0.2cm}\\ 
    v_i^{(t,k+1)}\vspace{0.2cm}\\
    x_i^{(t,k+1)}
    \end{pmatrix} 
    \leftarrow \begin{pmatrix}
    y_i^{(t,k)}\vspace{0.2cm}\\ 
    v_i^{(t,k)}\vspace{0.2cm}\\
    x_i^{(t,k)}
    \end{pmatrix}- a_i^{(t,k)}\begin{pmatrix}
    \eta_y \nabla_y g_i\big(x_i^{(t,k)},y_i^{(t,k)}; \zeta_i^{(t,k)}\big)\vspace{0.2cm}\\ 
    \eta_v \nabla_v R_i\big(x_i^{(t,k)},y_i^{(t,k)},v_i^{(t,k)};\psi_i^{(t,k)}\big) \vspace{0.2cm}\\
    \eta_x \bar{\nabla}f_i\big(x_i^{(t,k)},y_i^{(t,k)},v_i^{(t,k)};\xi_i^{(t,k)}\big)
\end{pmatrix}
\end{align}
where $\eta_y$, $\eta_v$, $\eta_x$ correspond to the local stepsizes, $a_i^{(t,k)}$ is a client-specific coefficient to increase the flexibility of the framework, 
$\zeta_i^{(t,k)}$, $\psi_i^{(t,k)}$, $\xi_i^{(t,k)}$ are independent samples, 
and the local hypergradient estimate takes the form of 
% \begin{align*}
{\small$\bar{\nabla}f_i\big(x,y,v;\xi\big)= \nabla_x f_i\big(x,y;\xi\big) - \nabla_{xy}^2g_i\big(x,y;\xi\big)v_i.$} 
% \end{align*}
The variables $y,v$ and $x$ in \cref{eq:localupdate}, which optimize the lower-level problem, the LS problem and the upper-level problem, are updated with totally $\tau_i^{(t)}$ local steps. Note that the updates in \cref{eq:localupdate} also allow for parallel computation on $x,v$ and $y$ locally.

\noindent
\textbf{Local and server-side aggregation. } 
After completing all local updates, the next step is to aggregate such local information on both the client and server sides. As shown in \cref{eq:localandserveraggregation1}, 
each participating client $i\in C^{(t)}$ aggregates all the local gradients, and then communicate the aggregations $q_{y,i}^{(t)},q_{v,i}^{(t)}$ and $q_{x,i}^{(t)}$ to the server. Then, on the server side, such local information is further aggregated to be $q_{y}^{(t)},q_{v}^{(t)}$ and $q_{x}^{(t)}$, which will be used for a subsequent generalized server-side update.  
% \begin{align}\label{eq:simfbo_AE}
%     q_{y,i}^{(t)} = \sum_{k=0}^{\tau_i-1}a_i^{(t,k)}\nabla_y g_i\big(x_i^{(t,k)},y_i^{(t,k)}; \zeta_i^{(t,k)}\big),\quad  q_{y}^{(t)} = \sum_{i \in C^{(t)}}\widetilde{p}_iq_{y,i}^{(t)}\nonumber \\
%     q_{v,i}^{(t)} = \sum_{k=0}^{\tau_i-1}a_i^{(t,k)}\nabla_v R_i\big(x_i^{(t,k)},y_i^{(t,k)}v_i^{(t,k)}; \psi_i^{(t,k)}\big), \quad q_{v}^{(t)} = \sum_{i \in C^{(t)}}\widetilde{p}_iq_{v,i}^{(t)}\nonumber \\
%     q_{x,i}^{(t)} = \sum_{k=0}^{\tau_i-1}a_i^{(t,k)}\bar{\nabla} f_i\big(x_i^{(t,k)},y_i^{(t,k)},v_i^{(t,k)}; \xi_i^{(t,k)}\big),\quad q_{x}^{(t)} = \sum_{i \in C^{(t)}}\widetilde{p}_iq_{x,i}^{(t)}.
% \end{align}
% \vspace{-0.1cm}
\begin{align}\label{eq:localandserveraggregation1}
    q_{y}^{(t)} = \sum_{i \in C^{(t)}}\widetilde{p}_iq_{y,i}^{(t)} =& \sum_{i \in C^{(t)}}\widetilde{p}_i\sum_{k=0}^{\tau_i-1}a_i^{(t,k)}\nabla_y g_i\big(x_i^{(t,k)},y_i^{(t,k)}; \zeta_i^{(t,k)}\big),\nonumber \\
    q_{v}^{(t)} = \sum_{i \in C^{(t)}}\widetilde{p}_iq_{v,i}^{(t)} =& \sum_{i \in C^{(t)}}\widetilde{p}_i\sum_{k=0}^{\tau_i-1}a_i^{(t,k)}\nabla_v R_i\big(x_i^{(t,k)},y_i^{(t,k)}v_i^{(t,k)}; \psi_i^{(t,k)}\big),\nonumber \\
    \underbrace{q_{x}^{(t)} = \sum_{i \in C^{(t)}}\widetilde{p}_iq_{x,i}^{(t)}}_{\text{Server aggregation}} =& \sum_{i \in C^{(t)}}\widetilde{p}_i \underbrace{\sum_{k=0}^{\tau_i-1}a_i^{(t,k)}\bar{\nabla} f_i\big(x_i^{(t,k)},y_i^{(t,k)},v_i^{(t,k)}; \xi_i^{(t,k)}\big)}_{\text{Local aggregation}},
\end{align}
% \vspace{-0.2cm}

\noindent
where $\widetilde{p}_i := \frac{n}{|C^{(t)}|}p_i$ is the effective weight of client $i\in C^{(t)}$ among all participating clients such that $\mathbb{E}(\sum_{i\in C^{(t)}}\widetilde{p}_i)=1$. 
Note that in \cref{eq:localandserveraggregation1}, the local aggregation $q_{y,i}^{(t)}$ (similarly for $v$ and $x$) can be regarded as a linear combination of all local stochastic gradients, and hence 
covers a variety of local optimizers such as stochastic gradient descent, momentum-based gradient, variance reduction by choosing different coefficients $a^{(t,k)}_i$ for $i\in C^{(t)}$. This substantially enhances the flexibility of the proposed framework.  

\begin{algorithm}[t]
	\caption{ \colorbox{DarkSeaGreen2}{SimFBO} and  \colorbox{LemonChiffon1}{ShroFBO}}   
	\small
	\label{alg:main}
	\begin{algorithmic}[1]
		\STATE {\bfseries Input:} initialization $\bm{x}^{(0)}$,$\bm{y}^{(0)}$, number of communication rounds $T$, learning rates: client $\{\eta_y, \eta_v, \eta_x\}$, server: $\{\gamma_y, \gamma_v, \gamma_x\}$, local update rounds: $\{\tau_i^{(t)}\}$
        % \STATE *add something here*
        \FOR{$t = 0,1,2,...,T$} 
        \FOR{$i \in C^{(t)}$ \textbf{in parallel}}
        \STATE $y_i^{(t,0)} = y^{(t)}$, $v_i^{(t,0)} = v^{(t)}$, $x_i^{(t,0)} = x^{(t)}$
        \FOR{$k = 0,1,2,...,\tau_i^{(t)}-1$}
        \STATE Locally update $y_i^{(t,k)}$, $v_i^{(t,k)}$ and $x_i^{(t,k)}$ simultaneously via  \cref{eq:localupdate}
        \ENDFOR
        \STATE \colorbox{DarkSeaGreen2}{Client $i$ locally aggregates gradients to compute $q_{y,i}^{(t)}$, $q_{v,i}^{(t)}$, $q_{x,i}^{(t)}$ via~\cref{eq:localandserveraggregation1}}
        \STATE \colorbox{LemonChiffon1}{Client $i$ locally aggregates gradients to compute $h_{y,i}^{(t)}$, $h_{v,i}^{(t)}$, $h_{x,i}^{(t)}$ defined in \cref{eq:localAEofq}} 
        \ENDFOR
        \STATE Client $i \in C^{(t)}$ communicate $\{h_{y,i}^{(t)}, h_{v,i}^{(t)}, h_{x,i}^{(t)}\}$ to the server
        \STATE \colorbox{DarkSeaGreen2}{Server aggregates local estimators to compute $\{q_{y}^{(t)}, q_{v}^{(t)}, q_{x}^{(t)}\}$ using \cref{eq:localandserveraggregation1}}
        \STATE \colorbox{LemonChiffon1}{Server aggregates local estimators to compute $\{h_{y}^{(t)}, h_{v}^{(t)}, h_{x}^{(t)}\}$ using \cref{eq:AEofyvx2}}
        \STATE \colorbox{DarkSeaGreen2}{Server updates using \cref{eq:serverupdate}}
        \STATE \colorbox{LemonChiffon1}{Server updates using \cref{eq:serverupdatefinal}}
        \ENDFOR
	\end{algorithmic}
\end{algorithm}
\noindent
\textbf{Server-side updates.} Based on the aggregated gradients $q_{y}^{(t)},q_{v}^{(t)}$ and $q_{x}^{(t)}$, we then perform server-level gradient-based updates on variables $x,v$ and $y$ simultaneously as 
\begin{align}  \label{eq:serverupdate}
        y^{(t+1)} = y^{(t)} - \gamma_y q_{y}^{(t)}, \;\;
        v^{(t+1)} = \mathcal{P}_{r}\big(v^{(t)} - \gamma_v q_{v}^{(t)}\big), \;\;
        x^{(t+1)} = x^{(t)} - \gamma_x q_{x}^{(t)},
\end{align}
where $\gamma_y$, $\gamma_v $ and $\gamma_x $ are server-side updating stepsizes for $y$, $v$, $x$ and $\mathcal{P}_{r}(v) := \min\big\{1, \frac{r}{\|v\|}\big\}v$ is a simple projection on a bounded ball with a radius of $r$. There are a few remarks about the updates in \cref{eq:serverupdate}. First, in contrast to existing FBO algorithms such as~\cite{tarzanagh2022fednest,xiao2023communication}, our introduced server-side updates leverage not only the client-side stepsizes $\eta_y,\eta_v,\eta_x$, but also the server-side stepsizes $\gamma_y,\gamma_v$ and $\gamma_x$. This generalized two-learning-rate paradigm can provide more algorithmic and theoretical flexibility, and provides improved communication efficiency in practice and in theory. Second, the projection $\mathcal{P}_r(\cdot)$ serves as  
an important step to ensure the boundedness of variable $v^{(t)}$, and hence guarantee the smoothness of the global LS problem and the boundedness of the estimation variance in $v$ and $x$ updates, both of which are crucial and necessary in the final convergence analysis. Note that we do not impose such projection on the local $v_i^{(t,k)}$ variables because we can prove via induction that they are bounded given the boundedness of $v^{(t)}$ (see \Cref{pps:propositionboundv}). 

\subsection{Resilient Server-side Updates against System-level Heterogeneity}
{\bf Limitations under system-level heterogeneity.} When clients have heterogeneous computing and storing capabilities (e.g., computer server v.s.~phone in edge computing), an unequal number of local updates are often performed such that the global solution can be biased toward those of the clients with much more local steps or stronger optimizers. As observed in \cite{wang2020tackling}, this heterogeneity can deviate the iterates to minimize a different objective function. To explain this mismatch phenomenon, inspired by~\cite{sharma2023federated},  we rewrite the server-side update on $x$ (similarly for $v$ and $y$) in \cref{eq:localandserveraggregation1} as 
\begin{align}\label{eq:localAEofq}
    % q_{y}^{(t)} = \sum_{i \in C^{(t)}}\widetilde{p}_iq_{y,i}^{(t)} =& \sum_{i \in C^{(t)}}\widetilde{p}_i\sum_{k=0}^{\tau_i-1}a_i^{(t,k)}\nabla_y g_i\big(x_i^{(t,k)},y_i^{(t,k)}; \zeta_i^{(t,k)}\big),\nonumber \\
    % q_{v}^{(t)} = \sum_{i \in C^{(t)}}\widetilde{p}_iq_{v,i}^{(t)} =& \sum_{i \in C^{(t)}}\widetilde{p}_i\sum_{k=0}^{\tau_i-1}a_i^{(t,k)}\nabla_v R_i\big(x_i^{(t,k)},y_i^{(t,k)}v_i^{(t,k)}; \psi_i^{(t,k)}\big),\nonumber \\
   q_{x}^{(t)} &= \sum_{i = 1}^n p_i q_{x,i}^{(t,k)} = \underbrace{\Big(\sum_{j=1}^np_j\|a_j^{(t)}\|_1\Big)}_{\rho^{(t)}} \sum_{i = 1}^n \underbrace{\frac{p_i\|a_i^{(t)}\|_1}{\sum_{j=1}^np_j\|a_j^{(t)}\|_1}}_{w_i} \underbrace{\frac{q_{x,i}^{(t)}}{\|a_i^{(t)}\|_1}}_{h_{x,i}^{(t)}}. 
\end{align}
where {\small$a_i^{(t)}=\big[a_i^{(t,0)},...,a_i^{(t,\tau_i^{(t)}-1)}\big]^T$} collects all local coefficients of client $i$, and 
$h_{x,i}^{(t)}$ {\bf normalizes} the aggregated gradient $q_{x,i}^{(t)}$ by {\small $1/\|a_i^{(t)}\|_1$} such that $\|h_{x,i}^{(t)}\|$ does not grow with the increasing of $\tau_i^{(t)}$. Although such normalization can help to mitigate the system-level heterogeneity, 
the effective weight $w_i$ can deviate from the true weight $p_i$ of the original objective in \cref{eq:intro}, and the iterates converge to the stationary point of a different objective that replaces all $p_i$ by $w_i$ in \cref{eq:intro} (see \Cref{th:theorem1}).

\noindent
{\bf  System-level heterogeneity robust FBO (ShroFBO).} To address this convergence issue, we then propose a new method named ShroFBO with stronger resilience to such heterogeneity.  Motivated by the normalized reformulation in \cref{eq:localAEofq}, ShroFBO adopts a different server-side aggregation as 
\begin{align}\label{eq:AEofyvx2}
    h_y^{(t)} = \sum_{i \in C^{(t)}}  \widetilde{p}_ih_{y,i}^{(t)}, 
    \quad  h_v^{(t)} = \sum_{i \in C^{(t)}}  \widetilde{p}_ih_{v,i}^{(t)},
    \quad h_x^{(t)} = \sum_{i \in C^{(t)}}  \widetilde{p}_ih_{x,i}^{(t)},
\end{align}
where $\widetilde{p}_i := \frac{n}{|C^{(t)}|}p_i$ and $h_{y,i}^{(t)},h_{v,i}^{(t)},h_{x,i}^{(t)}$ are the normalized local aggregations defined in \cref{eq:localAEofq}. Accordingly, the server-side updates become 
\begin{align}\label{eq:serverupdatefinal}
    y^{(t+1)} = y^{(t)} - \rho^{(t)}\gamma_y h_{y}^{(t)}, \;\;
    v^{(t+1)} = \mathcal{P}_{r}\big(v^{(t)} - \rho^{(t)}\gamma_v h_{v}^{(t)}\big) , \;\; 
    x^{(t+1)} = x^{(t)} - \rho^{(t)}\gamma_x h_{x}^{(t)}. 
\end{align}
Differently from \cref{eq:localAEofq}, we select the client weights to be $\widetilde p_i$ to enforce the correct convergence to the stationary point of the original objective in \cref{eq:intro}, as shown in \Cref{th:theorem2} later.

\section{Main Result}
\subsection{Assumptions and Definitions}
% {\color{red}KJ: revise your assumptions. A1: geometries. A2: smoothness and lipschitz. A3: variance (remove so many inequalities. Add a definition of variance at the beginning.) Remove A4: change to a lemma.  }
We make the following standard definitions and assumptions for the outer- and inner-level objective functions, as also adopted in stochastic bilevel optimization~\cite{ji2021bilevel, hong2020two, khanduri2021near} and in federated bilevel optimization~\cite{tarzanagh2022fednest,xiao2023communication,huang2023achieving}.
\begin{definition}
A mapping $F$ is $L$-Lipschitz continuous if for  $\forall\,z,z^\prime$,  
% \begin{align*}
$\|F(z)-F(z^\prime)\|\leq L\|z-z^\prime\|.$
% \end{align*}
\end{definition} 
\noindent
Since the overall objective $\Phi(x)$ is nonconvex, the goal is expected to find an $\epsilon$-accurate stationary point defined as follows.
\begin{definition}
We say $z$ is an $\epsilon$-accurate stationary point of the objective function $\Phi(x)$ if $\mathbb{E}\|\nabla \Phi(z)\|^2\leq \epsilon$, where $z$ is the output of an algorithm. 
\end{definition}
% The following assumption characterizes the geometries of the objective functions. 
\begin{assumption}\label{as:diffandSC}
    For any $x \in \mathbb{R}^{d_{x}}$, $y \in \mathbb{R}^{d_{y}}$ and $i \in \{1,2,...,n\}$, $f_i(x,y)$ and $g_i(x,y)$ are twice continuously differentiable, and $g_i(x,y)$ is $\mu_g$-strongly convex w.r.t.~$y$. 
\end{assumption}
\noindent
The following assumption imposes the Lipschitz continuity conditions on the upper- and lower-level objective functions and their derivatives. 
\begin{assumption}\label{as:Lipschitz}
    % For any $x \in \mathbb{R}^{d_{x}}$, $y \in \mathbb{R}^{d_{y}}$ and $i \in \{1,2,...,n\}$, 
    Function $f_i(x,y)$ is $L_f$-Lipschitz continuous;  the gradients $\nabla f_i(x,y)$ and $\nabla g_i(x,y)$ are $L_1$-Lipschitz continuous;  the second-order derivatives  
    % w.r.t. $(x, y)$; 
    $\nabla^2 f_i(x,y)$ and $\nabla^2 g_i(x,y)$ are $L_2$-Lipschitz continuous; 
    and the third-order derivatives $\nabla^3 g_i(x,y)$ is $L_3$-Lipschitz continuous
    for some constants $L_f,L_1, L_2,L_3>0$.
    % w.r.t. $(x, y)$. 
\end{assumption}
% \begin{assumption}\label{as:Lipschitz}
%     % For any $x \in \mathbb{R}^{d_{x}}$, $y \in \mathbb{R}^{d_{y}}$ and $i \in \{1,2,...,n\}$, 
%     Function $f_i(x,y)$ is $L_f$-Lipschitz continuous,  the gradients $\nabla f_i(x,y)$ and $\nabla g_i(x,y)$ are $L_1$-Lipschitz continuous, and the second-order derivative  
%     % w.r.t. $(x, y)$; 
%      $\nabla^2 g_i(x,y)$ is $L_2$-Lipschitz continuous for some constants $L_f,L_1, L_2>0$.
%     % w.r.t. $(x, y)$. 
% \end{assumption}
\noindent
% It is worth mentioning that third-order Lipchitz continuity is required since $R_i$ contains a second-order gradient. Thus, such an assumption helps us figure out the Lipchitz continuity of $v^*(x)$. 

% It is worth mentioning that the second-order gradient in $R_i$ ensures the simultaneous and single-loop updates but also makes $\nabla v^*(x)$ hard to analyze. To achieve the Lipschitz-continuity of $\nabla v^*(x)$, we enhance the smoothness assumption by requiring the Lipchitz continuity of the third-order gradient. 
\noindent
The Lipschitz continuity of the third-order derivative is necessary here to ensure the smoothness of $v^*(x)$, which guarantees the descent in the iterations of LS function (see \Cref{lm:servergap}), under our more challenging simultaneous and single-loop updating structure. 
% Our simultaneous and single-loop updates structure requires the smoothness of $v^*(x)$ to achieve the descent in the iterations of the LS function in \Cref{lm:servergap}. The third-order gradient Lipchitz continuity here guarantees the smooth $v^*(x)$ in \Cref{lm:3ieq}. 
% To achieve simultaneous and single-loop local updates, we utilize LS parameter $v$ and LS function $R$. However, the Hessian matrix in $R$ makes the properties of $v^*(x)$ much harder to analyze than $y^*(x)$. The assumption of third-order Lipchitz continuity helps us obtain the smoothness of $v^*(x)$ in \Cref{lm:3ieq}, 
% which implicitly helps us achieve the descent in the iterations of the LS function in \Cref{lm:servergap}. 
% \red{rewrite}
\noindent
Next, we assume the bounded variance conditions on the gradients and second-order derivatives.  
\begin{assumption}\label{as:varaince}
% Assumptions~\ref{as:diffandSC} and ~\ref{as:Lipschitz} hold for $f_i(x, y; \xi)$ and $g_i(x, y; \zeta)$  for $\forall\,\xi$ and $\zeta$. 
% (Version 1)There exists positive constant $\sigma_{g}^2$ as the bound of the variance of $\nabla_y g_i(x,y;\zeta)$ satisfying
% \begin{align}
%     \mathbb{E}\big[\|\nabla_y g_i(x,y) - \nabla_y g_i(x,y;\zeta)\|^2\big] \leq \sigma_g^2. \nonumber 
% \end{align}
% Moreover, there exist positive constants $ \sigma_{f}^2$ as the bound of the variance of $\nabla_x f_i(x,y,\xi)$ and $\nabla_y f_i(x,y,\xi)$; positive constants $ \sigma_{gg}^2$ as the bound of the variance of $\nabla^2_{xy} g_i(x,y,\zeta)$ and $\nabla^2_{yy} g_i(x,y,\zeta)$ for all $i \in \{1,...,n\}$.
There exist constants $\sigma_{f}^2$, $\sigma_{g}^2,\sigma_{gg}^2$ such that $\mathbb{E}\big[\|\nabla f_i(x,y) - \nabla f_i(x,y;\xi)\|^2\big] \leq \sigma_f^2$,$\quad$ $\mathbb{E}\big[\|\nabla g_i(x,y) - \nabla g_i(x,y;\zeta)\|^2\big] \leq \sigma_g^2$ and $\mathbb{E}\big[\|\nabla^2 g_i(x,y) - \nabla^2 g_i(x,y;\zeta)\|^2\big] \leq \sigma_{gg}^2$.
\end{assumption}

\begin{assumption}\label{as:globalheter}
    For any $x \in \mathbb{R}^{d_{x}}$, $y \in \mathbb{R}^{d_{y}}$, there exist constants $\beta_{gh} \geq 1$ and $\sigma_{gh} \geq 0$ such that 
    \begin{align}
        \sum_{i=1}^n w_i \|\nabla_y g_i(x,y)\|^2 \leq \beta_{gh}^2 \|\sum_{i=1}^n w_i \nabla_y g_i(x,y)\|^2 + \sigma_{gh}^2. \nonumber
    \end{align}
We have $\beta_{gh}= 1$, and $\sigma_{gh}= 0$ when all $g_i$'s are identical.
\end{assumption}
\noindent
This assumption of global heterogeneity uses $\beta_{gh}$ and $\sigma_{gh}$ to measure
the dissimilarity of $\nabla_y g_i(x,y)$ for all $i$. 
% which guarantees the summation of client drift $\|y^{(t)} - y^*{(x^{(t)})}\big\|^2$ to be bounded in \Cref{lm:boundofCD} when $g_i$ is strongly convex. 

\subsection{Convergence and Complexity Analysis}
It can be seen from \cref{def:R} that the boundedness of $v$ is necessary to guarantee the smoothness (w.r.t.~$x,y$) and bounded variance in solving the local and global LS problems. Projecting the global $v^{(t)}$ vector and the local $v_i^{(t,k)},k\geq 1$  vectors onto a bounded set can be a feasible solution, but in this case, the local aggregation $q_{v,i}^{(t)}$ is no longer a linear combination of local gradients. This can complicate the implementation and analysis, and degrade the flexibility of the framework. Fortunately, we show via induction that the projection of the server-side vector $v^{(t)}$ on a bounded set suffices to guarantee the boundedness of local vectors $v_i^{t,k}$.
% as we shown in the following proposition, w  

% this will break the linear combination 

% Note that in our \Cref{alg:main}, we do not 

% Note that we use auxiliary projection to bound $v^{(t)}$ on the server in \cref{eq:serverupdate}. 
% First, we show that server projection contributes to the boundedness of the local $v_{i}^{(t,k)}$. 
\begin{proposition}[Boundedness of Local $v$]\label{pps:propositionboundv} 
    Under Assumptions~\ref{as:diffandSC} and~\ref{as:Lipschitz}, for each iteration $t$, client $i$, and local iteration $k=1,2,..., \tau_i^{(t)}$, we have 
    % \begin{align}
     $    r_i := \|v_i^{(t,k)}\| \leq \Big( 1 + \frac{\alpha_{\max}}{\alpha_{\min}} \Big) r,$ 
         % \nonumber
    % \end{align}
where the radius $r = \frac{L_f}{\mu_g}$ and $\alpha_{\min},\alpha_{\max}$ are chosen such that  $\alpha_{\min} \leq a_i^{(t,k)} \leq \alpha_{\max}$. 
\end{proposition}
% It can be seen from \Cref{pps:propositionboundv}, $\|v_i^{(t,k)}\|$ can be bounded by server auxiliary projection radius no matter how many local update rounds we use, so we only need to implement projection on the server, not locally. As a result, the local LS aggregated estimator will keep being a linear combination of local gradients, which provides a huge convenience in practice: we can compute the local LS aggregated estimator as $h_{v,i}^{(t)} = \big(v_i^{(t,\tau_i)} - v_i^{(t,0)}\big)\big/\big({\eta_v\|a_i^{(t)}\|_1}\big)$. To be specific, we only need to save the start point $v_i^{(t,0)}$ and endpoint $v_i^{(t,\tau_i)}$ instead of $\tau_i$ local gradients, which saves considerable memory. 
\noindent
Next, we show an important proposition in characterizing the per-iteration progress of the global $v^{(t)}$ updates in approximating the solution of a reweighted global LS problem. Let $\Delta_v^{(t)}= \mathbb{E}\|v^{(t)} -\widetilde{v}^*(x^{(t)})\|^2$ denote the approximation error, where $\widetilde{v}^*$ be the  minimizer of $\sum_{i=1}^n w_i R_i(x,\widetilde{y}^*,\cdot)$.  
\begin{proposition}\label{pps:propositionofdescent} 
Under the Assumption \ref{as:diffandSC}, \ref{as:Lipschitz} and \ref{as:varaince}, the iterates $v^{(t)}$ in solving the global LS problem generated by  Algorithm \ref{alg:main} satisfy
\begingroup
\allowdisplaybreaks
% \begin{align}
%     \Delta_v^{(t+1)} 
%     \leq& (1 + \delta_t')(1 - \rho^{(t)}\gamma_v\frac{\mu_g}{2}) \Delta_v^{(t)} + (1+\delta_t')(\rho^{(t)}\gamma_v)^2\mathbb{E}\Big\|\sum_{i \in C^{(t)}} \widetilde{w}_i h^{(t)}_{v,i}\Big\|^2 \nonumber\\[-4.0mm]
%     & + 2(1+\delta_t')\rho^{(t)}\gamma_v \sum_{i \in C^{(t)}} \widetilde{w}_i \sum_{k=0}^{\tau_i^{(t)}-1}\frac{a^{(t,k)}_i}{\|a_i^{(t)}\|_1}L_{1} \mathbb{E}\|v^{(t)} - v^{(t,k)}_i\|^2 \nonumber\\[-1.0mm]
%     & + (1+\frac{1}{\delta_t'})\big(\rho^{(t)}\gamma_x\big)^2L_v^2 \mathbb{E}\Big\|\sum_{i \in C^{(t)}} \widetilde{w}_i h^{(t)}_{x,i}\Big\|^2  \nonumber
% \end{align}
\begin{align}
    \mathbb{E}\|&v^{(t+1)} - \widetilde{v}^*(x^{(t+1)})\|^2 -  \mathbb{E}\|v^{(t)} - \widetilde{v}^*(x^{(t)})\|^2 \nonumber\\
    \leq& (\delta_t' - \rho^{(t)}\gamma_v \mu_g - \delta_t'\rho^{(t)}\gamma_v \mu_g) \mathbb{E}\|v^{(t)} - \widetilde{v}^*(x^{(t)})\|^2 + (1+\delta_t')(\rho^{(t)}\gamma_v)^2\mathbb{E}\Big\|\sum_{i \in C^{(t)}} \widetilde{w}_i h^{(t)}_{v,i}\Big\|^2 \nonumber\\
    & + (1+\delta_t')\rho^{(t)}\gamma_v \frac{4L_R^2}{\mu_g}\sum_{i = 1}^n w_i \sum_{k=0}^{\tau_i-1}\frac{a^{(t,k)}_i}{\|a_i^{(t)}\|_1} \mathbb{E}\Big[\big\|x^{(t)} - x_i^{(t,k)}\big\|^2 + \big\|y^{(t)} - y_i^{(t,k)}\big\|^2 + \big\|v^{(t)} - v_i^{(t,k)}\big\|^2\Big] \nonumber\\
    & + (1+\delta_t')\rho^{(t)}\gamma_v \frac{4L_R^2}{\mu_g}\mathbb{E}\big\|y^{(t)} - \widetilde{y}^*{(x^{(t)})}\big\|^2
    + \big(\rho^{(t)}\gamma_x\big)^2\bigg(L_v^2+\frac{L_{vx}}{4}\bigg) \mathbb{E}\bigg\|\sum_{i \in C^{(t)}}\widetilde{w}_i h_{x,i}^{(t)}\bigg\|^2 \nonumber \\
    &  + (\rho^{(t)}\gamma_x)^2\frac{2L_v}{\delta_{t,1}'}\mathbb{E}\Big\| \sum_{i=1}^n w_i \widetilde{h}^{(t)}_{x,i}\Big\|^2 .\nonumber
\end{align}
\endgroup
 for all $t \in \{0,1,...,T-1\}$, $k \in \{0,1,...,\tau_i^{(t)}-1\}$ and $i \in \{1,2,...,n\}$, where $\widetilde{w}_i := \frac{n}{|C^{(t)}|}w_i$. 
\end{proposition}
% [Descent in Iterates of the LS problem]
\noindent
Similarly, we can provide a per-iteration process of $y^{(t)}$ in approximating the solution $\widetilde y^*$ of the reweighted  lower-level global function $\sum_{i=1}^n w_i g_i(x,\cdot)$. Note that such characterizations do not exist in previous studies in single-level or minimax federated optimization with a single objective (e.g., \cite{sharma2023federated}) because our analysis needs to handle three different lower-level,  LS  and upper-level objectives. As shown in \Cref{pps:propositionofdescent}, the bound involves the client drift term $\mathbb{E}\|v^{(t)}-v_i^{(t,k)}\|^2$ (similarly for $y,x$), so the next step is to characterize this important quantity.  

% \Cref{lm:obj_function} provides the surrogate objective function and the main challenge different form~\cite{sharma2023federated} is that we need to bound the gap between server updated point and optimal solution $\|y^{(t)} - \widetilde{y}^*(t)\|^2$ and $\|v^{(t)} - \widetilde{v}^*(t)\|^2$. 

% This is mainly because you need to solve distinct outer-, inner- and LS functions, but the minimax problem only needs to deal with one function, which is strongly convex w.r.t $y$. 

% As shown in \Cref{pps:propositionofdescent}, $\mathbb{E}\|v^{(t)} - \widetilde{v}^*(t)\|^2$ is guaranteed a per-iteration improvement with a factor $1-\theta(\gamma_v)$ by choosing properly small $\delta_t'$ and $\gamma_v$. 
% Define $\widetilde{y}^*$ as the optimal solution of $\min_y \sum_{i=1}^n w_i g_i(x,y)$, then we also have similar result of $\mathbb{E}\|y^{(t)} - \widetilde{y}^*(t)\|^2$.Client drift of $v$ is bounded by \Cref{pps:propositionofCD} and the aggregated estimators of $y$ and $v$ are bounded by \Cref{lm:boundofAE}, which .
% We can dominate the aggregated estimator of $x$ by the objective function with properly small $\gamma_x$, so \Cref{pps:propositionofdescent} provides premises of the construction of Lyapunov function. 

% Next, we provide the upper bound of client drifts at each iteration $t$. 
\begin{proposition}\label{pps:propositionofCD}
Under Assumption \ref{as:diffandSC} and \ref{as:Lipschitz}, the local iterates client drift of $v_i^{(t,k)}$ is bounded as \vspace{-2.0mm}
\begin{align}
    \sum_{i=1}^n w_i\frac{1}{\|a_i^{(t)}\|_1}\sum_{k=1}^{\tau_i-1}a^{(t,k)}_i \mathbb{E}\|v^{(t,k)}_i - v^{(t)}\|^2 &\leq \eta^2_v \bar{\tau}\sigma_{M1}^2, \nonumber
    % \sum_{i=1}^n w_i\frac{1}{\|a_i^{(t)}\|_1}\sum_{k=1}^{\tau_i-1}a^{(t,k)}_i 
    % \frac{1}{\|a_i^{(t)}\|_1}\sum_{k=1}^{\tau_i^{(t)}-1}a^{(t,k)}_i \mathbb{E}\|v^{(t,k)}_i - v^{(t)}\|^2 &\leq \eta^2_v \Big(\|a_i^{(t)}\|^2_2(\sigma^2_{f}+r_i\sigma^2_{gg}) + \|a_i^{(t)}\|_1(L^2 + r^2_iL^2_1) \Big), \nonumber
\end{align}
for all $t \in \{0,1,...,T-1\}$, $k \in \{0,1,...,\tau_i-1\}$ and $i \in \{1,2,...,n\}$. We define 
$\bar{\tau}:= \sum_{i=1}^n \tau_i/n$ 
and 
$\sigma_{M1}^2:= \alpha_{\max}^2(\sigma^2_{f}+r_{\max}^2\sigma^2_{gg}) + \alpha_{\max}(L_f^2 + r^2_{\max}L^2_1) $.
\end{proposition}
\noindent
It can be seen from \Cref{pps:propositionofCD} that the bound on the client drift of the local updates on $v$ is proportional to $\eta_v$ and $\|a_i^{(t)}\|_1$. Since $\alpha_{\min}\leq a_i^{(t,k)} \leq \alpha_{\max}$,  $\|a_i^{(t)}\|_1$ is proportional to the number  $\tau_i^{(t)}$ of local steps.  Thus, this client drift is controllable by choosing $\tau_i^{(t)}$ and the local stepsizes $\eta_v$ properly. 
% ,  where $\tau_i$ highly depends on the local computing power so we mainly use a sufficiently small stepsize $\eta_v$ to control client drifts. We also have similar results for upper- and inner-problems. 
Then, combining the results in the above \Cref{pps:propositionboundv}, \ref{pps:propositionofdescent}, \ref{pps:propositionofCD}, and under a proper Lyapunov function, we obtain the following theorem. Let $P= |C^{(t)}|$ be the number of sampled clients. 
\begin{theorem}\label{th:theorem1} 
Define $\widetilde{\Phi}(x) = \widetilde{F}(x,\widetilde{y}^*)$ as the objective function by replacing $p_i$ in \cref{eq:intro} with $w_i$. Suppose Assumptions \ref{as:diffandSC}, \ref{as:Lipschitz} and \ref{as:varaince} are satisfied. 
% \begin{align}
%     \gamma_x = \mathcal{O}\Big(\sqrt{\frac{P}{T}}\Big), \gamma_y &= \mathcal{O}\Big(\frac{\sqrt{P}}{T^{3/4}}\Big), \gamma_v = \mathcal{O}\Big(\sqrt{\frac{P}{T}}\Big), \nonumber \\
%     \eta_x = \mathcal{O}\Big(\sqrt{\frac{1}{\bar{\tau}T}}\Big), \eta_y &= \mathcal{O}\Big(\sqrt{\frac{1}{\bar{\tau}T}}\Big), \eta_v = \mathcal{O}\Big(\sqrt{\frac{1}{\bar{\tau}T}}\Big). \nonumber
% \end{align}
The iterates by SimFBO in \Cref{alg:main}  satisfy
\begin{align}\label{eq:maintheorem}
    \min_t \mathbb{E}\Big\|\nabla \widetilde{\Phi}(x^{(t)})\Big\|^2 = \underbrace{\mathcal{O}\Big(\frac{M_1(n-P)}{n}\sqrt{\frac{\bar{\tau}}{PT}}\Big)}_{\text{partial participation error}} + \underbrace{\mathcal{O}\Big(M_2\sqrt{\frac{1}{P\bar{\tau}T}}\Big)}_{\text{full synchronization error}} + \underbrace{\mathcal{O}\Big(\frac{M_3}{\bar{\tau}T}\Big)}_{\text{local updates error}} , 
\end{align}
% \red{add buckets below.}
% by setting $\gamma_x = \mathcal{O}\big(\sqrt{P/T}\big)$, $\gamma_y = \mathcal{O}\Big(\frac{\sqrt{P}}{T^{3/4}}\Big)$, $\gamma_v = \mathcal{O}\big(\sqrt{P/T}\big)$, $\eta_x = \mathcal{O}\big(1/\sqrt{\bar{\tau}T}\big)$, $\eta_y = \mathcal{O}\big(1/\sqrt{\bar{\tau}T}\big)$, $\eta_v = \mathcal{O}\big(1/\sqrt{\bar{\tau}T}\big)$ and $M_1$, $M_2$ and $M_4$ are defined in \cref{eq:constants} in appendix. \red{Ugly.}
where $\gamma_x$, $\gamma_y$, $\gamma_v$, $\eta_x$, $\eta_y$, $\eta_v$ are set in  \cref{eq:stepsize} and $M_1$, $M_2$, $M_3$ are defined by  \cref{eq:constants} in appendix.
For the full client participation (i.e., $P=n$), the sample complexity is $\bar{\tau}T = \mathcal{O}(n^{-1}\epsilon^{-2})$, and the number of communication rounds is $T = \mathcal{O}(\epsilon^{-1})$. For partial client participation, the sample complexity is $\bar{\tau}T = \mathcal{O}(P^{-1}\epsilon^{-2})$, and the number of communication rounds is $T = \mathcal{O}(P^{-1}\epsilon^{-2})$. 
% \Cref{eq:setting}.
\end{theorem}

\noindent
First, when set $\bar{\tau} = \mathcal{O}(1)$, \Cref{th:theorem1} shows that SimFBO converges to a stationary point of an objective function $\widetilde \Phi(x)$ with a rate of $\mathcal{O}(\frac{1}{\sqrt{PT}}+\frac{1}{T})$, which, to the best of our knowledge, is the first linear speedup result under partial client participation without replacement. Note that
without system-level heterogeneity, i.e., $\|a_1^{(t)}\|=...=\|a_n^{(t)}\|$, $w_i=\frac{p_i\|a_i^{(t)}\|_1}{\sum_{j=1}^np_j\|a_j^{(t)}\|_1}=p_i$, and hence SimFBO converges to the stationary point of the original objective in \cref{eq:intro}. However, in the presence of system-level heterogeneity, SimFBO may converge to the stationary point of a different objective. 
Second, when nearly full clients participate, the partial participation error is approximately zero. Then we can see that setting local update round $\bar{\tau}$ to its upper-bound results in the best performance. 

% However, the partial participation error can not be ignored when $P$ is not close to $n$, which results in the partial participation error increase with local updates round $\bar{\tau}$. To achieve the best performance, $\bar{\tau} = \mathcal{O}(1)$ is the best choice.

% Proofs are provided in \Cref{proofoftheorem}. \Cref{th:theorem1} shows the convergence rate of SimFBO without system heterogeneity. Partial clients error in \cref{eq:maintheorem} can be canceled when full clients are included. Local updates error represents that at least one of the clients implements multiple local update rounds. The full synchronization error in \cref{eq:maintheorem} represents the optimization error for a centralized algorithm. 

% \Cref{th:theorem1} demonstrates the convergence for the surrogate objective function $\widetilde{F}$, which matches original $F$ only if there is no heterogeneity. 
% Next, we show the convergence by using our robust update estimators. 

% which indicates the reason for the mismatch of \cref{eq:localAEofq} with system heterogeneity and the feasibility of our robust strategy. 
\begin{theorem}\label{th:theorem2}
Define $\Phi(x) = F(x,y^*)$ as \cref{eq:intro}. Suppose Assumptions \ref{as:diffandSC}, \ref{as:Lipschitz} and \ref{as:varaince} are satisfied.  The iterates generated by ShroFBO in \Cref{alg:main} satisfy \vspace{-1.0mm}
\begin{align}\label{eq:maintheorem2}
    \min_t \mathbb{E}\Big\|\nabla \Phi(x^{(t)})\Big\|^2 = \mathcal{O}\Big(\frac{M_1(n-P)}{n}\sqrt{\frac{\bar{\tau}}{PT}}\Big) + \mathcal{O}\Big(M_2\sqrt{\frac{1}{P\bar{\tau}T}}\Big) + \mathcal{O}\Big(\frac{M_3}{\bar{\tau}T}\Big), 
\end{align} 
by setting the same server-side and local stepsizes and  $M_1$, $M_2$ and $M_3$ as in \Cref{th:theorem1}. For full client participation, the sample complexity is $\bar{\tau}T = \mathcal{O}(n^{-1}\epsilon^{-2})$, and the number of communication rounds is $T = \mathcal{O}(\epsilon^{-1})$. For partial client participation, the sample complexity is $\bar{\tau}T = \mathcal{O}(P^{-1}\epsilon^{-2})$, and the number of communication rounds is $T = \mathcal{O}(P^{-1}\epsilon^{-2})$. 
\end{theorem}
\noindent
In \Cref{th:theorem2}, we show that even under the system-level heterogeneity, ShroFBO can converge to the original objective function with the same convergence rate as SimFBO. This justifies the design principle of robust server-side updates.

\section{Related Work}
{\bf Bilevel optimization.} Bilevel optimization, first introduced by~\cite{bracken1973mathematical}, has been studied for decades. A class of constraint-based bilevel methods was then proposed~\cite{hansen1992new, gould2016differentiating, shi2005extended, sinha2017review}, whose idea is to replace the lower-level problem by the optimality conditions. Gradient-based bilevel algorithms have attracted considerable attention due to the effectiveness in machine learning. Among them, AID-based approaches~\cite{domke2012generic, pedregosa2016hyperparameter, liao2018reviving, arbel2021amortized} leveraged the implicit derivation of the hypergradient, which was then approximated via solving a linear system. 
ITD-based approaches~\cite{maclaurin2015gradient, franceschi2017forward, finn2017model, shaban2019truncated, grazzi2020iteration} approximated the hypergradient based on automatic differentiation via the forward or backward mode. A group of stochastic bilevel approaches has been developed and analyzed recently based on Neumann series~\cite{chen2021single, ji2021bilevel,arbel2021amortized}, recursive momentum~\cite{yang2021provably, huang2021biadam, guo2021randomized} and variance reduction~\cite{yang2021provably, dagreou2022framework}, etc. For the lower-level problem with  multiple solutions, several approaches were proposed based on the upper- and lower-level gradient aggregation~\cite{sabach2017first, liu2020generic, li2020improved}, barrier types of regularization~\cite{liu2021value,liu2022bome}, penalty-based formulations~\cite{shen2023penalty}, primal-dual technique~\cite{sow2022constrained}, and dynamic system-based methods~\cite{liu2021towards}.

% and a large body of bilevel optimization methods have been proposed after this. By replacing the lower-level problem with its optimality conditions, bilevel problem was reformulated to the single-level problem~\cite{hansen1992new, gould2016differentiating, shi2005extended, sinha2017review}. 
% Gradient-based bilevel optimization methods have recently demonstrated significant potential, which can be divided into approximate implicit differentiation (AID)\cite{domke2012generic, pedregosa2016hyperparameter, liao2018reviving, arbel2021amortized} and iterative differentiation (ITD)~\cite{maclaurin2015gradient, franceschi2017forward, finn2017model, shaban2019truncated, grazzi2020iteration} based approaches. Recently, a number of stochastic bilevel optimization algorithms have been proposed using the Neumann series~\cite{chen2021single, ji2021bilevel}, recursive momentum~\cite{yang2021provably, huang2021biadam, guo2021randomized} and variance reduction~\cite{yang2021provably, dagreou2022framework}. Theoretically, the convergence of bilevel optimization has been analyzed by~\cite{franceschi2018bilevel, shaban2019truncated, liu2021investigating, ghadimi2018approximation, ji2021bilevel, hong2020two, arbel2021amortized,dagreou2022framework}. 
% There are also multiple methods to find inner minima like mixed gradient aggregation~\cite{sabach2017first, liu2020generic, li2020improved}, log-barrier regularization~\cite{liu2021value}, primal-dual method~\cite{sow2022constrained} and dynamic barrier~\cite{liu2022bome}. 

\noindent
{\bf Federated (bilevel) learning.}  
Federated Learning was proposed to enable collaborative model training across multiple clients without compromising the confidentiality of individual data~\cite{konevcny2015federated, shokri2015privacy, mohri2019agnostic}. As one of the earliest methods of federated learning~\cite{mcmahan2017communication}, FedAvg  has inspired an increasing number of approaches to deal with different limitations such as slower convergence, high communication cost and undesired client drift by leveraging the techniques including proximal regularization~\cite{li2020federated}, periodic variance reduction~\cite{mitra2021linear, karimireddy2020scaffold}, proximal splitting~\cite{pathak2020fedsplit}, adaptive gradients~\cite{reddi2020adaptive}. Theoretically, the convergence of FedAvg and its variants has been analyzed in various settings with the homogeneous~\cite{stich2018local, wang2020zeroth, stich2020error, basu2019qsparse} or heterogeneous datasets~\cite{li2020federated, wang2021cooperative, mitra2021linear, khaled2019first}. \cite{wang2020tackling} analyzed the impact of the system-level heterogeneity such as heterogeneous local computing on the convergence. \cite{sharma2023federated} further extended the analysis and the methods to the minimax problem setting.

% In the setting with homogeneous setting, the convergence of local SGD has been analyzed by~\cite{stich2018local, wang2020zeroth, stich2020error, basu2019qsparse}. 

% address the issues such as 
% the slow convergence and client drift via regularization~\cite{li2020federated}, variance reduction~\cite{mitra2021linear, karimireddy2020scaffold}, proximal splitting~\cite{pathak2020fedsplit} and adaptive optimization~\cite{reddi2020adaptive}. Some focused on the convergence and performance with homogeneity~\cite{stich2018local, wang2020zeroth, stich2020error, basu2019qsparse} and heterogeneity~\cite{li2020federated, wang2021cooperative, mitra2021linear, khaled2019first}. 

% {\bf Federated bilevel learning.} 
\noindent
Federated bilevel optimization has not been explored well except for a few attempts recently. For example, \cite{gao2022convergence, li2022local} proposed momentum-based bilevel algorithms, and analyzed their convergence in the setting with homogeneous datasets. In the setting with non-i.i.d.~datasets, \cite{tarzanagh2022fednest} and \cite{huang2023achieving} proposed FedNest and FedMBO via AID-based federated hypergraident estimation, and \cite{xiao2023communication} proposed an ITD-based aggregated approach named Agg-ITD. Momentum-based techniques have been also used by \cite{huang2022fast,li2023communication} to improve the sample complexity. Moreover, there are some studies that focus on other distributed scenarios, including decentralized bilevel optimization~\cite{chen2022decentralized, yang2022decentralized, lu2022decentralized}, asynchronous bilevel optimization over directed network~\cite{yousefian2021bilevel}, and distributed bilevel network utility maximization~\cite{ji2023network}.

\begin{figure}[H]
    \centering
    \vspace{-0.35cm}
    \includegraphics[width=0.35\textwidth]{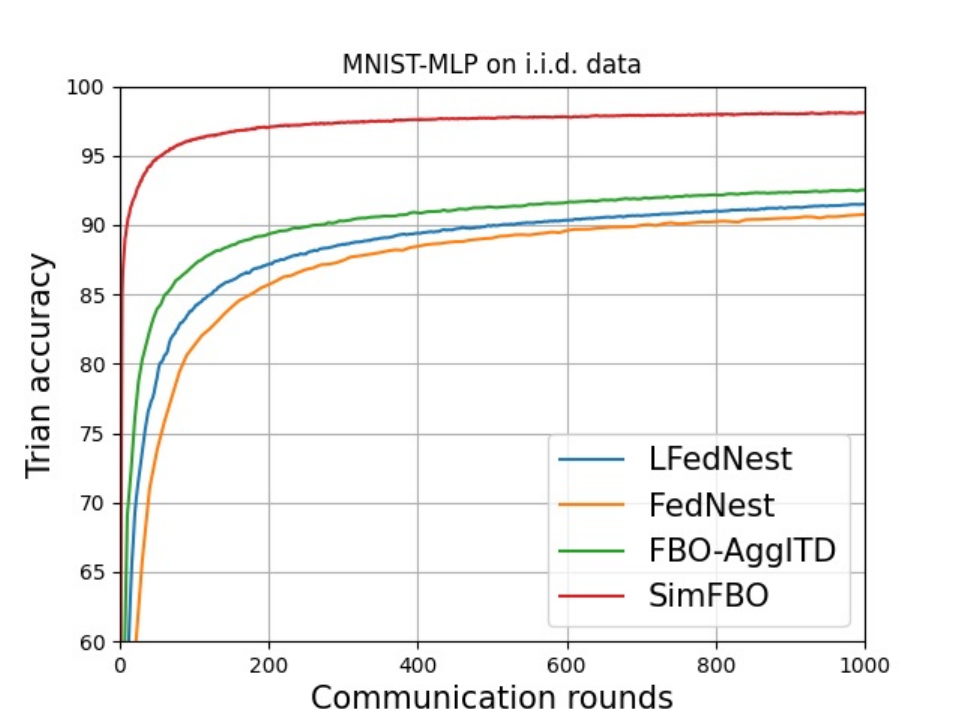}
    \hspace{-0.7cm}
    \includegraphics[width=0.35\textwidth]{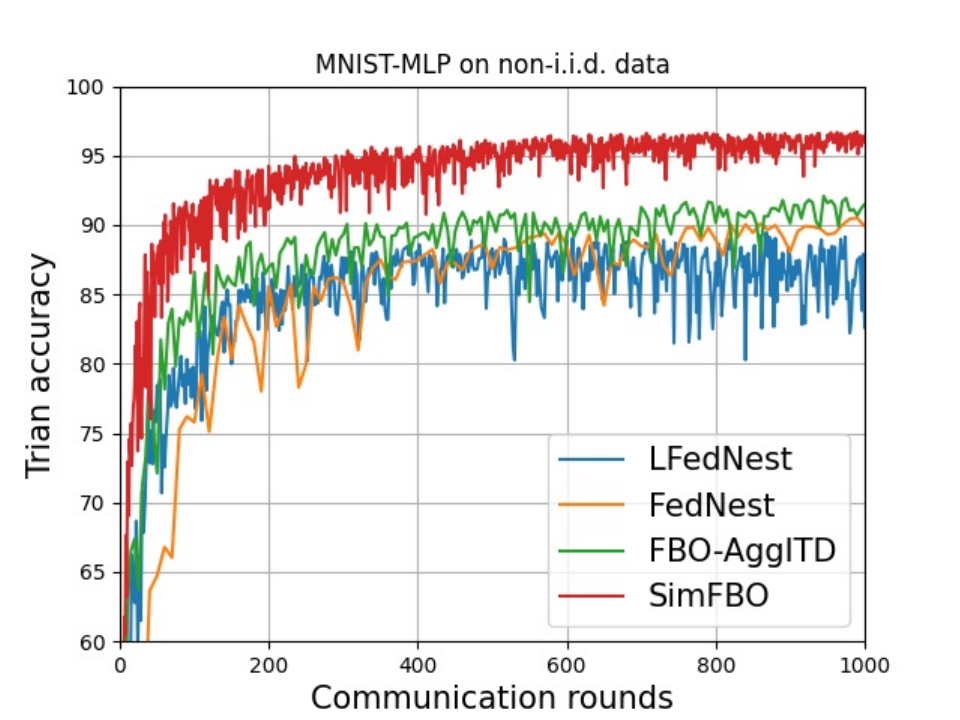}
    \hspace{-0.7cm}
    \includegraphics[width=0.35\textwidth]{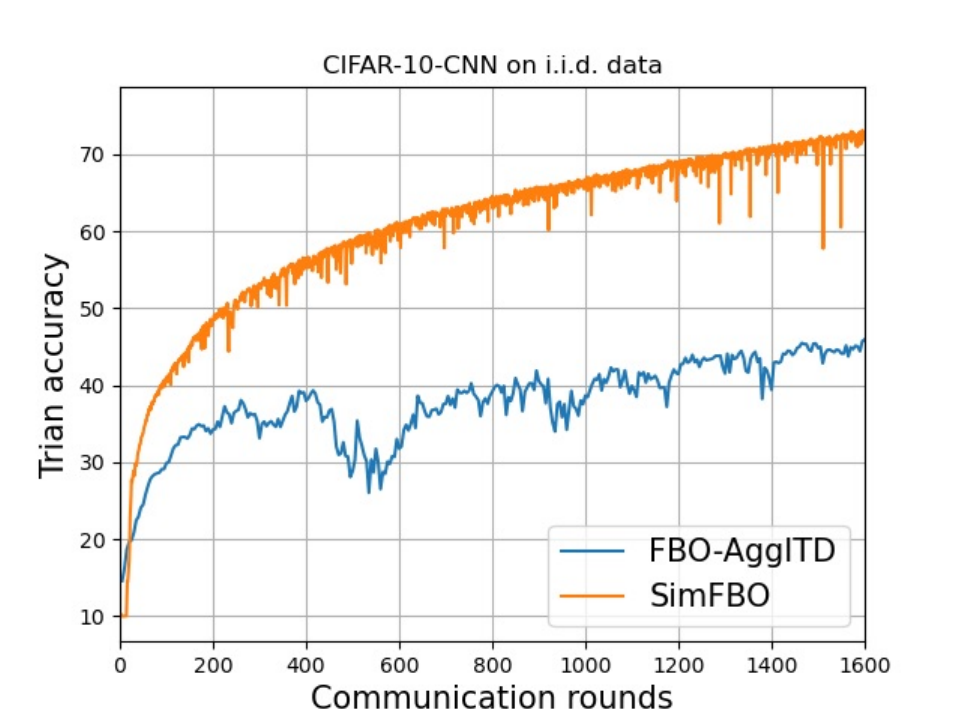}
    \vspace{-0.3cm} 
    \caption{Comparison among our SimFBO, FBO-AggITD~\cite{xiao2023communication}, FedNest~\cite{tarzanagh2022fednest} and  LFedNest~\cite{tarzanagh2022fednest}. The left and middle ones plot the training accuracy v.s.~\# of communication rounds on i.i.d.~MNIST datasets with MLP networks, and the right one plots the training accuracy v.s.~\# of rounds on i.i.d.~CIFAR-10 datasets with a 7-layer CNN.} 
    \label{realexptrain}
    \vspace{-0.2cm}
\end{figure}

\begin{figure}[H]
    \centering
    \vspace{-0.3cm}
    \includegraphics[width=0.35\textwidth]{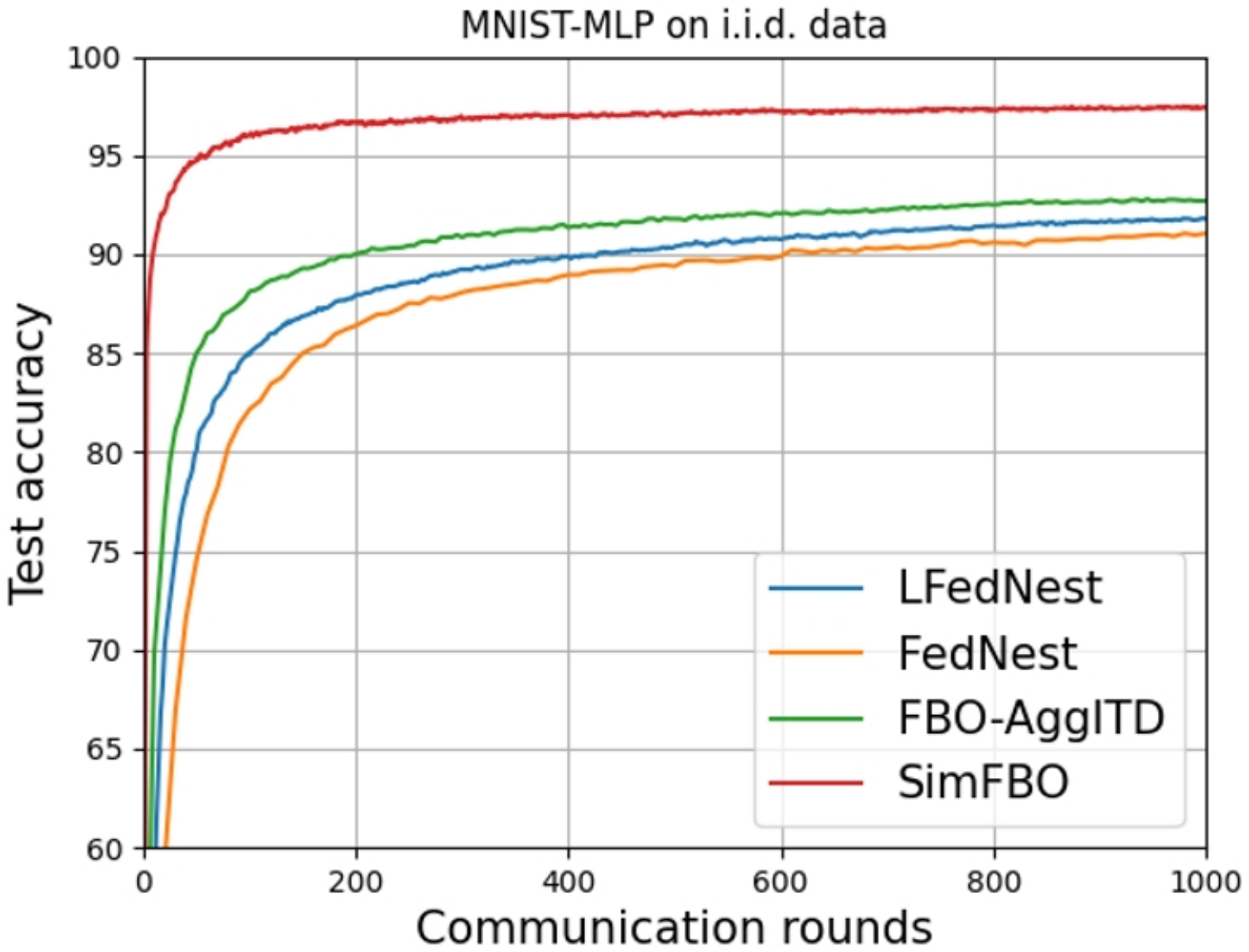}
    \hspace{-0.7cm}
    \includegraphics[width=0.35\textwidth]{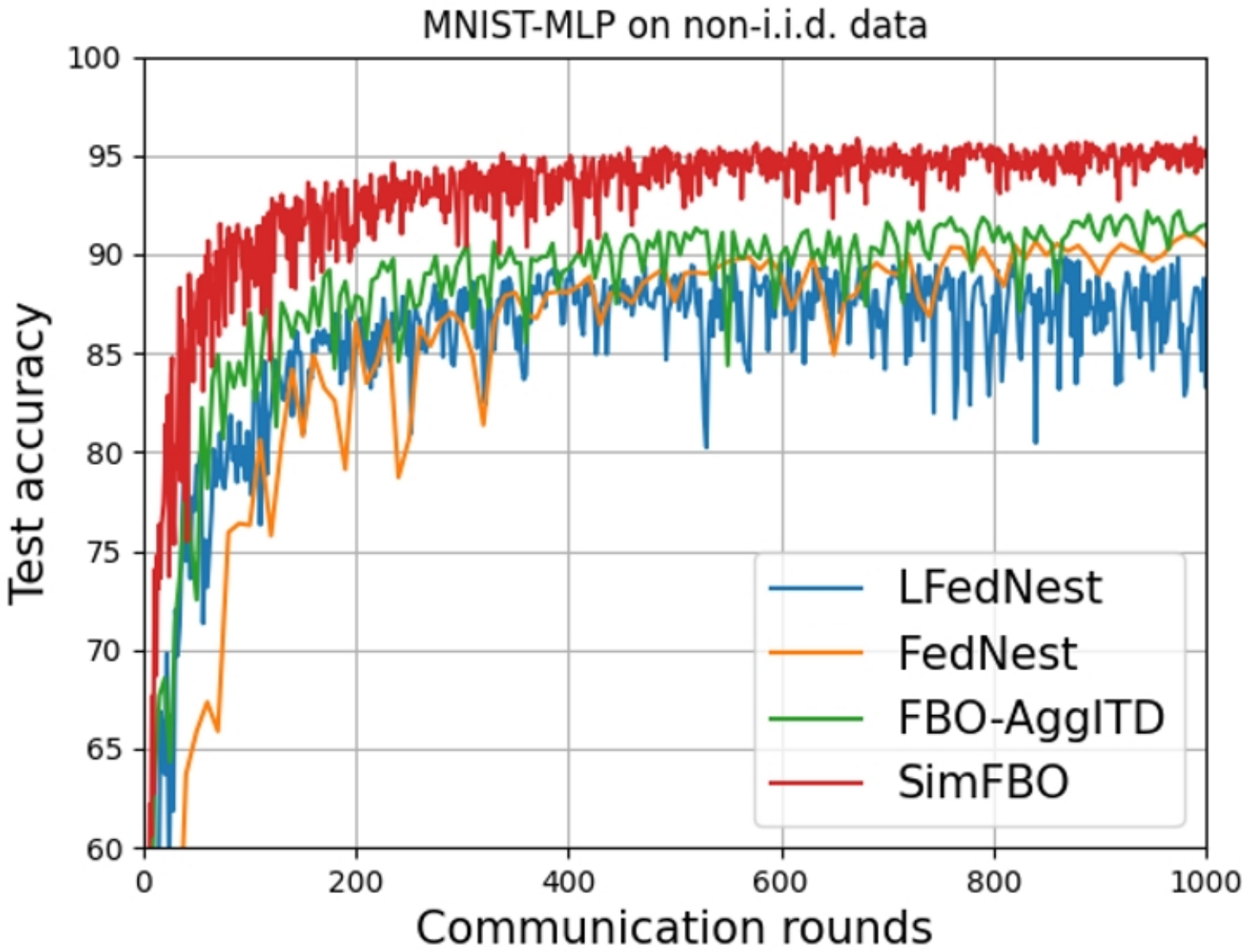}
    \hspace{-0.7cm}
    \includegraphics[width=0.35\textwidth]{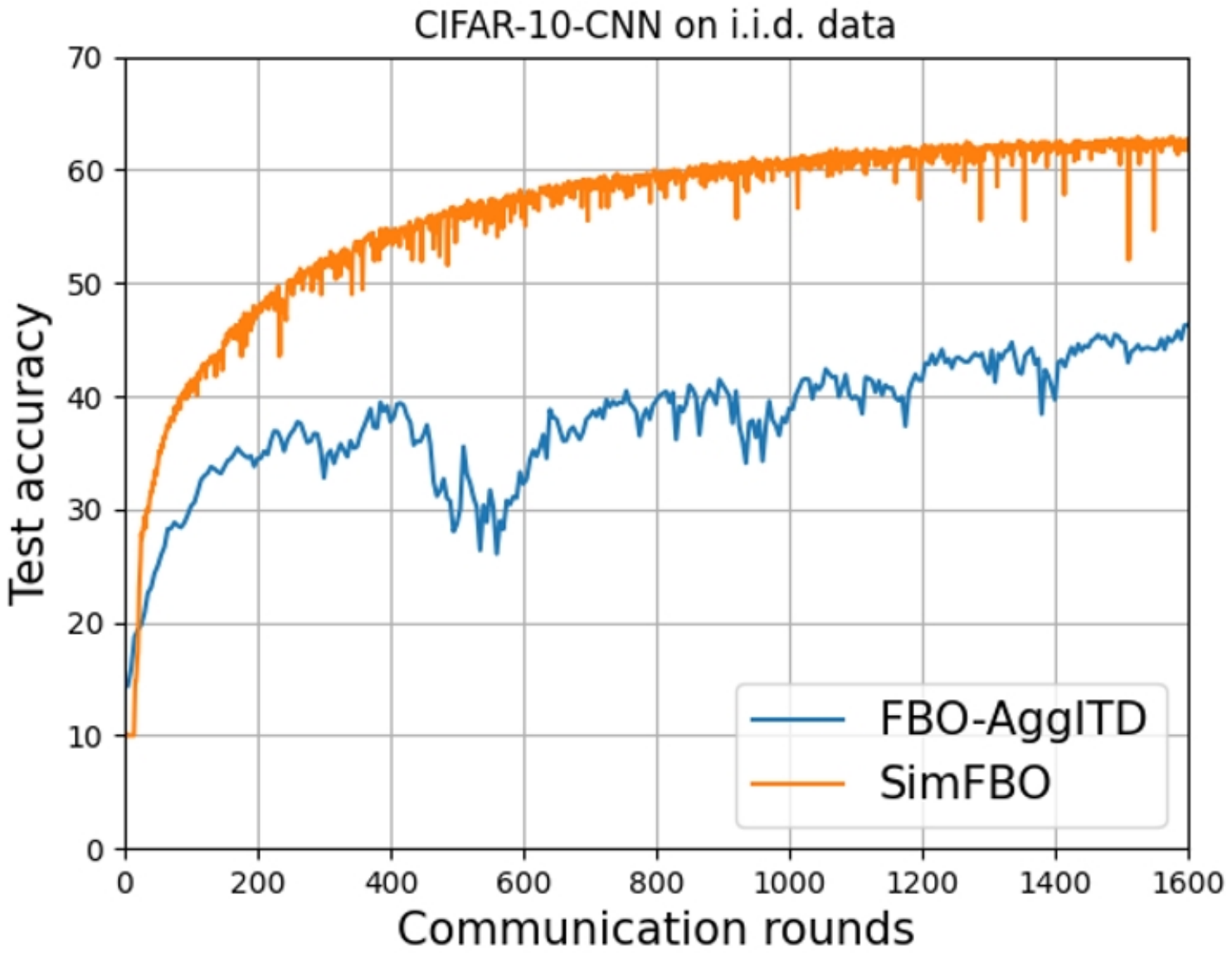}
    \vspace{-0.3cm}
    \caption{Comparison of different methods: the test accuracy v.s.~\# of communication rounds.} 
    \label{realexp}
    \vspace{-0.3cm}
\end{figure}

\section{Experiments}
In this section, we perform two hyper-representation experiments  to compare the performance of our proposed SimFBO algorithm with FBO-AggITD~\cite{xiao2023communication}, FedNest~\cite{xiao2023communication}, and LFedNest~\cite{tarzanagh2022fednest}, and validate the better performance  of ShroFBO in the presence of heterogeneous local computation. We test the performance on MNIST and CIFAR datasets with MLP and CNN backbones. We follow the same experimental setup and problem formulation as in \cite{tarzanagh2022fednest,xiao2023communication}. 
The details of all experimental specifications can be found in  \Cref{expdetails1}.

% on a hyper-representation problem, 

% which follows the same problem setup in~\cite{franceschi2018bilevel}. 

% For the \red{toy experiment}, we study the system-level heterogeneity while in hyper-representation experiment, we show that our SimFBO is more robust for image classification tasks using MNIST and CIFAR-10 datasets. 
% In detail, we use a 2-layer multilayer perceptron (MLP) as the backbone to train MNIST and CNN to train CIFAR-10. More details can be found in \Cref{expdetails1}.
\noindent
\textbf{Comparison to existing methods. } 
The comparison results are presented in \Cref{realexptrain} and \Cref{realexp}. It can be seen that across different datasets and backbones, 
our proposed SimFBO consistently converges much faster than other comparison methods, while achieving a much higher training and test accuracy. We do not plot the curves of FedNest and LFedNest on CIFAR and CNN, because they are hard to converge under various hyperparameter configurations using their source codes.

\noindent
\textbf{Performance under heterogeneous local computation. }
We now test the performance in the setting where a total of $10$ clients perform a variable number of local steps. This is to simulate the scenario where clients have heterogeneous computing capabilities and hence can perform an uneven number of local updates. In this experiment, we choose the number $\tau_i$ of the client $i$'s local steps from the set $\{1,...,10\}$ uniform at random. As shown in \Cref{toyexample}, the proposed ShroFBO method performs the best due to the better resilience to such client heterogeneity.  
% investigate the ability to solve system-level heterogeneity in different methods. The heterogeneity is introduced by choosing different local update rounds $\tau_i\in[1,10]$ randomly.  In \Cref{toyexample}, it can be seen that ShroFBO converges much faster than other methods, and achieves a higher test accuracy with much fewer communication rounds due to the normalized estimators. Details are provided in \Cref{expdetails2}. 
We also compare the convergence rate of our proposed SimFBO, FedNest and FBO-AggITD w.r.t.~running time. The results are provided in~\Cref{time}. All the settings for different algorithms are the same as in \Cref{expdetails2}. It can be seen that the proposed SimFBO still converges fastest with a higher test accuracy in terms of running time.

\begin{figure}[H]
    \centering
    \vspace{-0.25cm}
    \includegraphics[width=0.41\textwidth]{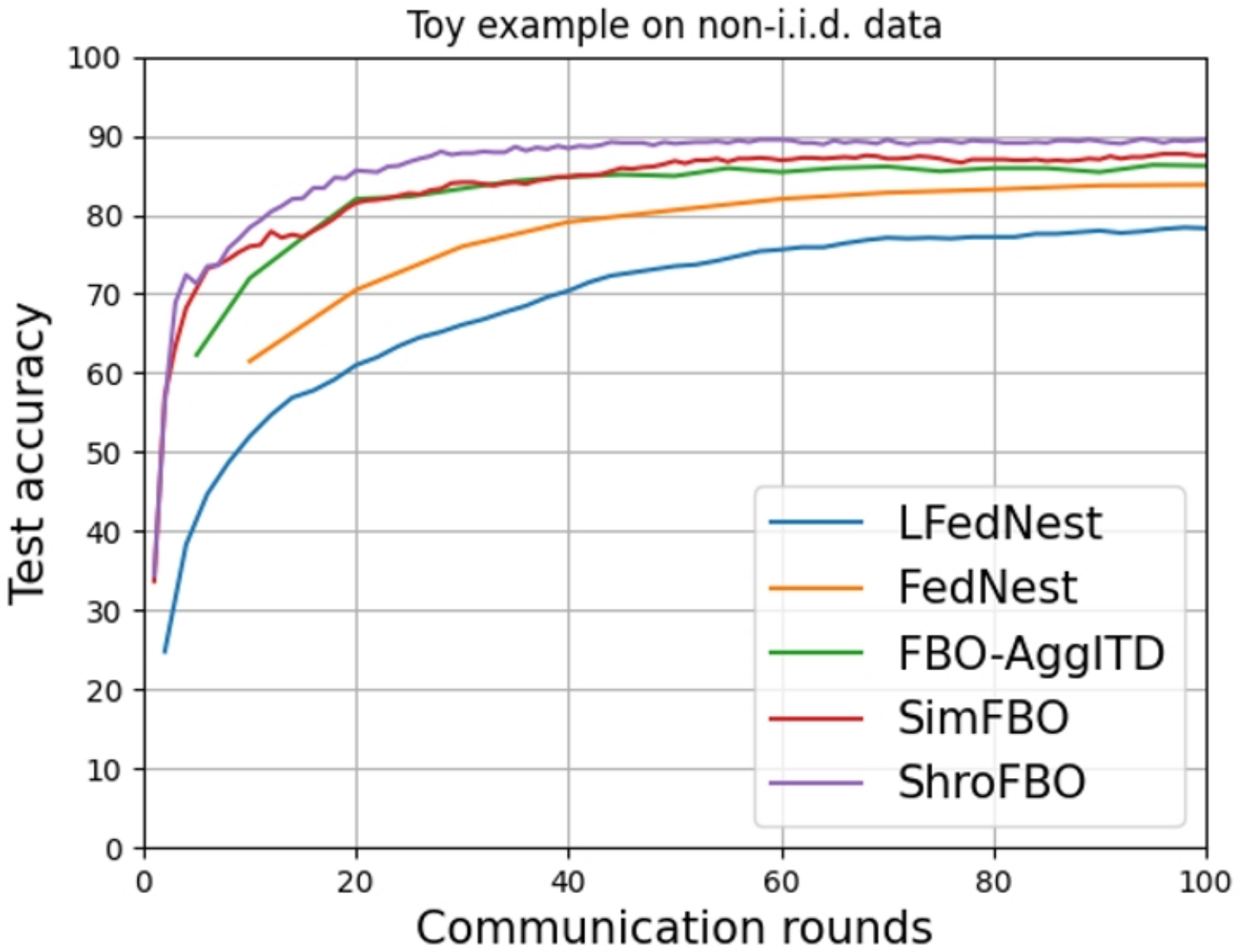}
    % \hspace{-1cm}
    \includegraphics[width=0.41\textwidth]{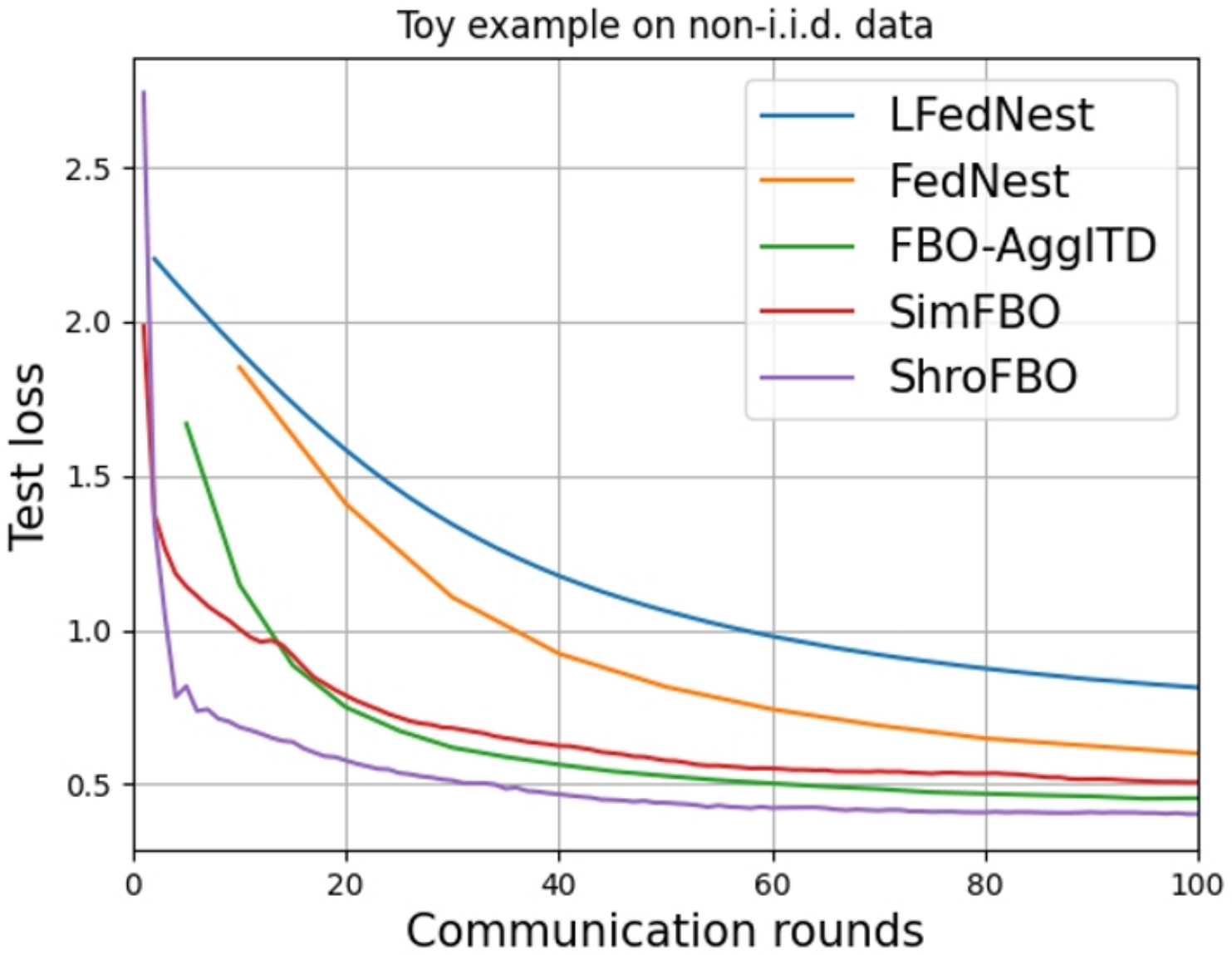}
        \vspace{-0.25cm}
    \caption{Comparison among different algorithms in the presence of heterogeneous local computation.} 
    \label{toyexample}
    \vspace{-0.2cm}
\end{figure}
\begin{figure}[H]
    \centering
    \vspace{-0.6cm}
\includegraphics[width=0.41\textwidth]{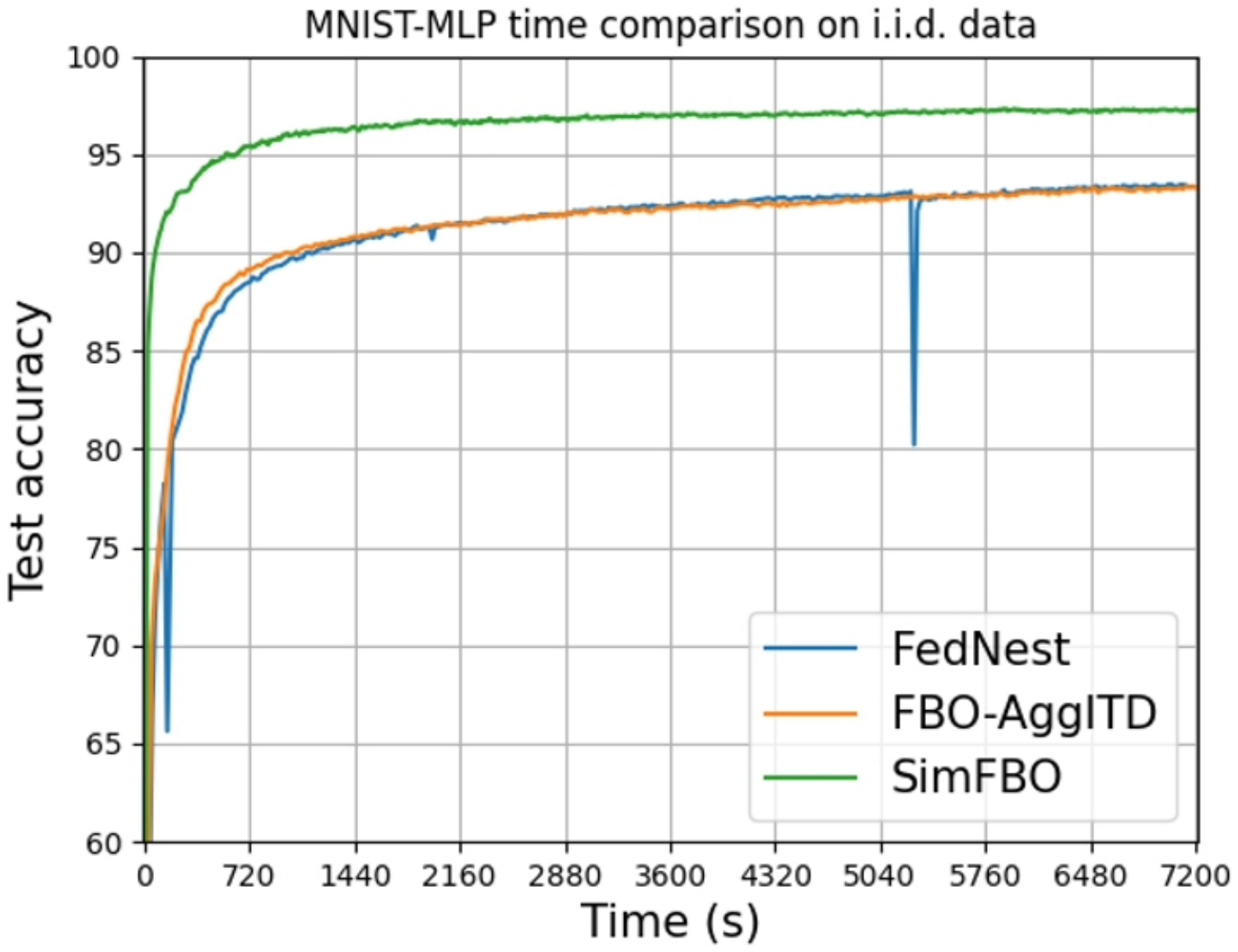}
    % \hspace{-0.8cm}
\includegraphics[width=0.41\textwidth]{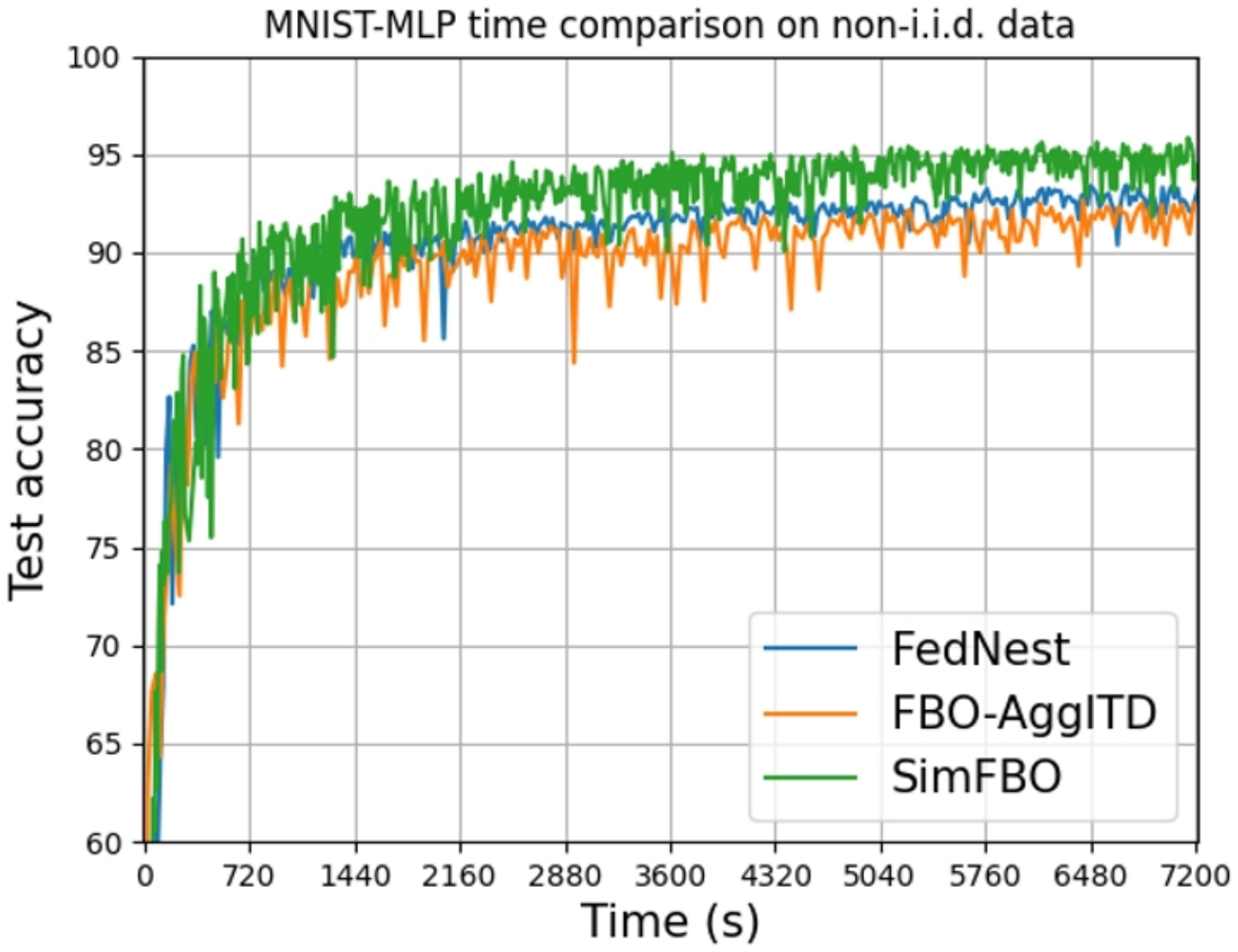}
        \vspace{-0.2cm}
    \caption{Time Comparison among different algorithms.} 
    \label{time}
    \vspace{-0.25cm}
\end{figure}
%for comparison reading
% \textbf{Performance under heterogeneous local computation.} The local update rounds for different clients are randomly chosen in the range from 1 to 10 for all methods to mimic system heterogeneity and there are 10 clients who are used repeatedly. 
% The inner stepsizes for LFedNest, FedNest, FBO-AggITD are 0.002, 0.01, 0.005, respectively. 
% The outer stepsizes for LFedNest, FedNest, FBO-AggITD are 0.005, 0.03, 0.04, respectively. 
% The rest parameters are chosen following~\cite{tarzanagh2022fednest, xiao2023communication}. 
% We take stepsizes of both SimFBO and ShroFBO to update $y, v, x$ as 0.03, 0.02, 0.01, respectively. 

% \begin{table}[]
% \begin{tabular}{cclccc}
% \toprule
% Algorithm & \multicolumn{2}{c}{Comm\_rounds/Outer\_itr} & Task & Data & Final Accuracy   \\ \midrule
% LFedNest & \multicolumn{2}{c}{K+1} & MNIST-MLP & IID / NON-IID & 91.8\% / 88.7\%  \\ \midrule
% FedNest & \multicolumn{2}{c}{2K+N+3} & MNIST-MLP & IID / NON-IID & 91.04\% / 90.9\% \\ \midrule
% \multirow{2}{*}{FBO-AggITD} & \multicolumn{2}{c}{\multirow{2}{*}{2K+3}} & MNIST-MLP & IID / NON-IID & 92.8\% / 91.3\%  \\ \cmidrule{4-6} 
%  & \multicolumn{2}{c}{} & CIFAR-CNN & IID & 40.00\% \\ \midrule
% \multirow{2}{*}{SimFBO}     & \multicolumn{2}{c}{\multirow{2}{*}{1}}      & MNIST-MLP & IID / NON-IID & 97.4\% / 95.2\% \\ \cmidrule{4-6} 
% & \multicolumn{2}{c}{} & CIFAR-CNN & IID & 62.50\% \\ \bottomrule
% \end{tabular}
% \caption{Quantitative comparison between LFedNest, FedNest, FBO-AggITD and SimFBO.} 
% \end{table}\label{acctable}

\section{Conclusion}
In this paper, we propose a simple and communication-efficient federated bilevel algorithm named SimFBO and its variant ShroFBO with better resilience to the system-level heterogeneity. We show that both SimFBO and ShroFBO allow for the more practical client sampling without replacement, and achieve better sample and communication complexities. Experiments demonstrate the great promise of the proposed methods. We anticipate that the proposed algorithms and the developed analysis can be applied to other distributed settings such as decentralized or asynchronous bilevel optimization, and the algorithms may be useful in applications such as hyperparameter tuning or fine-tuning in federated learning or AI-aware edge computing.

% \newpage
\bibliography{ref}
\bibliographystyle{abbrv}

\newpage
\appendix
\noindent
{\huge \bf Supplementary material}

\section{Specifications of Experiments}\label{expdetails}
\subsection{Model architecture and dataset}\label{expdetails1}

\textbf{MLP.} The 2-layer multilayer perceptron (MLP) has 784 input units and 200 hidden units so that the hidden layer parameters (157,000 parameters) are optimized for solving the upper-level problem and the output layer parameters (2,010 parameters) are optimized for solving the lower-level problem.

\noindent
\textbf{CNN.} We use the 7-layer CNN~\cite{lecun1998gradient} model to train CIFAR-10. We optimize the last fully connected layer's parameters for solving  the lower-level problem and optimize the rest layers' parameters for solving the upper-level problem.

\noindent
\textbf{Dataset.} For the hyper-representation experiment, we use full MNIST and CIFAR-10 datasets. In the experiment with heterogeneous local computation, we treat the first 2000 images in MNIST's default training dataset as the training data and the first 1000 images in MNIST's default test dataset as test data.

% \red{
% \textbf{MLP.} The 2-layer multilayer perceptron (MLP) has 200 hidden units so that the hidden layer is optimized by the outer problem (157,000 parameters) and the output layer is optimized by the inner problem (2,010 parameters).

% \textbf{CNN.} We use the 7-layer CNN~\cite{lecun1998gradient} model to train CIFAR-10. We treat the last fully connected layer (850 parameters) as the inner problem and the rest layers as the outer problem.

% \textbf{Dataset.} For the hyper-representation experiment, we use full MNIST and CIFAR-10 datasets. In the resistance of system heterogeneity experiment, we treat the first 2000 data in MNIST default training set as training data and the first 1000 data in MNIST default test set as test data.}

\subsection{Hyperparameter settings} \label{expdetails2}
For all comparison methods, we optimize their hyperparameters via grid search guided by the default values in their source codes, to ensure the best performance given the algorithms are convergent.  

\noindent
\textbf{Comparison to existing methods.} 
First of all, for all methods, 10 clients from 100 clients are chosen randomly and participate in each communication round. For the baseline methods FedNest and LFedNest, we use their published codes in \href{https://github.com/ucr-optml/FedNest}{https://github.com/ucr-optml/FedNest}. 
For FBO-AggITD, we use the source codes sent from the authors. % For FBO-AggITD, we use their codes published in the openreview 
% The code of FedNest, LFedNest can be found in Github repository, \href{https://github.com/ucr-optml/FedNest}{https://github.com/ucr-optml/FedNest}. 
For our method, SimFBO, we take the number of local updates,$\tau_i$, for each client $i$ to be 1, $a_i^{(t,k)}$ to be 1, and $\widetilde{p_i}$ to be 0.1. In MNIST-MLP experiment,
% First, for the MNIST-MLP experiment, the parameters in FedNest, LFedNest and FBO-AggITD are optimized via grid search guided by the default values in their published codes. 
% The inner stepsize of LFedNest is 0.001 and the outer stepsize is 0.0025 under the non-i.i.d setting. 
the stepsizes $[\eta_y, \eta_v, \eta_x]$ and $[\gamma_y,\gamma_v,\gamma_x]$ of our method for updating $[y_i^{t,k}(y^t), v_i^{t,k}(v^t), x_i^{t,k}(x^t)]$ are both [0.2, 0.1, 0.05], respectively.
Second, for the CIFAR-10-CNN experiment under the i.i.d.~setup, we only draw the result of FBO-AggITD and SimFBO because other algorithms cannot converge under various hyperparameter configurations. We take the best inner stepsize as 0.003 and the best outer stepsize as 0.005 for FBO-AggITD, and the stepsizes $[\eta_y, \eta_v, \eta_x]$ and $[\gamma_y,\gamma_v,\gamma_x]$ of our method for updating $[y_i^{t,k}(y^t), v_i^{t,k}(v^t), x_i^{t,k}(x^t)]$ are both [0.1, 0.05, 0.03], respectively.

% The stepsize of our method for updating $y, v, x$ are 0.2, 0.1, 0.05 respectively.
% Second, for the CIFAR-10-CNN experiment under the i.i.d setup, we only compare FBO-AggITD and SimFBO because other algorithms cannot converge to optimal solutions efficiently. We take the inner stepsize as 0.003 and the outer stepsize as 0.005 for FBO-AggITD, and the stepsizes of our method for updating $y, v, x$ are 0.1, 0.05, 0.03 respectively. The code of FedNest, LFedNest can be found in Github repository, \href{https://github.com/ucr-optml/FedNest}{https://github.com/ucr-optml/FedNest}. 
\noindent
\textbf{Performance under heterogeneous local computation.} In this experiment, we compare the results among LFedNest, FedNest, FBO-AggITD, SimFBO, and ShroFBO. For all methods, the numbers of local updates of different clients are randomly chosen in the range from 1 to 10 to simulate the system-level heterogeneity and there are a total of 10 clients participating during the entire procedure.
% The hyperparameters of the baseline methods LFedNest, FedNest and FBO-AggITD are selected via grid search guided by their source codes. 
In specific, inner stepsizes for LFedNest, FedNest, FBO-AggITD are [0.002, 0.01, 0.005], respectively while outer stepsizes for LFedNest, FedNest, FBO-AggITD are [0.005, 0.03, 0.04], respectively. Other hyperpameters of the above mentioned three methods are chosen the same as in the above experiment. Then for ShroFBO, we keep choosing $a_i^{(t,k)}$ for each client $i$ as 1 but the values of $\|a_j^{(t)}\|_1$ for $j\in[1,n]$ would be different since the number of local updates $\tau_i$ is randomly chosen between 1 and 10. The value of $p_j$ is chosen as 0.1.
% The rest parameters are chosen following~\cite{tarzanagh2022fednest, xiao2023communication}. 
Similarly, the stepsizes $[\eta_y, \eta_v, \eta_x]$ and $[\gamma_y,\gamma_v,\gamma_x]$ of our method for updating $[y_i^{t,k}(y^t), v_i^{t,k}(v^t), x_i^{t,k}(x^t)]$ are both [0.03, 0.02, 0.01], respectively. Lastly, SimFBO has the same settings mentioned above.

% \subsection{Running Time Comparison}
% We also compare the convergence rate of our proposed SimFBO, FedNest and FBO-AggITD w.r.t.~running time. The results are provided in~\Cref{time}. All the settings for different algorithms are the same as in \Cref{expdetails2}. It can be seen that the proposed SimFBO still converges fastest with a higher test accuracy in terms of running time.

% %for comparison reading
% % \textbf{Performance under heterogeneous local computation.} The local update rounds for different clients are randomly chosen in the range from 1 to 10 for all methods to mimic system heterogeneity and there are 10 clients who are used repeatedly. 
% % The inner stepsizes for LFedNest, FedNest, FBO-AggITD are 0.002, 0.01, 0.005, respectively. 
% % The outer stepsizes for LFedNest, FedNest, FBO-AggITD are 0.005, 0.03, 0.04, respectively. 
% % The rest parameters are chosen following~\cite{tarzanagh2022fednest, xiao2023communication}. 
% % We take stepsizes of both SimFBO and ShroFBO to update $y, v, x$ as 0.03, 0.02, 0.01, respectively. 

% \begin{figure}[ht]
%     \centering
%     % \vspace{-0.2cm}
% \includegraphics[width=0.45\textwidth]{figures/MNIST-MLP-iidtime.pdf}
%     % \hspace{-0.8cm}
% \includegraphics[width=0.45\textwidth]{figures/MNIST-MLP-noniidtime.pdf}
%         % \vspace{-0.2cm}
%     \caption{Time Comparison among different algorithms.} 
%     \label{time}
%     % \vspace{-0.35cm}
% \end{figure}

\section{Notations}
For notational convenience, we define 
\begin{align*}
\widetilde{F}(x,y):= \sum_{i=1}^n w_i f_i(x,y),\;
\widetilde{G}(x,y) := \sum_{i=1}^n w_i g_i(x,y),\; \widetilde{R}(w,y,v) := \sum_{i=1}^n w_i R_i(x,y,v).
\end{align*}
For SimBFO, the problem we solve here is 
\begin{align}
&\min_{x\in\mathbb{R}^{p}} \widetilde{\Phi}(x)=\widetilde{F}(x, \widetilde{y}^*(x)) : =\sum_{i=1}^{n} w_if_i(x,\widetilde{y}^*(x)) = \sum_{i=1}^{n} w_i\mathbb{E}_{\xi}\big[f_i(x,\widetilde{y}^*(x);\xi)\big] \nonumber \\
&\;\;\mbox{s.t.} \; \widetilde{y}^*(x)= \argmin_{y\in\mathbb{R}^q} \widetilde{G}(x, y) : =\sum_{i=1}^{n} w_ig_i(x,y) = \sum_{i=1}^{n} w_i\mathbb{E}_{\zeta}\big[g_i(x,y;\zeta)\big]. \nonumber
% &\min_{x\in\mathbb{R}^{p}} \Phi(x)=f(x, y^*(x)) : =
% \mathbb{E}_{\xi}  p_if(x,y^*(x);\xi)  \nonumber \\
% &\;\;\mbox{s.t.} \; y^*(x)= \argmin_{y\in\mathbb{R}^q} g(x,y):=\mathbb{E}_{\zeta} \left[g(x,y^*(x);\zeta)\right] 
\end{align}
Similarly, we define 
\begin{align}
\widetilde{\Phi}(x) := \widetilde{F}(x,\widetilde{y}^*),\; \nabla\widetilde{\Phi}(x) := \sum_{i=1}^n w_i \bar{\nabla}f(x, \widetilde{y}^*, \widetilde{v}^*)
\end{align}
where $\widetilde{y}^* = \argmin_y \widetilde{G}(x,y)$ and $\widetilde{v}^* = \argmin_v \widetilde{R}(x,\widetilde{y}^*,v)$.

\noindent
We can see that $\widetilde{y}^*$ and $\widetilde{v}^*$ are unique due to the strong convexity of $g_i(x,y)$ and $R_i(x,y,v)$. 
Client updates are aggregated to compute $\{h^{(t)}_{x,i}, h^{(t)}_{y,i}, h^{(t)}_{v,i}\}$ as 
\begin{align}
    h_{y,i}^{(t)} &= \frac{1}{\|a_i^{(t)}\|_1}\sum_{k=1}^{\tau_i^{(t)}}a_i^{(t,k)}\nabla_yg_i(x_i^{(t,k)},y_i^{(t,k)}; \zeta_i^{(t,k)}), \nonumber\\ 
    h_{v,i}^{(t)} &= \frac{1}{\|a_i^{(t)}\|_1}\sum_{k=1}^{\tau_i^{(t)}}a_i^{(t,k)}\nabla_vR_i(x_i^{(t,k)},y_i^{(t,k)},v_i^{(t,k)}; \psi_i^{(t,k)}), \nonumber\\
    h_{x,i}^{(t)} &= \frac{1}{\|a_i^{(t)}\|_1}\sum_{k=1}^{\tau_i^{(t)}}a_i^{(t,k)}\bar{\nabla}f_i(x_i^{(t,k)},y_i^{(t,k)},v_i^{(t,k)}; \xi_i^{(t,k)}), \nonumber
\end{align}
and their expectations are 
\begin{align}
    \widetilde{h}_{y,i}^{(t)} &= \mathbb{E}[h_{y,i}^{(t)}] = \frac{1}{\|a_i^{(t)}\|_1}\sum_{k=1}^{\tau_i^{(t)}}a_i^{(t,k)}\nabla_yg_i(x_i^{(t,k)},y_i^{(t,k)}), \nonumber\\
    \widetilde{h}_{v,i}^{(t)} &= \mathbb{E}[h_{v,i}^{(t)}] =\frac{1}{\|a_i^{(t)}\|_1}\sum_{k=1}^{\tau_i^{(t)}}a_i^{(t,k)}\nabla_vR_i(x_i^{(t,k)},y_i^{(t,k)},v_i^{(t,k)}), \nonumber\\
    \widetilde{h}_{x,i}^{(t)} &= \mathbb{E}[h_{x,i}^{(t)}] =\frac{1}{\|a_i^{(t)}\|_1}\sum_{k=1}^{\tau_i^{(t)}}a_i^{(t,k)}\bar{\nabla}f_i(x_i^{(t,k)},y_i^{(t,k)},v_i^{(t,k)}) \nonumber
\end{align}
for all $t \in \{0,1,...,T-1\}$, $k \in \{0,1,...,\tau_i-1\}$ and $i \in \{1,2,...,n\}$.

\noindent
To ensure the robustness of server updates, we set $\alpha_{\min} \leq a_i^{(t,k)} \leq \alpha_{\max}$ for all $t= 0, 1, ..., T$, $i = 1,2, ... n$ and $k = 0, 1, ..., \tau_i-1$; we also set $\frac{\beta_{\min}}{n} \leq w_i \leq \frac{\beta_{\max}}{n}$ and $\frac{\beta_{\min}'}{n} \leq p_i \leq \frac{\beta_{\max}'}{n}$ for all $i = 1,2, ... n$.

\noindent
At global iteration $t$, the server samples $|C^{(t)}|$ clients without replacement (\textbf{WOR}) uniformly at random. 
On the server side, the aggregated client $i$ update is weighed by $\widetilde{w}_i = \frac{n}{|C^{(t)}|}w_i$. 
The aggregates $\{h^{(t)}_y,h^{(t)}_v, h^{(t)}_x\}$ computed at the server are of the form
\begin{align}
    h^{(t)}_y = \sum_{i \in C^{(t)}}\widetilde{w}_ih^{(t)}_{y,i},\quad 
    h^{(t)}_v = \sum_{i \in C^{(t)}}\widetilde{w}_ih^{(t)}_{v,i}, \quad
    h^{(t)}_x = \sum_{i \in C^{(t)}}\widetilde{w}_ih^{(t)}_{x,i}, \nonumber
\end{align}
and we also have partial clients expectation as 
\begingroup
\allowdisplaybreaks
\begin{align}
    \mathbb{E}_{C_{(t)}}\big[h^{(t)}_y\big] &= \mathbb{E}_{C_{(t)}}\bigg[\sum_{i=1}^n \mathbb{I}(i\in C^{(t)})\widetilde{w}_ih^{(t)}_{y,i}\bigg] = \sum_{i=1}^n w_ih^{(t)}_{y,i}, \nonumber\\
    \mathbb{E}_{C_{(t)}}\big[h^{(t)}_v\big] &= \mathbb{E}_{C_{(t)}}\bigg[\sum_{i=1}^n \mathbb{I}(i\in C^{(t)})\widetilde{w}_ih^{(t)}_{v,i}\bigg] = \sum_{i=1}^n w_ih^{(t)}_{v,i}, \nonumber\\
    \mathbb{E}_{C_{(t)}}\big[h^{(t)}_x\big] &= \mathbb{E}_{C_{(t)}}\bigg[\sum_{i=1}^n \mathbb{I}(i\in C^{(t)})\widetilde{w}_ih^{(t)}_{x,i}\bigg] = \sum_{i=1}^n w_ih^{(t)}_{x,i}.\nonumber
\end{align}
\endgroup
Generally, the expectations we use contain both the expectations of samples and the expectations of clients. 
And in analysis, we simply define $|C^{(t)}| = P$ for all $t$. And server updates $y^{(t+1)}$, $v^{(t+1)}$, $x^{(t+1)}$ as 
\begin{align}
    y^{(t+1)} &= y^{(t)} - \rho^{(t)}\gamma_y h^{(t)}_y, \nonumber\\
    v^{(t+1)} &= \mathcal{P}_{r}\big(v^{(t)} - \rho^{(t)}\gamma_v h_{v}^{(t)}\big),\nonumber\\ 
    x^{(t+1)} &= x^{(t)} - \rho^{(t)}\gamma_x h_{x}^{(t)},\nonumber
\end{align}
where the auxiliary projection function is defined as $\mathcal{P}_{r}(v) := \min\{1, \frac{r}{\|v\|}\}v$ and $r = \frac{L_f}{\mu_g}$ is the server side auxiliary projection radius. To simplify the problem, we set $a_i^{(t,k)}$ such that $\rho^{(t)} \in [\frac{1}{2}\bar{\rho}, \frac{3}{2}\bar{\rho}]$ and $c_{a}'\bar{\tau}\alpha_{\min} \leq \|a_i^{(t)}\|_1 \leq c_{a}\bar{\tau}\alpha_{\max}$ for some positive constant $c_a$, $c_{a}'$ and $i \in C^{(t)}$, where $\bar{\rho}:= \frac{1}{T}\sum_{t=0}^{T-1}\rho^{(t)}$ and $\bar{\tau}:=\sum_{i=1}^n \tau_i$. 

\section{Proofs of Preliminary Lemmas}
% \begin{lemma}[Perturbed Strong Convexity~\cite{karimireddy2020scaffold}]\label{lm:perturbedSC}
% Under Assumptions~\ref{as:diffandSC} and~\ref{as:Lipschitz}, for $\mu_g$-strongly convex and $L_g$-smooth inner function $g(x,y)$, we have 
% \begin{align}
%     \langle \nabla g(x,y_0), y_1 - y_2 \rangle \geq g(x,y_1) - g(x,y_2) + \frac{\mu_g}{4}\|y_1 - y_2\|^2 - L_g\|y_1 - y_0\|^2   \nonumber
% \end{align}
% for all $y_0, y_1, y_2$ in the domain. 
% \end{lemma}

\begin{lemma}[Boundedness of $v^*$]\label{lm:boundofv}
Under Assumptions~\ref{as:diffandSC} and~\ref{as:Lipschitz}, we have  
for $v^*$ in~\cref{def:R},
% \begin{align}
    $\norm{v^*}^2 \leq \frac{L^2_{f}}{\mu^2_g}.$
% \end{align}
\end{lemma}
\begin{proof}
Remind that we define $v^* = \arg \min_v R(x,y^*,v)$, then we have
% From \cref{def:R}, we have
\begin{align}
    \norm{v^*}^2 = \norm{\big[\nabla_{yy}^2G(x,y^*)\big]^{-1}\nabla_y F(x,y^*)}^2 \leq \norm{\big[\nabla_{yy}^2G(x,y^*)\big]^{-1}}^2 \norm{\nabla_y F(x,y^*)}^2 \overset{(a)}{\leq} r^2, \nonumber
\end{align}
where (a) follows Assumptions~\ref{as:diffandSC},~\ref{as:Lipschitz} and defines $r := \frac{L_{f}}{\mu_g}$. Then, the proof is complete. 
\end{proof}

\begin{lemma}[Boundedness of local $v$]\label{lm:boundofvi} 
Under Assumptions~\ref{as:diffandSC} and~\ref{as:Lipschitz}, for each global iteration $t$, client $i$, and local iteration $k=1,2,..., \tau_i$, we have 
\begin{align}
     r_i := \|v_i^{(t,k)}\| \leq \Big( 1 + \frac{\alpha_{\max}}{\alpha_{\min}} \Big) r =: r_{\max}, \nonumber
\end{align}
where $r = \frac{L_f}{\mu_g}$ is the server side auxiliary projection radius. 
\end{lemma}
\begin{proof}
By the local update rule of $v_i^{(t,k)}$ from step 6 in Algorithm \ref{alg:main}, we have 
\begin{align}
    v_i^{(t,k)} &= v_i^{(t,k-1)} - \eta_v a_i^{(t,k-1)} \nabla_v R_i(x_i^{(t,k-1)}, y_i^{(t,k-1)}, v_i^{(t,k-1)}; \psi_i^{(t,k-1)}) \nonumber\\
    &= \big(I - \eta_v a_i^{(t,k-1)} \nabla^2_{yy} g_i(x_i^{(t,k-1)}, y_i^{(t,k-1)}; \psi_i^{(t,k-1)})\big)v_i^{(t,k-1)} \nonumber\\
    &\quad \ + \eta_v a_i^{(t,k-1)} \nabla_y f_i(x_i^{(t,k-1)}, y_i^{(t,k-1)}; \psi_i^{(t,k-1)}). \nonumber
\end{align}
By taking $l_2$ norm, we have 
\begin{align}\label{eq:boundofvi1}
    \|v_i^{(t,k)}\| \leq (1-\eta_v a_i^{(t,k-1)}\mu_g)\|v_i^{(t,k-1)}\| + \eta_va_i^{(t,k-1)}L_f.
\end{align}
Telescope \cref{eq:boundofvi1} over $j \in [0, ... ,k-1]$, and we have 
\begin{align}
    \|v_i^{(t,k)}\| &\leq (1-\eta_v a_i^{(t,k-1)}\mu_g)\|v_i^{(t,k-1)}\| + \eta_va_i^{(t,k-1)}L_f \nonumber\\
    & \leq (1-\eta_v \alpha_{\min}\mu_g)^k \|v_i^{(t,0)}\| + \sum_{j=0}^{k-1}(1-\eta \alpha_{\min}\mu_g)^j\eta_v \alpha_{\max}L_f \nonumber\\
    & \leq (1-\eta_v \alpha_{\min}\mu_g)^k \|v_i^{(t,0)}\| + \frac{\alpha_{\max}L_f}{\alpha_{\min}\mu_g} \nonumber\\
    & \leq \Big( 1 + \frac{\alpha_{\max}}{\alpha_{\min}} \Big) r.\nonumber
\end{align}
Then, the proof is complete. 
\end{proof}
\noindent
Since both $\alpha_{\max}$ and $\alpha_{\min}$ are preset constants, the bound of local $v_i$ has the same order $\kappa$ as server-side auxiliary projection radius $r$. 

\begin{lemma}[Basic properties of linear system function R]\label{lm:propertiesofR}
Under Assumptions~\ref{as:diffandSC} and~\ref{as:Lipschitz}, we have $R_i(x,y,v)$ is $\mu_g$-strongly convex w.r.t $v$ and $\nabla_v R_i(x,y,v)$ is $L_R$-Lipschitz continuous w.r.t $(x,y)$, where we define $L_R^2 := 2(L^2_2r^2_{\max} + L_1^2)$.
\end{lemma}
\begin{proof}
The strong convexity can be easily observed since $\nabla_{vv}^2 R(x,y,v) = \nabla_{yy}^2 g(x,y) \succeq \mu_g I$. 
By using the definition of $R_i(x,y,v)$ and assumptions~\ref{as:diffandSC} and~\ref{as:Lipschitz}, we have
\begin{align}
    \|\nabla_v& R_i(x_1,y_1,v) - \nabla_v R_i(x_2,y_2,v)\|^2 \nonumber \\
    & \leq 2\|\big(\nabla_{yy}^2 g(x_1,y_1) - \nabla_{yy}^2 g(x_2,y_2)\big)v\|^2 + 2\|\nabla_y f(x_1,y_1) - \nabla_y f(x_2,y_2)\|^2\nonumber \\
    & \leq 2(L^2_2r^2_{\max} + L_1^2)(\|x_1-x_2\|^2 + \|y_1-y_2\|^2) = L_R^2(\|x_1-x_2\|^2 + \|y_1-y_2\|^2). \nonumber 
\end{align}
Then, the proof is complete.
\end{proof}

\begin{lemma}[\cite{ghadimi2018approximation} lemma 2.2, \cite{chen2021closing} lemma 2 and extensions]\label{lm:3ieq}
Under Assumptions~\ref{as:diffandSC},~\ref{as:Lipschitz} and ~\ref{as:varaince}, we have, for all $x, x_1 , x_2 \in \mathbb{R}^{d_{x}}$ and $y \in \mathbb{R}^{d_{y}}$,  
\begin{align}
    \|\nabla f(x,y) - \nabla \Phi(x)\|^2 \leq \widetilde{L}^2\|y - y^*(x)\|^2 ,&\ \ 
    \|\nabla \Phi(x_1) - \nabla \Phi(x_2)\|^2 \leq L_\Phi^2\|x_1 - x_2\|^2, \nonumber \\
    \|y^*(x_1) - y^*(x_2)\|^2 \leq L_y^2\|x_1 - x_2\|^2 ,& \ \ 
    \|\nabla y^*(x_1) - \nabla y^*(x_2)\|^2 \leq L_{yx}^2\|x_1 - x_2\|^2, \nonumber\\
    \|v^*(x_1) - v^*(x_2)\|^2 \leq L_v^2\|x_1 - x_2\|^2, &\ \ 
    \|\nabla v^*(x_1) - \nabla v^*(x_2)\|^2 \leq L_{vx}^2\|x_1 - x_2\|^2 \nonumber
\end{align}
where the constants are given by 
% \red{double check after all}
\begin{align}
    \widetilde{L} = L_{1} + \frac{L_{1}^2}{\mu_g} + L_f\bigg(\frac{L_2}{\mu_g}+&\frac{L_1L_2}{\mu_g^2}\bigg), 
    \ \ \  L_{\Phi} = \widetilde{L}+\frac{\widetilde{L}L_1}{\mu_g}, \nonumber \\
    L_y = \frac{L_1}{\mu_g}, 
    \ \ \  L_v = \bigg(\frac{2L_1^2}{\mu_g^2}& + \frac{2L_f^2L_2^2}{\mu_g^4}\bigg)^{\frac{1}{2}}\big(1+L_y^2\big)^{\frac{1}{2}}, \nonumber \\
    L_{yx} = \frac{L_2+L_2L_y}{\mu_g} + \frac{L_1(L_2+L_2L_y)}{\mu_g^2},\quad 
    L_{vx} =& \frac{2}{\mu_g}\Big(\big(L_2^2 + r^2L_3^2 + L_2^2L_v^2 \big)\big(1 + L_y^2\big) + L_2^2L_v^2 \Big)^{\frac{1}{2}}. 
\end{align}
\end{lemma}
\begin{proof}
The proof of the first 3 inequalities is provided in~\cite{ghadimi2018approximation}. For the fifth inequality, we have
\begin{align}
    \big\|&v^*(x_1) - v^*(x_2)\big\|^2 \nonumber \\
    &= \Big\|\big[\nabla^2_{yy}G\big(x_1, y^*(x_1)\big)\big]^{-1}\nabla_y F\big(x_1, y^*(x_1)\big) - \big[\nabla^2_{yy}G\big(x_2, y^*(x_2)\big)\big]^{-1}\nabla_y F\big(x_2, y^*(x_2)\big)\Big\|^2 \nonumber \\
    & \leq 2\Big\|\big[\nabla^2_{yy}G\big(x_1, y^*(x_1)\big)\big]^{-1}\Big(\nabla_y F\big(x_1, y^*(x_1)\big) -  \nabla_y F\big(x_2, y^*(x_2)\big)\Big)\Big\|^2 \nonumber \\
    & \quad + 2\Big\|\Big(\big[\nabla^2_{yy}G\big(x_1, y^*(x_1)\big)\big]^{-1} - \big[\nabla^2_{yy}G\big(x_2, y^*(x_2)\big)\big]^{-1}\Big)\nabla_y F\big(x_2, y^*(x_2)\big)\Big\|^2 \nonumber \\
    & \leq \frac{2L_1^2}{\mu_g^2}\big(\big\|x_1 - x_2\big\|^2 + \big\|y^*(x_1) - y^*(x_2)\big\|^2\big) \nonumber \\
    & \quad + 2L_f^2\Big\|\big[\nabla^2_{yy}G\big(x_1, y^*(x_1)\big)\big]^{-1}\big[\nabla^2_{yy}G\big(x_1, y^*(x_1)\big)-\nabla^2_{yy}G\big(x_2, y^*(x_2)\big)\big]\big[\nabla^2_{yy}G\big(x_2, y^*(x_2)\big)\big]^{-1}\Big\|^2 \nonumber \\
    & \leq \bigg(\frac{2L_1^2}{\mu_g^2} + \frac{2L_f^2L_2^2}{\mu_g^4}\bigg)\big(\big\|x_1 - x_2\big\|^2 + \big\|y^*(x_1) - y^*(x_2)\big\|^2\big) \nonumber \\
    & \leq \bigg(\frac{2L_1^2}{\mu_g^2} + \frac{2L_f^2L_2^2}{\mu_g^4}\bigg)\big(1+L_y^2\big)\big\|x_1 - x_2\big\|^2. \nonumber
\end{align}
And for the sixth inequality, we have
\begin{align}
    \nabla_v R\big(x, y^*(x), v^*(x)\big) = \nabla_{yy}^2G\big(x, y^*(x)\big)v^*(x) - \nabla_yF\big(x,y^*(x)\big) = 0, \nonumber
\end{align}
which implies that 
\begin{align}\label{eq:nablaxv}
    \frac{\partial_x \nabla_v R\big(x, y^*(x), v^*(x)\big)}{\partial x} &= \big[v^*(x)\big]^T\Big[\nabla_{yyx}^3G\big(x,y^*(x)\big) + \nabla_{yyy}^3G\big(x,y^*(x)\big)\nabla y^*(x)\Big] \nonumber \\
    & \quad + \nabla_{yy}^2G\big(x,y^*(x)\big)\nabla v^*(x)  - \nabla_{yx}^2F\big(x,y^*(x)\big) - \nabla_{yy}^2F\big(x,y^*(x)\big)\nabla y^*(x)\nonumber \\
    & = \mathbf{0}_{d^y, d^x}
\end{align}
and 
\begin{align}\label{eq:nablayv}
    \frac{\partial_y \nabla_v R(x, y^*(x), v^*(x))}{\partial y} =& \big[v^*(x)\big]^T\nabla_{yyy}^3G\big(x,y^*(x)\big) - \nabla_{yy}^2 F\big(x, y^*(x)\big)= \mathbf{0}_{d^y, d^y}.
\end{align}
By combining \cref{eq:nablaxv} and \cref{eq:nablayv}, we have 
\begin{align}
    \nabla_{yy}^2G\big(x,y^*(x)\big)\nabla_x v^*\big(x,y^*(x)\big) =& \nabla_{yx}^2 F\big(x, y^*(x)\big) - \big[v^*(x)\big]^T \nabla_{yyx}^3G\big(x,y^*(x)\big). 
\end{align}
Then we can get that 
\begin{align}
    \nabla_{yy}^2G&\big(x_1,y^*(x_1)\big)\nabla_x v^*\big(x_1,y^*(x_1)\big) - \nabla_{yy}^2G\big(x_2,y^*(x_2)\big)\nabla_x v^*\big(x_2,y^*(x_2)\big) \nonumber \\
    & = \nabla_{yy}^2G\big(x_1,y^*(x_1)\big)\nabla_x v^*\big(x_1,y^*(x_1)\big) - \nabla_{yy}^2G\big(x_2,y^*(x_2)\big)\nabla_x v^*\big(x_1,y^*(x_1)\big) \nonumber \\
    &\quad  + \nabla_{yy}^2G\big(x_2,y^*(x_2)\big)\nabla_x v^*\big(x_1,y^*(x_1)\big) - \nabla_{yy}^2G\big(x_2,y^*(x_2)\big)\nabla_x v^*\big(x_2,y^*(x_2)\big) \nonumber \\
    & = \Big[\nabla_{yx}F\big(x_1, y^*(x_1)\big) - \nabla_{yx}F\big(x_2, y^*(x_2)\big)\Big] \nonumber \\
    % & \quad  + \nabla_{yy}^2 G\big(x_1, y^*(x_1)\big)\nabla y^*(x_1) - \nabla_{yy}^2 G\big(x_2, y^*(x_2)\big)\nabla y^*(x_2) \nonumber \\ 
    &\quad - \Big( \big[v^*(x_1)\big]^T \nabla_{yyx}^3G\big(x_1,y^*(x_1)\big) - \big[v^*(x_2)\big]^T \nabla_{yyx}^3G\big(x_2,y^*(x_2)\big) \Big) \nonumber 
    % & - \Big(\nabla_{yy}^2F\big(x_1,y^*(x_1)\big)\nabla y^*(x_1) - \nabla_{yy}^2F\big(x_2,y^*(x_2)\big)\nabla y^*(x_2)\Big) \nonumber
\end{align}
By taking the norm and using Assumption \ref{as:Lipschitz}, we have 
\begin{align}
    \bigg\|\nabla_{yy}^2&G\big(x_2, y^*(x_2)\big)\Big[\nabla_x v^*(x_1) - \nabla_x v^*(x_2)\big)\Big]\bigg\|^2 \nonumber \\
    &\leq 4\bigg\| \nabla_{xy}^2F\big(x_1, y^*(x_1)\big) - \nabla_{xy}^2F\big(x_2, y^*(x_2)\big)\bigg\|^2 \nonumber \\
    % &\quad + 8\bigg\| \big[\nabla_{yy}^2 G\big(x_1, y^*(x_1)\big) - \nabla_{yy}^2 G\big(x_2, y^*(x_2)\big)\big]\nabla y^*(x_2) \bigg\|^2 \nonumber \\
    % &\quad + 8\bigg\| \nabla_{yy}^2 G\big(x_1, y^*(x_1)\big)\big[\nabla y^*(x_1) - \nabla y^*(x_2)\big] \bigg\|^2 \nonumber \\
    &\quad + 4\bigg\| \big[v^*(x_1) - v^*(x_2)\big]^T \nabla_{yyx}^3G\big(x_2,y^*(x_2)\big) \bigg\|^2 \nonumber \\
    &\quad + 4\bigg\| \big[v^*(x_1)\big]^T \Big[ \nabla_{yyx}^3G\big(x_1,y^*(x_1)\big) -  \nabla_{yyx}^3G\big(x_2,y^*(x_2)\big) \Big] \bigg\|^2 \nonumber \\
    % &\quad + 8\bigg\| \nabla_{yy}^2G\big(x_1,y^*(x_1)\big)\big[\nabla y^*(x_1) - \nabla y^*(x_2)\big] \bigg\|^2 \nonumber \\
    % &\quad + 8\bigg\| \Big[\nabla_{yy}^2G\big(x_1,y^*(x_1)\big) - \nabla_{yy}^2G\big(x_2,y^*(x_2)\big)\Big]\nabla y^*(x_2) \bigg\|^2 \nonumber \\
    &\quad + 4\bigg\| \Big[\nabla_{yy}^2G\big(x_1,y^*(x_1)\big) - \nabla_{yy}^2G\big(x_2,y^*(x_2)\big)\Big]\nabla_x v^*(x_1) \bigg\|^2 \nonumber \\ 
    & \leq 4\Big(L_2^2 + r^2L_3^2 + L_2^2L_v^2 \Big)\Big(\big\|x_1 - x_2\big\|^2 + \big\|y^*(x_1) - y^*(x_2)\big\|^2\Big)  + 4\Big(L_2^2L_v^2 \Big)\big\|x_1 - x_2\big\|^2 \nonumber \\
    & \leq 4\Big(\big(L_2^2 + r^2L_3^2 + L_2^2L_v^2 \big)\big(1 + L_y^2\big) + L_2^2L_v^2 \Big)\big\|x_1 - x_2\big\|^2
\end{align}
By using the strong convexity of $g_i$ in Assumption \ref{as:diffandSC}, we have
\begin{align}
    \Big\|\nabla_x v^*\big(x_1,y^*(x_1)\big) - \nabla_x v^*\big(x_2,y^*(x_2)\big)\Big\|^2 \leq L_{vx}^2 \big\|x_1 - x_2\big\|^2,
\end{align}
where $L_{vx} = \frac{2}{\mu_g}\Big(\big(L_2^2 + r^2L_3^2 + L_2^2L_v^2 \big)\big(1 + L_y^2\big) + L_2^2L_v^2 \Big)^{\frac{1}{2}}$. 
Then the proof is complete. 
\end{proof}
\begin{lemma}[Global Heterogeneity Extension]\label{lm:globalheterogeneity}
For any set of non-negative weight $\{w_i\}_{i=1}^n$ such that $\sum_{i=1}^n w_i = 1$, under Assumption \ref{as:Lipschitz} and \Cref{lm:boundofvi}, we have the bounds of global heterogeneity of $\nabla g_i(x,y)$, $\nabla R_i(x,y,v)$ and $\nabla f_i(x,y,v)$ as 
\begin{align}
    \sum_{i=1}^n w_i\|\nabla_y g_i(x, y)\|^2 \leq \beta_{gh}^2& L_1^2\big\|y - y^*(x)\big\|^2 + \sigma_{gh}^2, \nonumber \\
    \sum_{i=1}^n w_i\|\nabla_v R_i(x, y, v)\|^2 &\leq 2r_{\max}^2L_1^2 + 2L_f^2, \nonumber \\
    \sum_{i=1}^n w_i\|\bar{\nabla}f_i(x, y, v)\|^2 &\leq 2r_{\max}^2L_1^2 + 2L_f^2 \nonumber 
\end{align}
for all $i \in \{1,...,n\}$. 
% \red{$\sigma_{G_g}$,$\sigma_{G_R}$,$\sigma_{G_f}$, change}
\begin{proof}
    Under Assumption \ref{as:globalheter}, we have 
    \begin{align}
        \sum_{i=1}^n w_i\|\nabla_y g_i(x, y)\|^2 &\leq \beta_{gh}^2 \bigg\|\sum_{i=1}^n w_i\Big[\nabla_y g_i(x, y) - \nabla_y g_i\big(x, y^*(x)\big)\Big]\bigg\|^2 + \sigma_{gh}^2 \nonumber \\
        & \leq \beta_{gh}^2 L_1^2\big\|y - y^*(x)\big\|^2 + \sigma_{gh}^2. \nonumber
    \end{align}
For $R_i$ and $f_i$, we can easily have 
\begingroup
\allowdisplaybreaks
\begin{align}
    \|\nabla_v R_i(x, y, v)\|^2 &\leq 2\|\nabla_{yy}g(x,y)v\|^2 + 2\|\nabla_y f(x,y)\|^2 \leq 2r_{\max}^2L_1^2 + 2L_f^2, \nonumber \\
    \|\bar{\nabla}f_i(x, y, v)\|^2 &\leq 2\|\nabla_{xy}g(x,y)v\|^2 + 2\|\nabla_x f(x,y)\|^2 \leq 2r_{\max}^2L_1^2 + 2L_f^2. \nonumber
\end{align}
\endgroup
Then the proof is complete. 
\end{proof}

\begin{lemma}\label{lm:CDcombo}
Under Assumption \ref{as:Lipschitz}, we have the following bounds
\begin{align}
    \big\| \bar{\nabla} f_i(x^{(t)}, y^{(t)}, v^{(t)}) &- \bar{\nabla} f_i(x_i^{(t,k)}, y_i^{(t,k)}, v_i^{(t,k)}) \big\|^2 \leq \Delta_{f,i}^{(t,k)}, \nonumber \\ 
    \big\| \nabla g_i(x^{(t)}, y^{(t)}) &- \nabla g_i(x_i^{(t,k)}, y_i^{(t,k)}) \big\|^2 \leq \Delta_{g,i}^{(t,k)}, \nonumber \\ 
    \big\| \nabla_v R_i(x^{(t)}, y^{(t)}, v^{(t)}) &- \nabla R_i(x_i^{(t,k)}, y_i^{(t,k)}, v_i^{(t,k)}) \big\|^2 \leq \Delta_{R,i}^{(t,k)}, \nonumber
\end{align}
where we define the combinations of client drift from as 
\begin{align}
    \Delta_{f_i}^{(t,k)} = \Delta_{R_i}^{(t,k)} := 3\big(L^2_1 &+ r^2L^2_2\big)\Big[\big\|x_i^{(t,k)} - x^{(t)}\big\|^2 + \big\|y_i^{(t,k)} - y^{(t)}\big\|^2\Big] + 3L^2_1 \big \| v_i^{(t,k)} - v^{(t)}\big \|^2 \nonumber \\
    \Delta_{g_i}^{(t,k)} &:= L_1^2\Big[\big\|x_i^{(t,k)} - x^{(t)}\big\|^2 + \big\|y_i^{(t,k)} - y^{(t)}\big\|^2\Big].  \nonumber
\end{align}
\end{lemma}
\begin{proof}
According to the definition, we have 
\begin{align}
    \big\| \bar{\nabla} f_i(x^{(t)}, & y^{(t)}, v^{(t)}) - \bar{\nabla} f_i(x_i^{(t,k)}, y_i^{(t,k)}, v_i^{(t,k)}) \big\|^2 \nonumber \\
    & \leq \big\| \nabla_x f_i(x^{(t)}, y^{(t)}, v^{(t)}) - \nabla_x f_i(x_i^{(t,k)}, y_i^{(t,k)}, v_i^{(t,k)}) \nonumber \\
    & \quad - \nabla_{xy}^2 g_i(x^{(t)}, y^{(t)}, v^{(t)})v^{(t)} + \nabla_{xy}^2 g_i(x_i^{(t,k)}, y_i^{(t,k)}, v_i^{(t,k)})v_i^{(t,k)} \big\|^2 \nonumber \\
    & \leq 3\big\| \nabla_x f_i(x^{(t)}, y^{(t)}, v^{(t)}) - \nabla_x f_i(x_i^{(t,k)}, y_i^{(t,k)}, v_i^{(t,k)}) \big\|^2 \nonumber \\
    & \quad + 3\big\| \big(\nabla_{xy}^2 g_i(x^{(t)}, y^{(t)}, v^{(t)}) - \nabla_{xy}^2 g_i(x_i^{(t,k)}, y_i^{(t,k)}, v_i^{(t,k)})\big)v^{(t)} \big\|^2 \nonumber \\
    & \quad + 3\big\| \nabla_{xy}^2 g_i(x_i^{(t,k)}, y_i^{(t,k)}, v_i^{(t,k)})\big(  v^{(t)} - v_i^{(t,k)}\big) \big\|^2 \nonumber \\
    & \overset{(a)}{\leq} 3\big(L^2_1 + r^2L^2_2\big)\Big[\big\|x_i^{(t,k)} - x^{(t)}\big\|^2 + \big\|y_i^{(t,k)} - y^{(t)}\big\|^2\Big] + 3L^2_1 \big \| v_i^{(t,k)} - v^{(t)}\big \|^2, \nonumber
\end{align} 
where (a) follows from Assumption\ref{as:Lipschitz} and \Cref{lm:boundofv}. Similarly, we can easily get the results of $\nabla g_i$ and $\nabla_v R_i$, then the proof is complete. 
\end{proof}

\end{lemma}
\section{Proofs of \Cref{th:theorem1}}\label{proofoftheorem}
\subsection{Descent in Objective Function}
\begin{lemma}\label{lm:obj_function}
Under Asusmption \ref{as:diffandSC}, for non-convex and smooth $\widetilde{\Phi}(x)$, the consecutive iterates of Algorithm \ref{alg:main} satisfy:
\begin{align}
    \mathbb{E}[\widetilde{\Phi}(x^{(t+1)})]& - \mathbb{E}[\widetilde{\Phi}(x^{(t)})] \nonumber\\
    \leq& -\frac{\rho^{(t)}\gamma_x}{2}\Big(\mathbb{E}\|\nabla \widetilde{\Phi}(x^{(t)})\|^2 + \mathbb{E}\Big\|\sum_{i=1}^n w_i \widetilde{h}^{(t)}_{x,i}\Big\|^2 \Big) \nonumber\\
    &+ \frac{\rho^{(t)}\gamma_x}{2} \Big(6(L^2_1 + r^2L^2_2)\mathbb{E}\|y^{(t)} - \widetilde{y}^*(x^{(t)})\|^2 + 3L^2_1\mathbb{E}\|v^{(x^{(t)})} - \widetilde{v}^*(x^{(t)})\|^2 \Big) \nonumber\\
    &+ \frac{3\rho^{(t)}\gamma_x}{2} \sum_{i=1}^n w_i \frac{1}{\|a_i^{(t)}\|_1}\sum_{k=0}^{\tau_i-1}a^{(t,k)}_i \Delta^{(t,k)}_{f,i} + \frac{L_{\Phi}^2(\rho^{(t)}\gamma_x)^2}{2} \Big\|\sum_{i \in C^{(t)}} \widetilde{w}_i h^{(t)}_{x,i}\Big\|^2 \nonumber
\end{align}
for all $t \in \{0,1,...,T-1\}$. 
\end{lemma}
\begin{proof}
Using the $L_{\Phi}$ in \Cref{lm:3ieq}, we have 
\begingroup
\allowdisplaybreaks
\begin{align}\label{eq:obj_function1}
    \mathbb{E}\big[\widetilde{\Phi}(x^{(t+1)}&)\big] \leq \mathbb{E}\big[\widetilde{\Phi}(x^{(t)})\big] - \mathbb{E}\Big\langle\nabla \widetilde{\Phi}(x^{(t)}), \rho^{(t)}\gamma_x \sum_{i \in C^{(t)}}\widetilde{w}_i h_{x,i}^{(t,k)} \Big\rangle \nonumber \\
    &\quad\quad + \frac{\big(\rho^{(t)}\gamma_x\big)^2L_{\Phi}^2}{2}\mathbb{E}\bigg\|\sum_{i \in C^{(t)}}\widetilde{w}_i h_{x,i}^{(t,k)}\bigg\|^2 \nonumber \\
    &\overset{(a)}{=} \mathbb{E}\big[\widetilde{\Phi}(x^{(t)})\big] - \rho^{(t)}\gamma_x\mathbb{E}\Big\langle\nabla \widetilde{\Phi}(x^{(t)}), \sum_{i = 1}^n w_i \widetilde{h}_{x,i}^{(t,k)} \Big\rangle \nonumber \\
    &\quad + \frac{\big(\rho^{(t)}\gamma_x\big)^2L_{\Phi}^2}{2}\mathbb{E}\bigg\|\sum_{i \in C^{(t)}}\widetilde{w}_i h_{x,i}^{(t,k)}\bigg\|^2 \nonumber \\
    & = \mathbb{E}\big[\widetilde{\Phi}(x^{(t)})\big] - \frac{\rho^{(t)}\gamma_x}{2}\mathbb{E}\Bigg[\bigg\|\nabla \widetilde{\Phi}(x^{(t)})\bigg\|^2 + \bigg\|\sum_{i = 1}^n w_i \widetilde{h}_{x,i}^{(t,k)}\bigg\|^2\Bigg] \nonumber \\
    & \quad + \frac{\rho^{(t)}\gamma_x}{2}\mathbb{E}\bigg\| \nabla \widetilde{\Phi}(x^{(t)}) - \sum_{i = 1}^n w_i \widetilde{h}_{x,i}^{(t,k)}\bigg\|^2 + \frac{\big(\rho^{(t)}\gamma_x\big)^2L_{\Phi}^2}{2}\mathbb{E}\bigg\|\sum_{i \in C^{(t)}}\widetilde{w}_i h_{x,i}^{(t,k)}\bigg\|^2,
\end{align}
\endgroup
where (a) holds because clients are selected without replacement. 
For the third part of the right-hand side in \cref{eq:obj_function1}, we have
\begingroup
\allowdisplaybreaks
\begin{align}
\mathbb{E}\bigg\|& \nabla \widetilde{\Phi}(x^{(t)}) - \sum_{i = 1}^n w_i \widetilde{h}_{x,i}^{(t,k)}\bigg\|^2 \nonumber \\
& = \mathbb{E}\bigg\|\sum_{i=1}^n w_i \Big[\bar{\nabla}f_i\big(x^{(t)}, y^*(x^{(t)}), v^*(x^{(t)})\big) - \bar{\nabla}f_i\big(x^{(t)}, y^{(t)}, v^*(x^{(t)})\big) + \bar{\nabla}f_i\big(x^{(t)}, y^{(t)}, v^*(x^{(t)})\big) \nonumber \\
& \qquad\qquad\qquad - \bar{\nabla}f_i\big(x^{(t)}, y^{(t)}, v^{(t)}\big) + \bar{\nabla}f_i(x^{(t)}, y^{(t)}, v^{(t)}) - \widetilde{h}_{x,i}^{(t)}\Big]\bigg\|^2 \nonumber \\
& \leq 3 \mathbb{E}\bigg\|\sum_{i=1}^n w_i \Big[\big(\nabla_x f_i(x^{(t)}, y^*(x^{(t)})) - \nabla_x f_i(x^{(t)}, y^{(t)})\big) \nonumber \\
& \qquad\qquad\qquad - \big(\nabla^2_{xy}g_i(x^{(t)}, y^*(x^{(t)})) - \nabla^2_{xy}g_i(x^{(t)}, y^{(t)})\big)v^*(x)\big)\Big]\bigg\|^2 \nonumber \\
& \quad + 3 \mathbb{E}\Big\|\sum_{i=1}^n w_i \nabla^2_{xy}g_i(x^{(t)}, y^{(t)})\big(v^*(x^{(t)}) - v^{(t)}\big)\Big\|^2 \nonumber \\
% & \quad + 3 \mathbb{E}\Big\|\sum_{i=1}^n w_i \big(\nabla^2_{xy}g_i(x^{(t)}, y^{(t)})- \nabla^2_{xy}g_i(x^{(t)}, y^{(t)})\big)v^*(x)\Big\|^2 \nonumber \\
& \quad + 3 \mathbb{E}\Big\| \sum_{i=1}^n w_i \big(\bar{\nabla}f_i(x^{(t)}, y^{(t)}, v^{(t)}) - \widetilde{h}_{x,i}^{(t)}\big)\Big\|^2 \nonumber \\
&\overset{(a)}{\leq} 6\big(L^2_1 + r^2L_2^2\big) \mathbb{E}\big\|y^*{(x^{(t)})} - y^{(t)}\big\|^2 + 3L^2_1 \mathbb{E}\big\|v^*{(x^{(t)})} - v^{(t)}\big\|^2 \nonumber \\
& \quad + 3\sum_{i=1}^n w_i \sum_{k=0}^{\tau_i-1} \frac{a_i^{(t,k)}}{\|a_i^{(t)}\|_1}\mathbb{E}\big\|\bar{\nabla}f_i(x^{(t)}, y^{(t)}, v^{(t)}) - \bar{\nabla}f_i(x_i^{(t,k)}, y_i^{(t,k)}, v_i^{(t,k)})\big\|^2 \nonumber \\
&\overset{(b)}{\leq} 6\big(L^2_1 + r^2L_2^2\big) \mathbb{E}\big\|y^*{(x^{(t)})} - y^{(t)}\big\|^2 + 3L^2_1 \mathbb{E}\big\|v^*{(x^{(t)})} - v^{(t)}\big\|^2 + 3\sum_{i=1}^n w_i \sum_{k=0}^{\tau_i-1} \frac{a_i^{(t,k)}}{\|a_i^{(t)}\|_1} \Delta_{f,i}^{(t,k)}, \nonumber
\end{align}
\endgroup
where (a) follows from smoothness of $\nabla_x f_i(x,y)$ and $L_2$-Lipschitz continuity of $\nabla^2_{g_{xy}}g(x,y)$ in Assumption \ref{as:Lipschitz} and (b) uses \Cref{lm:CDcombo}. 
\end{proof}

\subsection{Bounds of Client Drifts}
\begin{lemma}\label{lm:boundofCD}
Under Assumption \ref{as:diffandSC}, \ref{as:Lipschitz} and \ref{as:globalheter}, the local iterates client drifts of $y_i^{(t,k)}, v_i^{(t,k)}, x_i^{(t,k)}$ are bounded as 
\begin{align}
    \sum_{i=1}^n w_i\frac{1}{\|a_i^{(t)}\|_1}\sum_{k=1}^{\tau_i-1}a^{(t,k)}_i \mathbb{E}\|x^{(t,k)}_i - x^{(t)}\|^2 &\leq \eta^2_x \bar{\tau}\sigma_{M1}^2, \nonumber \\
    % \sum_{i=1}^n w_i&\frac{1}{\|a_i^{(t)}\|_1}\sum_{k=1}^{\tau_i-1}a^{(t,k)}_i \mathbb{E}\|x^{(t,k)}_i - x^{(t)}\|^2 \leq \eta^2_x \bar{\tau}\Big(\alpha_{\max}^2(\sigma^2_{f}+r_{\max}^2\sigma^2_{gg}) + \alpha_{\max}(L^2 + r^2_{\max}L^2_1) \Big), \nonumber \\
    \sum_{i=1}^n w_i\frac{1}{\|a_i^{(t)}\|_1}\sum_{k=1}^{\tau_i-1}a^{(t,k)}_i \mathbb{E}\|v^{(t,k)}_i - v^{(t)}\|^2 &\leq \eta^2_v \bar{\tau}\sigma_{M1}^2, \nonumber\\
    \sum_{i=1}^n w_i\frac{1}{\|a_i^{(t)}\|_1}\sum_{k=1}^{\tau_i-1}a^{(t,k)}_i \mathbb{E}\|y^{(t,k)}_i - y^{(t)}\|^2 &\leq \frac{\eta_y^2\Bar{\tau}}{1-2\eta_y^2 c_a\Bar{\tau}\alpha_{\max}L_1^2}\bigg[\alpha_{\max}^2 \sigma^2_g + 2 c_a\alpha_{\max}L_1^2\eta^2_x \bar{\tau}\sigma_{M1}^2 \nonumber \\
    &\qquad + 2c_{a}\alpha_{\max} L_1^2\mathbb{E}\big\|y^{(t)} - y^*{(x^{(t)})}\big\|^2 + 2c_{a}\alpha_{\max} \sigma_{gh}^2\bigg] \nonumber
\end{align}
for all $t \in \{0,1,...,T-1\}$, $k \in \{0,1,...,\tau_i-1\}$ and $i \in \{1,2,...,n\}$. We define $\sigma_{M1}^2:= \Big(\alpha_{\max}^2(\sigma^2_{f}+r_{\max}^2\sigma^2_{gg}) + \alpha_{\max}(L_f^2 + r^2_{\max}L^2_1) \Big)$.
And $\eta_y$, $\eta_v$, $\eta_x$ are local stepsizes. 
\end{lemma}
\begin{proof}
For $\big\|y^{(t,k)}_i - y^{(t)}\big\|^2$, we have
\begin{align}\label{eq:boundofCD1}
    &\sum_{i=1}^n w_i\sum_{k = 0}^{\tau_i-1}\frac{a_i^{(t,k)}}{\|a_i^{(t)}\|_1}\mathbb{E}\big\|y^{(t,k)}_i - y^{(t)}\big\|^2 \nonumber \\
    &= \eta_y^2\sum_{i=1}^n w_i\sum_{k = 0}^{\tau_i-1}\frac{a_i^{(t,k)}}{\|a_i^{(t)}\|_1}\mathbb{E}\bigg\|\sum_{j=0}^{k-1}a_i^{(t,j)}\Big(\nabla_y g_i\big(x_i^{(t,j)}, y_i^{(t,j)}; \zeta_i^{(t,j)}\big) \nonumber \\
    & \qquad\qquad\qquad\qquad\qquad\qquad \qquad\qquad- \nabla_y g_i\big(x_i^{(t,j)}, y_i^{(t,j)}\big) + \nabla_y g_i\big(x_i^{(t,j)}, y_i^{(t,j)}\big)\Big)\bigg\|^2 \nonumber \\ 
    &= \eta_y^2\sum_{i=1}^n w_i\sum_{k = 0}^{\tau_i-1}\frac{a_i^{(t,k)}}{\|a_i^{(t)}\|_1}\sum_{j=0}^{k-1}\Big(a_i^{(t,j)}\Big)^2\mathbb{E}\Big\|\nabla_y g_i\big(x_i^{(t,j)}, y_i^{(t,j)}; \zeta_i^{(t,j)}\big) - \nabla_y g_i\big(x_i^{(t,j)}, y_i^{(t,j)}\big)\Big\|^2 \nonumber \\
    & \quad + \eta_y^2\sum_{i=1}^n w_i\sum_{k = 0}^{\tau_i-1}\frac{a_i^{(t,k)}}{\|a_i^{(t)}\|_1}\mathbb{E}\Big\|\sum_{j=0}^{k-1}a_i^{(t,j)}\nabla_y g_i\big(x_i^{(t,j)}, y_i^{(t,j)}\big)\Big\|^2 \nonumber \\
    &\leq \eta_y^2 \sum_{i=1}^n w_i\sum_{k = 0}^{\tau_i-1}\frac{a_i^{(t,k)}}{\|a_i^{(t)}\|_1}\sum_{j=0}^{k-1}\Big(a_i^{(t,j)}\Big)^2 \sigma^2_g \nonumber \\
    & \quad + 2\eta_y^2\sum_{i=1}^n w_i \sum_{k = 0}^{\tau_i-1}\frac{a_i^{(t,k)}}{\|a_i^{(t)}\|_1}\sum_{j=0}^{k-1}a_i^{(t,j)}\mathbb{E}\Big\|\nabla_y g_i\big(x_i^{(t,j)}, y_i^{(t,j)}\big) - \nabla_y g_i\big(x^{(t)}, y^{(t)}\big)\Big\|^2\nonumber \\
    & \quad + 2\eta_y^2\sum_{i=1}^n w_i \sum_{k = 0}^{\tau_i-1}\frac{a_i^{(t,k)}}{\|a_i^{(t)}\|_1}\sum_{j=0}^{k-1}a_i^{(t,j)}\mathbb{E}\Big\|\nabla_y g_i\big(x^{(t)}, y^{(t)}\big)\Big\|^2 \nonumber \\
    &\overset{(a)}{\leq} \eta_y^2 \sum_{i=1}^n w_i\big\|a_i^{(t)}\big\|_2^2 \sigma^2_g + 2\eta_y^2\sum_{i=1}^n w_i \sum_{k = 0}^{\tau_i-1}\frac{a_i^{(t,k)}}{\|a_i^{(t)}\|_1}\mathbb{E}\Big\|\nabla_y g_i\big(x_i^{(t,j)}, y_i^{(t,j)}\big) - \nabla_y g_i\big(x^{(t)}, y^{(t)}\big)\Big\|^2\nonumber \\
    & \quad + 2\eta_y^2\sum_{i=1}^n w_i \big\|a_i^{(t)}\big\|_1\mathbb{E}\Big\|\nabla_y g_i\big(x^{(t)}, y^{(t)}\big)\Big\|^2 \nonumber \\
    &\overset{(b)}{\leq} \eta_y^2 \Bar{\tau}\alpha_{\max}^2 \sigma^2_g + 2\eta_y^2 c_a\Bar{\tau}\alpha_{\max}\sum_{i=1}^n w_i \sum_{k = 0}^{\tau_i-1}\frac{a_i^{(t,k)}}{\|a_i^{(t)}\|_1}\mathbb{E}\Big\|\nabla_y g_i\big(x_i^{(t,j)}, y_i^{(t,j)}\big) - \nabla_y g_i\big(x^{(t)}, y^{(t)}\big)\Big\|^2\nonumber \\
    & \quad + 2\eta_y^2c_{a}\bar{\tau}\alpha_{\max} \mathbb{E}\bigg\|\sum_{i=1}^n w_i\Big[\nabla_y g_i\big(x^{(t)}, y^{(t)}\big) - \nabla_y g_i\big(x^{(t)}, y^*{(x^{(t)})}\big)\Big]\bigg\|^2 + 2\eta_y^2c_{a}\bar{\tau}\alpha_{\max} \sigma_{gh}^2 \nonumber \\
    &\overset{(c)}{\leq} \eta_y^2 \Bar{\tau}\alpha_{\max}^2 \sigma^2_g + 2\eta_y^2 c_a\Bar{\tau}\alpha_{\max}\sum_{i=1}^n w_i \sum_{k = 0}^{\tau_i-1}\frac{a_i^{(t,k)}}{\|a_i^{(t)}\|_1}L_1^2\mathbb{E}\Big[\big\|x_i^{(t,j)} - x^{(t)}\big\|^2 + \big\|y_i^{(t,j)} - y^{(t)}\big\|^2\Big]\nonumber \\
    & \quad + 2\eta_y^2c_{a}\bar{\tau}\alpha_{\max} L_1^2\mathbb{E}\big\|y^{(t)} - y^*{(x^{(t)})}\big\|^2 + 2\eta_y^2c_{a}\bar{\tau}\alpha_{\max} \sigma_{gh}^2
    % &\leq \eta_y^2 \sum_{i=1}^n w_i\sum_{j=0}^{k-1}\Big(a_i^{(t,j)}\Big)^2 \sigma^2_g 
    % + 2\eta_y^2\sum_{i=1}^n w_i \sum_{j=0}^{k-1}a_i^{(t,j)}\mathbb{E}\bigg\|\nabla_y g_i\big(x_i^{(t,j)}, y_i^{(t,j)}\big) - \nabla_y g_i\big(x^{(t)}, y^{(t)}\big)\bigg\|^2\nonumber \\
    % & \quad + 2\eta_y^2 k\alpha_{\max}\beta_{gh}^2L_1^2\mathbb{E}\big\|y^{(t)} - y^*{(t)}\big\|^2 + 2\eta_y^2 k\alpha_{\max}\sigma_{gh}^2.
    % &\leq \eta_y^2 k\alpha_{\max}^2 \sigma^2_g + 2\eta_y^2 k \beta_{\max} \alpha_{\max} L_1^2\Big[\big\|x_i^{(t,j)} - x^{(t)}\big\|^2 + \big\|y_i^{(t,j)} - y^{(t)}\big\|^2\Big] \nonumber \\
    % & \quad + 2\eta_y^2 k\alpha_{\max} L_1^2\big\|y^{(t)} - y^*{(t)}\big\|^2.
    % & \overset{(a)}{\leq} \eta_y^2\sum_{i=1}^n w_i\sum_{j=0}^{k-1}\Big(a_i^{(t,j)}\Big)^2 \sigma^2_g + \eta_y^2\sum_{i=1}^n w_i\Big(\sum_{j=0}^{k-1}a_i^{(t,j)}\Big)\sum_{j=0}^{k-1}a_i^{(t,j)}\mathbb{E}\big\|\nabla_y g_i\big(x_i^{(t,j)}, y_i^{(t,j)}\big)\big\|^2 \nonumber \\
    % & \overset{(b)}{\leq} \eta_y^2k\beta_{\max}\alpha_{\max}^2 \sigma^2_g + \eta_y^2k^2\alpha_{\max}^2\Big(\beta_{gh}^2 \big\|y - y^*(x)\big\|^2 + \sigma_{gh}^2\Big), 
    % & \overset{(b)}{\leq} \eta_y^2\beta_{\max}\sum_{j=0}^{k-1}\Big(a_i^{(t,j)}\Big)^2 \sigma^2_g + \eta_y^2\Big(\sum_{j=0}^{k-1}a_i^{(t,j)}\Big)\sum_{j=0}^{k-1}a_i^{(t,j)}L^2,  
\end{align}
% where (a) uses Jensen's inequality; (b) follows from Assumption \ref{as:Lipschitz}. 
% Then we can easily have
% \begin{align}
%      \sum_{i=1}^n w_i\sum_{k = 0}^{\tau_i-1}&\frac{a_i^{(t,k)}}{\|a_i^{(t)}\|_1}\big\|y^{(t,k)}_i - y^{(t)}\big\|^2 \nonumber \\
%      \leq & 
%      % \leq &\eta_y^2 \tau_i^2\frac{\alpha_{\max}^3}{2\alpha_{\min}} \sigma^2_g + \eta_y^2\tau_i^2\frac{\beta_{\max} \alpha_{\max}^2}{\alpha_{\min}} L_1^2\Big[\big\|x_i^{(t,j)} - x^{(t)}\big\|^2 + \big\|y_i^{(t,j)} - y^{(t)}\big\|^2\Big] \nonumber \\
%      % & + \eta_y^2 \tau_i^2\frac{\alpha_{\max}^2}{\alpha_{\min}} L_1^2\big\|y^{(t)} - y^*{(t)}\big\|^2. \nonumber
%      % \eta_y^2 \Big(\|a_i\|^2_2\sigma^2_g + \|a_i^{(t)}\|^2_1L^2\Big)\nonumber 
% \end{align}
% \begin{align}
%      \sum_{i=1}^n w_i \sum_{k = 0}^{\tau_i-1}&\frac{a_i^{(t,k)}}{\|a_i^{(t)}\|_1}\big\|y^{(t,k)}_i - y^{(t)}\big\|^2 \nonumber \\
%      \leq &\eta_y^2 \tau_i^2\frac{\alpha_{\max}^3}{2\alpha_{\min}} \sigma^2_g + \eta_y^2\tau_i^2\frac{\beta_{\max} \alpha_{\max}^2}{\alpha_{\min}} L_1^2\Big[\big\|x_i^{(t,j)} - x^{(t)}\big\|^2 + \big\|y_i^{(t,j)} - y^{(t)}\big\|^2\Big] \nonumber \\
%      & + \eta_y^2 \tau_i^2\frac{\alpha_{\max}^2}{\alpha_{\min}} L_1^2\big\|y^{(t)} - y^*{(t)}\big\|^2. \nonumber
%      % \eta_y^2 \Big(\|a_i\|^2_2\sigma^2_g + \|a_i^{(t)}\|^2_1L^2\Big)\nonumber 
% \end{align}
where (a) holds because of Assumption\ref{as:varaince} and 
\begin{align}
     \frac{1}{\|a_i^{(t)}\|_1}\sum_{k = 0}^{\tau_i-1}a_i^{(t,k)}\sum_{j = 0}^{k-1}\Big(a_i^{(t,j)}\Big)^2 
     &\leq  \frac{1}{\|a_i^{(t)}\|_1}\sum_{k = 0}^{\tau_i-1}a_i^{(t,k)}\sum_{j = 0}^{\tau_i-2}\Big(a_i^{(t,j)}\Big)^2 = \sum_{j = 0}^{\tau_i-2}\Big(a_i^{(t,j)}\Big)^2
     \leq \|a_i\|_2^2 \nonumber \\
     \frac{1}{\|a_i^{(t)}\|_1}\sum_{k = 0}^{\tau_i-1}a_i^{(t,k)}\sum_{j = 0}^{k-1}a_i^{(t,j)}
     &\leq  \frac{1}{\|a_i^{(t)}\|_1}\sum_{k = 0}^{\tau_i-1}a_i^{(t,k)}\sum_{j = 0}^{\tau_i-2}a_i^{(t,j)} = \sum_{j = 0}^{\tau_i-2}a_i^{(t,j)}
     \leq \|a_i^{(t)}\|_1; \nonumber
\end{align}
(b) is obtained from Assumption\ref{as:globalheter} and (c) follows from Assumption\ref{as:Lipschitz}.  
% Similarly, by using Assumption \ref{as:Lipschitz} and replacing $\nabla g_i$ with $\nabla_v R$ and $\bar{\nabla}f$ in \cref{eq:boundofCD1}, it is easy to have 
As an easier case, we can easily know that $\bar{\nabla}f_i$ and $\nabla_v R_i$ are bounded from Assumption \ref{as:Lipschitz} and \Cref{lm:boundofvi}. Then we have 
\begin{align}
    \sum_{i=1}^n w_i&\frac{1}{\|a_i^{(t)}\|_1}\sum_{k=1}^{\tau_i-1}a^{(t,k)}_i \mathbb{E}\|x^{(t,k)}_i - x^{(t)}\|^2 \leq \eta^2_x \bar{\tau}\Big(\alpha_{\max}^2(\sigma^2_{f}+r_{\max}^2\sigma^2_{gg}) + \alpha_{\max}(L_f^2 + r^2_{\max}L^2_1) \Big), \nonumber \\
    \sum_{i=1}^n w_i&\frac{1}{\|a_i^{(t)}\|_1}\sum_{k=1}^{\tau_i-1}a^{(t,k)}_i \mathbb{E}\|v^{(t,k)}_i - v^{(t)}\|^2 \leq \eta^2_v \bar{\tau}\Big(\alpha_{\max}^2(\sigma^2_{f}+r_{\max}^2\sigma^2_{gg}) + \alpha_{\max}(L_f^2 + r^2_{\max}L^2_1) \Big), \nonumber\\
    \sum_{i=1}^n w_i&\frac{1}{\|a_i^{(t)}\|_1}\sum_{k=1}^{\tau_i-1}a^{(t,k)}_i \mathbb{E}\|y^{(t,k)}_i - y^{(t)}\|^2 \nonumber \\
    &\leq \frac{\eta_y^2\Bar{\tau}}{1-2\eta_y^2 c_a\Bar{\tau}\alpha_{\max}L_1^2}\bigg[\alpha_{\max}^2 \sigma^2_g + 2 c_a\alpha_{\max}L_1^2\eta^2_x \bar{\tau}\Big(\alpha_{\max}^2(\sigma^2_{f}+r_{\max}^2\sigma^2_{gg}) + \alpha_{\max}(L_f^2 + r^2_{\max}L^2_1) \Big) \nonumber\\
    &\qquad\qquad\qquad\qquad\qquad + 2c_{a}\alpha_{\max} L_1^2\mathbb{E}\big\|y^{(t)} - y^*{(x^{(t)})}\big\|^2 + 2c_{a}\alpha_{\max} \sigma_{gh}^2\bigg] \nonumber
\end{align}
% \begin{align}
%     \sum_{i=1}^n w_i\frac{1}{\|a_i^{(t)}\|_1}\sum_{k=1}^{\tau_i-1}a^{(t,k)}_i \mathbb{E}\|v^{(t,k)}_i - v^{(t)}\|^2 &\leq \eta^2_v \Big(\|a_i^{(t)}\|^2_2(\sigma^2_{f}+r_i\sigma^2_{gg}) + \|a_i^{(t)}\|^2_1(L^2 + r^2_iL^2_1) \Big), \nonumber\\
%     \sum_{i=1}^n w_i\frac{1}{\|a_i^{(t)}\|_1}\sum_{k=1}^{\tau_i-1}a^{(t,k)}_i \mathbb{E}\|x^{(t,k)}_i - x^{(t)}\|^2 &\leq \eta^2_x \Big(\|a_i^{(t)}\|^2_2(\sigma^2_{f}+r_i\sigma^2_{gg}) + \|a_i^{(t)}\|^2_1(L^2 + r^2_iL^2_1) \Big), \nonumber
% \end{align}
which finished the proof. 
\end{proof}
\noindent
In the later part, we will take $\{2\eta_x^2 \Bar{\tau}c_a\alpha_{\max}L_1^2 \leq 1\}$ and $\{4\eta_y^2 c_a\Bar{\tau}\alpha_{\max}L_1^2 \leq 1\}$, then we can simplify \Cref{lm:boundofCD} as 
\begin{align}\label{eq:boundofCDsimple}
    \sum_{i=1}^n w_i&\frac{1}{\|a_i^{(t)}\|_1}\sum_{k=1}^{\tau_i-1}a^{(t,k)}_i \mathbb{E}\|x^{(t,k)}_i - x^{(t)}\|^2 \leq \eta^2_x \bar{\tau}\sigma_{M1}^2, \nonumber \\
    \sum_{i=1}^n w_i&\frac{1}{\|a_i^{(t)}\|_1}\sum_{k=1}^{\tau_i-1}a^{(t,k)}_i \mathbb{E}\|v^{(t,k)}_i - v^{(t)}\|^2 \leq \eta^2_v \bar{\tau}\sigma_{M1}^2, \nonumber\\
    \sum_{i=1}^n w_i&\frac{1}{\|a_i^{(t)}\|_1}\sum_{k=1}^{\tau_i-1}a^{(t,k)}_i \mathbb{E}\|y^{(t,k)}_i - y^{(t)}\|^2 \leq 2\eta_y^2\Bar{\tau}\sigma_{M2}^2 + 4\eta_y^2\Bar{\tau}c_{a}\alpha_{\max} L_1^2\mathbb{E}\big\|y^{(t)} - y^*{(x^{(t)})}\big\|^2
\end{align}
where we define $\sigma_{M2}^2 := \alpha_{\max}^2 \sigma^2_g + \alpha_{\max}^2(\sigma^2_{f}+r_{\max}^2\sigma^2_{gg}) + \alpha_{\max}(L_f^2 + r^2_{\max}L^2_1)  + 2c_{a}\alpha_{\max} \sigma_{gh}^2$.
\noindent
As the combination of \Cref{lm:CDcombo} and \Cref{lm:boundofCD}, we have 
\begin{align}\label{eq:boundofCDcombo}
    \sum_{i=1}^n w_i\sum_{k=1}^{\tau_i-1}\frac{a^{(t,k)}_i}{\|a_i^{(t)}\|_1}\Delta_{f_i}^{(t,k)} &=  \sum_{i=1}^n w_i\sum_{k=1}^{\tau_i-1}\frac{a^{(t,k)}_i}{\|a_i^{(t)}\|_1}\Delta_{R_i}^{(t,k)} \nonumber \\
    &\leq 3\eta^2_x\bar{\tau}\big(L^2_1 + r^2L^2_2\big) \sigma_{M1}^2 
    + 3\eta^2_v\bar{\tau}L^2_1 \sigma_{M1}^2 + 6\eta_y^2\Bar{\tau}\big(L^2_1 + r^2L^2_2\big)\sigma_{M2}^2\nonumber \\
    &\quad 
    + 12\eta_y^2\Bar{\tau}\big(L^2_1 + r^2L^2_2\big)c_{a}\alpha_{\max} L_1^2\mathbb{E}\big\|y^{(t)} - y^*{(x^{(t)})}\big\|^2 \nonumber \\
    \sum_{i=1}^n w_i\sum_{k=1}^{\tau_i-1}\frac{a^{(t,k)}_i}{\|a_i^{(t)}\|_1}\Delta_{g_i}^{(t,k)} 
    &\leq \eta_x^2\Bar{\tau}L_1^2\sigma_{M1}^2 + 
    2\eta_y^2\Bar{\tau}L_1^2\sigma_{M2}^2 + 4\eta_y^2\Bar{\tau}c_{a}\alpha_{\max} L_1^4\mathbb{E}\big\|y^{(t)} - y^*{(x^{(t)})}\big\|^2.
\end{align}

\subsection{Bounds of Aggregated Estimations}
\begin{lemma}\label{lm:boundofAE}
Suppose the server selects $|C^{(t)}| = P$ clients in each round. Under Assumption \ref{as:diffandSC}, \ref{as:Lipschitz} and \ref{as:varaince}, the aggregated estimation of $x^{(t)}$ satisfies
\begin{align}
    \mathbb{E}\Big\|\sum_{i \in C^{(t)}} \widetilde{w}_i h^{(t)}_{x,i}\Big\|^2 \leq& \frac{2n}{P}\sum_{i=1}^n\frac{w_i^2}{\|a_i^{(t)}\|^2_1} \sum_{k=0}^{\tau_i-1}\big(a^{(t,k)}_i\big)^2(\sigma^2_{f} + r_i^2\sigma^2_{gg}) + \frac{n(P-1)}{P(n-1)}\mathbb{E}\Big\|\sum_{i=1}^n w_i \widetilde{h}^{(t)}_{x,i}\Big\|^2 \nonumber\\
    &+ \frac{2n(n-P)}{P(n-1)}\sum_{i=1}^n\frac{w_i^2}{\|a_i^{(t)}\|_1}\sum_{k=0}^{\tau_i-1} a^{(t,k)}_i \Delta^{(t,k)}_{f,i} + \frac{4(n-P)\beta_{\max}}{P(n-1)}(r_{\max}^2L_1^2 + L_f^2). \nonumber
\end{align}
% \begin{align}
%     \mathbb{E}\Big\|\sum_{i \in C^{(t)}} \widetilde{w}_i h^{(t)}_{x,i}\Big\|^2 &\leq \frac{2n}{P}\sum_{i=1}^n\frac{w_i^2}{\|a_i^{(t)}\|^2_1} \sum_{k=0}^{\tau_i-1}\big(a^{(t,k)}_i\big)^2(\sigma^2_{f} + r_i^2\sigma^2_{gg}) + \frac{n}{P}\frac{P-1}{n-1}\mathbb{E}\Big\|\sum_{i=1}^n w_i \widetilde{h}^{(t)}_{x,i}\Big\|^2 \nonumber\\
%     &+ \frac{n}{P}\frac{2(n-P)}{n-1}\sum_{i=1}^n\frac{w_i^2}{\|a_i^{(t)}\|_1}\sum_{k=0}^{\tau_i-1} a^{(t,k)}_i \Delta^{(t,k)}_{f,i} + \frac{2n}{P}\frac{2(n-P)}{n-1}\beta_{\max}(1+r_i^2)L^2, \nonumber
% \end{align}
the aggregated estimation of $y^{(t)}$ satisfies
\begin{align}
    \mathbb{E}\Big\|\sum_{i \in C^{(t)}} \widetilde{w}_i h^{(t)}_{y,i}\Big\|^2
    \leq& \frac{n}{P}\sum_{i=1}^n\frac{w_i^2}{\|a_i^{(t)}\|^2_1} \sum_{k=0}^{\tau_i-1}\big(a^{(t,k)}_i\big)^2 \sigma^2_{g} + \frac{2(n-P)\beta_{\max}}{P(n-1)}\sigma_{gh}^2 \nonumber\\
    &+ \bigg(\frac{2n(n-P)}{P(n-1)}\sum_{i=1}^n w_i^2 \sum_{k=0}^{\tau_i-1} \frac{a_i^{(t,k)}}{\|a_i^{(t)}\|_1} + 3\sum_{i=1}^n w_i \sum_{k=0}^{\tau_i-1} \frac{a_i^{(t,k)}}{\|a_i^{(t)}\|_1}
    \bigg)\Delta_{g,i}^{(t,k)} \nonumber \\
    &+ \bigg(\frac{2(n-P)\beta_{\max}\beta_{gh}^2}{P(n-1)} + 3L_1^2\bigg) \mathbb{E}\|y^{(t)} - y^*(x^{(t)})\|^2.  \nonumber
\end{align}
% \begin{align}
%     \mathbb{E}\Big\|\sum_{i \in C^{(t)}} \widetilde{w}_i h^{(t)}_{y,i}\Big\|^2 \leq& \frac{n}{P}\sum_{i=1}^n\frac{w_i^2}{\|a_i^{(t)}\|^2_1} \sum_{k=0}^{\tau_i-1}\big(a^{(t,k)}_i\big)^2 \sigma^2_{g} + \frac{2(n-P)\beta_{\max}}{P(n-1)}L^2\nonumber\\
%     &+ \bigg(\frac{2n(n-P)}{P(n-1)}\sum_{i=1}^n w_i^2 \sum_{k=0}^{\tau_i-1} \frac{a_i^{(t,k)}}{\|a_i^{(t)}\|_1} + 3\sum_{i=1}^n w_i \sum_{k=0}^{\tau_i-1} \frac{a_i^{(t,k)}}{\|a_i^{(t)}\|_1}
%     \bigg)\Delta_{g,i}^{(t,k)} \nonumber\\
%     &+ 3L_1^2 \mathbb{E}\|y^{(t)} - y^*(t)\|^2 ,
% \end{align}
% \begin{align}
%     \mathbb{E}\Big\|\sum_{i \in C^{(t)}} \widetilde{w}_i h^{(t)}_{y,i}\Big\|^2 \leq& \frac{2n}{P}\sum_{i=1}^n\frac{w_i^2}{\|a_i^{(t)}\|^2_1} \sum_{k=0}^{\tau_i-1}\big(a^{(t,k)}_i\big)^2\sigma^2_{g_y}\nonumber\\
%     &+ \bigg(\frac{n}{P}\frac{2(n-P)}{n-1}\sum_{i=1}^n\frac{w_i^2}{\|a_i^{(t)}\|_1}\sum_{k=0}^{\tau_i-1} a^{(t,k)}_i + 3 \sum_{i=1}^nw_i\sum_{k=0}^{\tau_i-1} \frac{a^{(t,k)}_i}{\|a_i^{(t)}\|_1}\bigg)\Delta^{(t,k)}_{g,i} \nonumber\\
%     &+ \frac{n}{P}\frac{2(n-P)}{n-1}\big(\max_i w_i\big)\sigma^2_{G_g} + L^2_{g_y}\mathbb\|y^{(t)} - y^*(t)\|^2, \nonumber 
% \end{align}
and the aggregated estimation of $v^{(t)}$ satisfies
\begin{align}
    \mathbb{E}\Big\|\sum_{i \in C^{(t)}} \widetilde{w}_i h^{(t)}_{v,i}\Big\|^2 \leq& \frac{2n}{P}\sum_{i=1}^n\frac{w_i^2}{\|a_i^{(t)}\|^2_1} \sum_{k=0}^{\tau_i-1}\big(a^{(t,k)}_i\big)^2 \big(\sigma^2_{f} + r_i^2\sigma^2_{gg}\big) + \frac{4(n-P)\beta_{\max}}{P(n-1)}(r_{\max}^2L_1^2 + L_f^2) \nonumber\\
    &+ \bigg(\frac{2n(n-P)}{P(n-1)}\sum_{i=1}^n w_i^2 \sum_{k=0}^{\tau_i-1} \frac{a_i^{(t,k)}}{\|a_i^{(t)}\|_1} + 3\sum_{i=1}^n w_i \sum_{k=0}^{\tau_i-1} \frac{a_i^{(t,k)}}{\|a_i^{(t)}\|_1}
    \bigg)\Delta_{R,i}^{(t,k)} \nonumber\\
    &+ 3L_1^2 \mathbb{E}\|v^{(t)} - v^*(x^{(t)})\|^2. \nonumber
\end{align}
% \begin{align}
%     \mathbb{E}\Big\|\sum_{i \in C^{(t)}} \widetilde{w}_i h^{(t)}_{v,i}\Big\|^2 \leq& \frac{2n}{P}\sum_{i=1}^n\frac{w_i^2}{\|a_i^{(t)}\|^2_1} \sum_{k=0}^{\tau_i-1}\big(a^{(t,k)}_i\big)^2 (\sigma^2_{f} + r_i^2\sigma^2_{gg}) + \frac{4(n-P)\beta_{\max}}{P(n-1)}(1+r_i^2)L^2 \nonumber\\
%     &+ \bigg(\frac{2n(n-P)}{P(n-1)}\sum_{i=1}^n w_i^2 \sum_{k=0}^{\tau_i-1} \frac{a_i^{(t,k)}}{\|a_i^{(t)}\|_1} + 3\sum_{i=1}^n w_i \sum_{k=0}^{\tau_i-1} \frac{a_i^{(t,k)}}{\|a_i^{(t)}\|_1}
%     \bigg)\Delta_{R,i}^{(t,k)} \nonumber\\
%     &+ 3L_1^2 \mathbb{E}\|v^{(t)} - v^*(t)\|^2 
% \end{align}
% \begin{align}
%     \mathbb{E}\Big\|\sum_{i \in C^{(t)}} \widetilde{w}_i h^{(t)}_{v,i}\Big\|^2 \leq& \frac{2n}{P}\sum_{i=1}^n\frac{w_i^2}{\|a_i^{(t)}\|^2_1} \sum_{k=0}^{\tau_i-1}\big(a^{(t,k)}_i\big)^2(\sigma^2_{f_y} + r_i^2\sigma^2_{g_{yy}}) \nonumber\\
%     &+ \bigg(\frac{n}{P}\frac{2(n-P)}{n-1}\sum_{i=1}^n\frac{w_i^2}{\|a_i^{(t)}\|_1}\sum_{k=0}^{\tau_i-1} a^{(t,k)}_i + 3 \sum_{i=1}^nw_i\sum_{k=0}^{\tau_i-1} \frac{a^{(t,k)}_i}{\|a_i^{(t)}\|_1}\bigg)\Delta^{(t,k)}_{R,i} \nonumber\\
%     &+ \frac{n}{P}\frac{2(n-P)}{n-1}\big(\max_i w_i\big)\sigma^2_{G_R} + L^2_{g_y}\mathbb\|v^{(t)} - v^*(t)\|^2, \nonumber 
% \end{align}
for all $t \in \{0,1,...,T-1\}$, $k \in \{0,1,...,\tau_i-1\}$ and $i \in \{1,2,...,n\}$. 
% Here, we define the combination of client drifts as 
% \begin{align}\label{def:combofCD}
%     \Delta^{(t,k)}_{f,i} &:=  3(L^2_{f_x} + r^2_iL^2_{g_{xy}})\big[\|x_i^{(t,k)} - x^{(t)}\|^2  + \|y_i^{(t,k)} - y^{(t)}\|^2 \big] + 3L^2_{g_y}\|v_i^{(t,k)} - v^{(t)}\|^2, \nonumber \\
%     \Delta^{(t,k)}_{g,i} &:= L^2_g \big[\|x_i^{(t,k)} - x^{(t)}\|^2 + \|y_i^{(t,k)} - y^{(t)}\|^2\big],  \nonumber \\
%     \Delta^{(t,k)}_{R,i} &:= (L^2_{f_y} + r^2_iL^2_{g_{yy}})\big[\|x_i^{(t,k)} - x^{(t)}\|^2  + \|y_i^{(t,k)} - y^{(t)}\|^2 \big] + 3L^2_{g_y}\|v_i^{(t,k)} - v^{(t)}\|^2 
% % \end{align}
% for all $t \in \{0,1,...,T-1\}$, $k \in \{0,1,...,\tau_i-1\}$ and $i \in \{1,2,...,n\}$.
\end{lemma}
\begin{proof}
For the aggregated estimation of $x^{(t)}$, we have
\begingroup
\allowdisplaybreaks
\begin{align}\label{eq:boundofAEx1}
    \mathbb{E}\bigg\|\sum_{i \in C^{(t)}}\widetilde{w}_ih_{x,i}^{(t)}\bigg\|^2 &= \mathbb{E}\bigg\|\sum_{i \in C^{(t)}}\widetilde{w}_i\Big(h_{x,i}^{(t)} - \widetilde{h}_{x,i}^{(t)} + \widetilde{h}_{x,i}^{(t)}\Big)\bigg\|^2 \nonumber \\
    &= \mathbb{E}\bigg\|\sum_{i \in C^{(t)}}\widetilde{w}_i\Big(h_{x,i}^{(t)} - \widetilde{h}_{x,i}^{(t)}\Big)\bigg\|^2 + \mathbb{E}\bigg\|\sum_{i \in C^{(t)}}\widetilde{w}_i\widetilde{h}_{x,i}^{(t)}\bigg\|^2 \nonumber \\
    & \overset{(a)}{=} \mathbb{E}\bigg[\sum_{i \in C^{(t)}}\widetilde{w}^2_i\Big\|h_{x,i}^{(t)} - \widetilde{h}_{x,i}^{(t)}\Big\|^2\bigg] + \mathbb{E}\bigg\|\sum_{i \in C^{(t)}}\widetilde{w}_i\widetilde{h}_{x,i}^{(t)}\bigg\|^2 \nonumber \\
    & \overset{(b)}{=} \frac{n}{P}\sum_{i = 1}^n w^2_i\mathbb{E}\Big\|h_{x,i}^{(t)} - \widetilde{h}_{x,i}^{(t)}\Big\|^2 + \mathbb{E}\bigg\|\sum_{i \in C^{(t)}}\widetilde{w}_i\widetilde{h}_{x,i}^{(t)}\bigg\|^2 \nonumber \\
    % & \overset{(b)}{=} \frac{n}{P}\sum_{i = 1}^n w^2_i\mathbb{E}\Big\|h_{x,i}^{(t)} - \widetilde{h}_{x,i}^{(t)}\Big\|^2 + \mathbb{E}\bigg\|\sum_{i \in C^{(t)}}\widetilde{w}_i\widetilde{h}_{x,i}^{(t)}\bigg\|^2 \nonumber \\
    &\overset{(c)}{\leq} \frac{2n}{P}\sum_{i=1}^n \frac{w_i^2}{\|a_i^{(t)}\|_1^2}\sum_{k=0}^{\tau_i-1}\Big(a_i^{(t,k)}\Big)^2(\sigma^2_f + r_i^2\sigma^2_{gg}) + \mathbb{E}\bigg\|\sum_{i \in C^{(t)}}\widetilde{w}_i\widetilde{h}_{x,i}^{(t)}\bigg\|^2,
\end{align}
\endgroup
where (a) holds because clients are selected without replacement; (b) follows from the definition $\widetilde{w}_i = \frac{n}{P}w_i$; (c) uses Assumption \ref{as:varaince}. 
For the second term in \cref{eq:boundofAEx1}, we have
\begingroup
\allowdisplaybreaks
\begin{align}\label{eq:boundofAEx2}
    \mathbb{E}\bigg\|&\sum_{i \in C^{(t)}}\widetilde{w}_i\widetilde{h}_{x,i}^{(t)}\bigg\|^2 \nonumber \\
    &= \mathbb{E}\bigg\|\sum_{i \in C^{(t)}}\widetilde{w}_i\widetilde{h}_{x,i}^{(t)} - \sum_{i  = i}^n w_i\widetilde{h}_{x,i}^{(t)} + \sum_{i  = i}^n w_i\widetilde{h}_{x,i}^{(t)}\bigg\|^2 \nonumber \\
    & \overset{(a)}{=} \mathbb{E}\bigg\|\sum_{i =1}^n \mathbb{I}(i \in C^{(t)})\widetilde{w}_i\widetilde{h}_{x,i}^{(t)} - \sum_{i  = i}^n w_i\widetilde{h}_{x,i}^{(t)}\bigg\|^2 + \mathbb{E}\bigg\|\sum_{i  = i}^n w_i\widetilde{h}_{x,i}^{(t)}\bigg\|^2 \nonumber \\
    & = \sum_{i=1}^n \mathbb{E}\bigg[\Big(\mathbb{I}(i \in C^{(t)})^2\widetilde{w}_i^2 + w_i^2 - 2 \mathbb{I}(i \in C^{(t)})\widetilde{w}_iw_i\Big)\Big\|h_{x,i}^{(t)}\Big\|^2\bigg] \nonumber \\
    & \quad + \sum_{i \neq j} \mathbb{E}\bigg\langle\big(\mathbb{I}(i \in C^{(t)})\widetilde{w}_i - w_i\big)h_{x,i}^{(t)}, \big(\mathbb{I}(j \in C^{(t)})\widetilde{w}_j - w_j\big)h_{x,j}^{(t)}\bigg\rangle 
    + \mathbb{E}\bigg\|\sum_{i = 1}^n w_i\widetilde{h}_{x,i}^{(t)}\bigg\|^2 \nonumber \\
    & = \mathbb{E}\bigg\|\sum_{i = 1}^n w_i\widetilde{h}_{x,i}^{(t)}\bigg\|^2 + \sum_{i=1}^n \mathbb{E}\bigg[w_i^2\Big(\frac{n}{P}-1\Big)\Big\|h_{x,i}^{(t)}\Big\|^2\bigg] \nonumber \\ 
    &\quad + \sum_{i \neq j} \mathbb{E}\bigg[\Big(\mathbb{I}(i,j \in C^{(t)})\widetilde{w}_i\widetilde{w}_j - \mathbb{I}(j \in C^{(t)})\widetilde{w}_jw_i - \mathbb{I}(i \in C^{(t)})\widetilde{w}_iw_j + w_iw_j\Big)\Big\langle h_{x,i}^{(t)}, h_{x,j}^{(t)}\Big\rangle\bigg] \nonumber \\
    & = \mathbb{E}\bigg\|\sum_{i = 1}^n w_i\widetilde{h}_{x,i}^{(t)}\bigg\|^2 
    + \Big(\frac{n}{P}-1\Big)\sum_{i = 1}^n \mathbb{E}\bigg[w_i^2\Big\|\widetilde{h}_{x,i}^{(t)}\Big\|^2\bigg] \nonumber \\
    & \quad+ \sum_{i \neq j}\mathbb{E}\bigg[w_iw_j \bigg(\frac{n}{P}\Big(\frac{P-1}{n-1}\Big)-1\bigg)\Big\langle h_{x,i}^{(t)}, h_{x,j}^{(t)}\Big\rangle\bigg] \nonumber \\ 
    &= \frac{n}{P}\bigg(\frac{P-1}{n-1}\bigg)\mathbb{E}\bigg\|\sum_{i=1}^n w_i \widetilde{h}_{x,i}^{(t)}\bigg\|^2  + \frac{n}{P}\bigg(\frac{n-P}{n-1}\bigg)\sum_{i=1}^n w_i^2 \mathbb{E} \Big\|\widetilde{h}_{x,i}^{(t)}\Big\|^2,
\end{align}
\endgroup
where (a) holds because clients are selected without replacement. And for the second term in \cref{eq:boundofAEx2}, we have 
\begingroup
\allowdisplaybreaks
\begin{align}\label{eq:boundofAEx3}
    &\sum_{i=1}^n w_i^2 \mathbb{E} \Big\|\widetilde{h}_{x,i}^{(t)}\Big\|^2 \nonumber \\
    &\ = \sum_{i=1}^n w_i^2 \mathbb{E} \Big\|\widetilde{h}_{x,i}^{(t)} - \bar{\nabla}f(x^{(t)}, y^{(t)}, v^{(t)}) + \bar{\nabla}f(x^{(t)}, y^{(t)}, v^{(t)})\Big\|^2 \nonumber \\
    & \overset{(a)}{\leq} 2\sum_{i=1}^n w_i^2 \mathbb{E} \Big\|\widetilde{h}_{x,i}^{(t)} - \bar{\nabla}f(x^{(t)}, y^{(t)}, v^{(t)}) \Big\|^2 + 2\frac{\beta_{\max}}{n}\sum_{i=1}^n w_i \mathbb{E} \Big\| \bar{\nabla}f(x^{(t)}, y^{(t)}, v^{(t)})\Big\|^2 \nonumber \\
    & \overset{(b)}{\leq} 2\sum_{i=1}^n w_i^2 \sum_{k=0}^{\tau_i-1}\frac{a_i^{(t,k)}}{\|a_i^{(t)}\|_1} \mathbb{E} \Big\|\bar{\nabla}f(x_i^{(t,k)}, y_i^{(t,k)}, v_i^{(t,k)}) - \bar{\nabla}f(x^{(t)}, y^{(t)}, v^{(t)}) \Big\|^2 + \frac{4\beta_{\max}}{n} (r_{\max}^2L_1^2 + L_f^2) \nonumber \\
    & \overset{(c)}{\leq} 2\sum_{i=1}^n w_i^2 \sum_{k=0}^{\tau_i-1}\frac{a_i^{(t,k)}}{\|a_i^{(t)}\|_1} \mathbb{E} \Delta_{f,i}^{(t,k)} + \frac{4\beta_{\max}}{n} (r_{\max}^2L_1^2 + L_f^2)
\end{align}
where (a) $w_i \leq \beta_{\max}/n$ for all $i \in \{1,..,n\}$; the first term of (b) uses Jensen's inequality and the second part of (b) follows from \Cref{lm:globalheterogeneity}; (c) uses the\Cref{lm:CDcombo}. By incorporating \cref{eq:boundofAEx2} and \cref{eq:boundofAEx3} into \cref{eq:boundofAEx1}, we get 
\begin{align}
    \mathbb{E}\Big\|\sum_{i \in C^{(t)}} \widetilde{w}_i h^{(t)}_{x,i}\Big\|^2 \leq& \frac{2n}{P}\sum_{i=1}^n\frac{w_i^2}{\|a_i^{(t)}\|^2_1} \sum_{k=0}^{\tau_i-1}\big(a^{(t,k)}_i\big)^2(\sigma^2_{f} + r_i^2\sigma^2_{gg}) + \frac{n(P-1)}{P(n-1)}\mathbb{E}\Big\|\sum_{i=1}^n w_i \widetilde{h}^{(t)}_{x,i}\Big\|^2 \nonumber\\
    &+ \frac{2n(n-P)}{P(n-1)}\sum_{i=1}^n\frac{w_i^2}{\|a_i^{(t)}\|_1}\sum_{k=0}^{\tau_i-1} a^{(t,k)}_i \Delta^{(t,k)}_{f,i} + \frac{4(n-P)\beta_{\max}}{P(n-1)}(r_{\max}^2L_1^2 + L_f^2). \nonumber
\end{align}
\endgroup
Similarly, by replacing $h_{x,i}^{(t)}$ and $\bar{\nabla}f$ with $h_{y,i}^{(t)}$ and $\nabla g$, we can easily get 
\begin{align}\label{eq:boundofAEy1}
    \mathbb{E}\Big\|\sum_{i \in C^{(t)}} \widetilde{w}_i h^{(t)}_{y,i}\Big\|^2 \leq& \frac{n}{P}\sum_{i=1}^n\frac{w_i^2}{\|a_i^{(t)}\|^2_1} \sum_{k=0}^{\tau_i-1}\big(a^{(t,k)}_i\big)^2 \sigma^2_{g} + \frac{n(P-1)}{P(n-1)}\mathbb{E}\Big\|\sum_{i=1}^n w_i \widetilde{h}^{(t)}_{y,i}\Big\|^2 \nonumber\\
    &+ \frac{2n(n-P)}{P(n-1)}\sum_{i=1}^n\frac{w_i^2}{\|a_i^{(t)}\|_1}\sum_{k=0}^{\tau_i-1} a^{(t,k)}_i \Delta^{(t,k)}_{g,i} + \frac{2(n-P)\beta_{\max}}{P(n-1)}\sigma_{gh}^2 \nonumber \\
    &+ \frac{2(n-P)\beta_{\max}\beta_{gh}^2}{P(n-1)}\mathbb{E}\|y^{(t)}-y^*(x^{(t)})\|^2. 
\end{align}
For the second terms in \cref{eq:boundofAEy1}, we have
\begin{align}\label{eq:boundofAEy2}
    \mathbb{E}\bigg\|\sum_{i=1}^n w_i \widetilde{h}^{(t)}_{y,i}\bigg\|^2 
    &\leq 3\mathbb{E}\bigg\|\sum_{i=1}^n w_i \nabla_y g_i(x^{(t)}, y^*(x^{(t)}))\bigg\|^2 \nonumber \\
    & \quad + 3\mathbb{E}\bigg\|\sum_{i=1}^n w_i \Big(\nabla_y g_i\big(x^{(t)}, y^{(t)}\big) - \nabla_y g_i\big(x^{(t)}, y^*(x^{(t)})\big)\Big)\bigg\|^2 \nonumber \\
    & \quad + 3\mathbb{E}\bigg\|\sum_{i=1}^n w_i \Big(\widetilde{h}_{y,i}^{(t)} - \nabla_y g_i\big(x^{(t)}, y^{(t)}\big)\Big)\bigg\|^2 \nonumber \\
    &\overset{(a)}{\leq} 3\mathbb{E}\bigg\|\sum_{i=1}^n w_i \Big(\nabla_y g_i\big(x^{(t)}, y^{(t)}\big) - \nabla_y g_i\big(x^{(t)}, y^*(x^{(t)})\big)\Big)\bigg\|^2 \nonumber \\
    & \quad + 3\mathbb{E}\bigg\|\sum_{i=1}^n w_i \Big(\widetilde{h}_{y,i}^{(t)} - \nabla_y g_i\big(x^{(t)}, y^{(t)}\big)\Big)\bigg\|^2 \nonumber \\
    & \overset{(b)}{\leq} 3L_1^2 \mathbb{E}\|y^{(t)} - y^*(x^{(t)})\|^2 \nonumber \\
    & \quad + 3\sum_{i=1}^n w_i \sum_{k=0}^{\tau_i-1}\frac{a_i^{(t,k)}}{\|a_i^{(t)}\|_1}\mathbb{E} \Big\| \nabla_y g_i \big(x_i^{(t,k)}, y_i^{(t,k)}\big) - \nabla_y g_i \big(x^{(t)}, y^{(t)}\big)\Big\|^2
    \nonumber \\
    & \overset{(b)}{\leq} 3L_1^2 \mathbb{E}\|y^{(t)} - y^*(x^{(t)})\|^2 + 3\sum_{i=1}^n w_i \sum_{k=0}^{\tau_i-1}\frac{a_i^{(t,k)}}{\|a_i^{(t)}\|_1} \Delta_{g,i}^{(t,k)}. 
\end{align}
where (a) follows from $\sum_{i=1}^n w_i \nabla_y g_i(x^{(t)}, y^*(x^{(t)})) = 0$; the first term of (b) uses Assumption \ref{as:Lipschitz} and the second term of (b) uses Jensen inequality; (c) follows from \Cref{lm:CDcombo}. 
By incorporate \ref{eq:boundofAEy2} into \ref{eq:boundofAEy1}, we get
\begin{align}
    \mathbb{E}\Big\|\sum_{i \in C^{(t)}} \widetilde{w}_i h^{(t)}_{y,i}\Big\|^2
    \leq& \frac{n}{P}\sum_{i=1}^n\frac{w_i^2}{\|a_i^{(t)}\|^2_1} \sum_{k=0}^{\tau_i-1}\big(a^{(t,k)}_i\big)^2 \sigma^2_{g} + \frac{2(n-P)\beta_{\max}}{P(n-1)}\sigma_{gh}^2 \nonumber\\
    &+ \bigg(\frac{2n(n-P)}{P(n-1)}\sum_{i=1}^n w_i^2 \sum_{k=0}^{\tau_i-1} \frac{a_i^{(t,k)}}{\|a_i^{(t)}\|_1} + 3\sum_{i=1}^n w_i \sum_{k=0}^{\tau_i-1} \frac{a_i^{(t,k)}}{\|a_i^{(t)}\|_1}
    \bigg)\Delta_{g,i}^{(t,k)} \nonumber \\
    &+ \bigg(\frac{2(n-P)\beta_{\max}\beta_{gh}^2}{P(n-1)} + 3L_1^2\bigg) \mathbb{E}\|y^{(t)} - y^*(x^{(t)})\|^2.  \nonumber
\end{align}
Similarly, by replacing by replacing $h_{y,i}^{(t)}$ and ${\nabla}g$ with $h_{R,i}^{(t)}$ and $\nabla R$, we can easily get 
\begin{align}
    \mathbb{E}\Big\|\sum_{i \in C^{(t)}} \widetilde{w}_i h^{(t)}_{v,i}\Big\|^2 \leq& \frac{2n}{P}\sum_{i=1}^n\frac{w_i^2}{\|a_i^{(t)}\|^2_1} \sum_{k=0}^{\tau_i-1}\big(a^{(t,k)}_i\big)^2 \big(\sigma^2_{f} + r_i^2\sigma^2_{gg}\big) + \frac{4(n-P)\beta_{\max}}{P(n-1)}(r_{\max}^2L_1^2 + L_f^2) \nonumber\\
    &+ \bigg(\frac{2n(n-P)}{P(n-1)}\sum_{i=1}^n w_i^2 \sum_{k=0}^{\tau_i-1} \frac{a_i^{(t,k)}}{\|a_i^{(t)}\|_1} + 3\sum_{i=1}^n w_i \sum_{k=0}^{\tau_i-1} \frac{a_i^{(t,k)}}{\|a_i^{(t)}\|_1}
    \bigg)\Delta_{R,i}^{(t,k)} \nonumber\\
    &+ 3L_1^2 \mathbb{E}\|v^{(t)} - v^*(x^{(t)})\|^2. 
\end{align}
Then, the proof is complete.
\end{proof}

\subsection{Descent in iterates of the inner- and LS-problem}
\begin{lemma}\label{lm:servergap}
Under the Assumption \ref{as:diffandSC}, \ref{as:Lipschitz} and \ref{as:varaince}, the iterates of the inner-problem generated according
to Algorithm \ref{alg:main} satisfy
\begin{align}
    \mathbb{E}\|&y^{(t+1)} - \widetilde{y}^*(x^{(t+1)})\|^2 -  \mathbb{E}\|y^{(t)} - \widetilde{y}^*(x^{(t)})\|^2 \nonumber\\
    \leq& (\delta_t - \rho^{(t)}\gamma_y \mu_g - \delta_t\rho^{(t)}\gamma_y \mu_g) \mathbb{E}\|y^{(t)} - \widetilde{y}^*(x^{(t)})\|^2 + (1+\delta_t)(\rho^{(t)}\gamma_y)^2\mathbb{E}\Big\|\sum_{i \in C^{(t)}} \widetilde{w}_i h^{(t)}_{y,i}\Big\|^2 \nonumber\\
    & + (1+\delta_t)\rho^{(t)}\gamma_y \frac{2L_1^2}{\mu_g}  \sum_{i = 1}^nw_i \sum_{k=0}^{\tau_i-1}\frac{a^{(t,k)}_i}{\|a_i^{(t)}\|_1} \mathbb{E}\Big[\big\|x^{(t)} - x_i^{(t,k)}\big\|^2 + \big\|y^{(t)} - y_i^{(t,k)}\big\|^2\Big] \nonumber\\
    & + \big(\rho^{(t)}\gamma_x\big)^2\bigg(L_y^2+\frac{L_{yx}}{2}\bigg) \mathbb{E}\bigg\|\sum_{i \in C^{(t)}}\widetilde{w}_i h_{x,i}^{(t)}\bigg\|^2 + (\rho^{(t)}\gamma_x)^2\frac{4L_y}{\delta_{t,1}}\mathbb{E}\Big\| \sum_{i=1}^n w_i \widetilde{h}^{(t)}_{x,i}\Big\|^2 .\nonumber
\end{align}
% \begin{align}
%     \mathbb{E}\|y^{(t+1)} &- \widetilde{y}^*(t+1)\|^2 -  \mathbb{E}\|y^{(t)} - \widetilde{y}^*(t)\|^2 \nonumber\\
%     \leq& (\delta_t - \rho^{(t)}\gamma_y \mu_g - \delta_t\rho^{(t)}\gamma_y \mu_g) \mathbb{E}\|y^{(t)} - \widetilde{y}^*(t)\|^2 + (1+\delta_t)(\rho^{(t)}\gamma_y)^2\mathbb{E}\Big\|\sum_{i \in C^{(t)}} \widetilde{w}_i h^{(t)}_{y,i}\Big\|^2 \nonumber\\
%     & + (1+\delta_t)\rho^{(t)}\gamma_y \frac{2L_1^2}{\mu_g}  \sum_{i = 1}^nw_i \sum_{k=0}^{\tau_i-1}\frac{a^{(t,k)}_i}{\|a_i^{(t)}\|_1} \mathbb{E}\Big[\big\|x^{(t)} - x_i^{(t,k)}\big\|^2 + \big\|y^{(t)} - y_i^{(t,k)}\big\|^2\Big] \nonumber\\
%     & + (1+\frac{1}{\delta_t})\big(\rho^{(t)}\gamma_x\big)^2 \frac{L^2_{1}}{\mu_g^2} \mathbb{E}\Big\|\sum_{i \in C^{(t)}} \widetilde{w}_i h^{(t)}_{x,i}\Big\|^2. \nonumber
% \end{align}
% \begin{align}
%     \mathbb{E}\|y^{(t+1)} &- \widetilde{y}^*(t+1)\|^2 -  \mathbb{E}\|y^{(t)} - \widetilde{y}^*(t)\|^2 \nonumber\\
%     \leq& (\delta_t - \rho^{(t)}\gamma_y\frac{\mu_g}{2} - \delta_t\rho^{(t)}\gamma_y\frac{\mu_g}{2}) \mathbb{E}\|y^{(t)} - \widetilde{y}^*(t)\|^2 + (1+\delta_t)(\rho^{(t)}\gamma_y)^2\mathbb{E}\Big\|\sum_{i \in C^{(t)}} \widetilde{w}_i h^{(t)}_{y,i}\Big\|^2 \nonumber\\
%     & + 2(1+\delta_t)\rho^{(t)}\gamma_y L \sum_{i \in C^{(t)}} \widetilde{w}_i \sum_{k=0}^{\tau_i-1}\frac{a^{(t,k)}_i}{\|a_i^{(t)}\|_1} \mathbb{E}\|y^{(t)} - y^{(t,k)}_i\|^2 \nonumber\\
%     & + (1+\frac{1}{\delta_t})\big(\rho^{(t)}\gamma_x\big)^2\frac{L^2_{1}}{\mu_g^2}\mathbb{E}\Big\|\sum_{i \in C^{(t)}} \widetilde{w}_i h^{(t)}_{x,i}\Big\|^2, \nonumber
% \end{align}
and the iterates of the LS problem satisfy
\begin{align}
    \mathbb{E}\|&v^{(t+1)} - \widetilde{v}^*(x^{(t+1)})\|^2 -  \mathbb{E}\|v^{(t)} - \widetilde{v}^*(x^{(t)})\|^2 \nonumber\\
    \leq& (\delta_t' - \rho^{(t)}\gamma_v \mu_g - \delta_t'\rho^{(t)}\gamma_v \mu_g) \mathbb{E}\|v^{(t)} - \widetilde{v}^*(x^{(t)})\|^2 + (1+\delta_t')(\rho^{(t)}\gamma_v)^2\mathbb{E}\Big\|\sum_{i \in C^{(t)}} \widetilde{w}_i h^{(t)}_{v,i}\Big\|^2 \nonumber\\
    & + (1+\delta_t')\rho^{(t)}\gamma_v \frac{4L_R^2}{\mu_g}\sum_{i = 1}^n w_i \sum_{k=0}^{\tau_i-1}\frac{a^{(t,k)}_i}{\|a_i^{(t)}\|_1} \mathbb{E}\Big[\big\|x^{(t)} - x_i^{(t,k)}\big\|^2 + \big\|y^{(t)} - y_i^{(t,k)}\big\|^2 + \big\|v^{(t)} - v_i^{(t,k)}\big\|^2\Big] \nonumber\\
    & + (1+\delta_t')\rho^{(t)}\gamma_v \frac{4L_R^2}{\mu_g}\mathbb{E}\big\|y^{(t)} - \widetilde{y}^*{(x^{(t)})}\big\|^2
    + \big(\rho^{(t)}\gamma_x\big)^2\bigg(L_v^2+\frac{L_{vx}}{2}\bigg) \mathbb{E}\bigg\|\sum_{i \in C^{(t)}}\widetilde{w}_i h_{x,i}^{(t)}\bigg\|^2 \nonumber \\
    &  + (\rho^{(t)}\gamma_x)^2\frac{4L_v}{\delta_{t,1}'}\mathbb{E}\Big\| \sum_{i=1}^n w_i \widetilde{h}^{(t)}_{x,i}\Big\|^2 .\nonumber
\end{align}
for all $t \in \{0,1,...,T-1\}$, $k \in \{0,1,...,\tau_i-1\}$ and $i \in \{1,2,...,n\}$.
\end{lemma}
\begin{proof}
For the gap of $y$ and $\widetilde{y}^*$ on server, we have 
% \begin{align}
%     \mathbb{E}\big \|y^{(t+1)} - \widetilde{y}^*{(t+1)}\big\|^2 \leq& \big(1+\delta_t \big)\mathbb{E}\big\|y^{(t+1)} - \widetilde{y}^*{(t)}\big\|^2 + \Big(1+\frac{1}{\delta_t} \Big)\mathbb{E}\big \|\widetilde{y}^*{(t)} - \widetilde{y}^*{(t+1)}\big\|^2 \nonumber 
% \end{align}
\begin{align}\label{eq:servergap01}
    \mathbb{E}\big \|y^{(t+1)} - \widetilde{y}^*{(x^{(t+1)})}\big\|^2 =& \mathbb{E}\big\|y^{(t+1)} - \widetilde{y}^*{(x^{(t)})}\big\|^2 + \mathbb{E}\big \|\widetilde{y}^*{(x^{(t)})} - \widetilde{y}^*{(t+1)}\big\|^2 \nonumber \\
    & + 2\mathbb{E}\big\langle y^{(t+1)} - \widetilde{y}^*{(x^{(t)})}, \widetilde{y}^*{(x^{(t)})} - \widetilde{y}^*{(x^{(t+1)})} \big\rangle. 
\end{align}
For the last term in \cref{eq:servergap01}, we have 
\begin{align}\label{eq:servergap11}
    2\mathbb{E}\big\langle& y^{(t+1)} - \widetilde{y}^*{(x^{(t)})}, \widetilde{y}^*{(x^{(t)})} - \widetilde{y}^*{(x^{(t+1)})} \big\rangle \nonumber \\
    = & -2\mathbb{E}\Big\langle y^{(t+1)} - \widetilde{y}^*{(x^{(t)})}, \nabla \widetilde{y}^*(x^{(t)})\big(x^{(t+1)}-x^{(t)}\big) \Big\rangle \nonumber \\
    & - 2\mathbb{E}\Big\langle y^{(t+1)} - \widetilde{y}^*{(x^{(t)})}, \widetilde{y}^*(x^{(t+1)}) - \widetilde{y}^*(x^{(t)}) - \nabla \widetilde{y}^*(x^{(t)})\big(x^{(t+1)}-x^{(t)}\big) \Big\rangle \nonumber \\
    \leq & 2\mathbb{E}\big\| y^{(t+1)} - \widetilde{y}^*{(x^{(t)})}\big\|\cdot\mathbb{E}\Big\| \rho^{(t)}\gamma_x\nabla \widetilde{y}^*(x^{(t)})\sum_{i=1}^n w_i \widetilde{h}^{(t)}_{x,i}\Big\| \nonumber \\
    & + 2\mathbb{E}\big\| y^{(t+1)} - \widetilde{y}^*{(x^{(t)})}\big\|\cdot\mathbb{E}\Big\| \widetilde{y}^*(x^{(t+1)}) - \widetilde{y}^*(x^{(t)}) - \nabla \widetilde{y}^*(x^{(t)})\big(x^{(t+1)}-x^{(t)}\big) \Big\| \nonumber \\
    \overset{(a)}{\leq} & 2\mathbb{E}\big\| y^{(t+1)} - \widetilde{y}^*{(x^{(t)})}\big\|\cdot\mathbb{E}\Big\| \rho^{(t)}\gamma_x\nabla \widetilde{y}^*(x^{(t)})\sum_{i=1}^n w_i \widetilde{h}^{(t)}_{x,i}\Big\| \nonumber \\
    & + {L_{yx}}\mathbb{E}\big\| y^{(t+1)} - \widetilde{y}^*{(x^{(t)})}\big\|\cdot\mathbb{E}\big\| x^{(t+1)}-x^{(t)} \big\|^2 \nonumber \\
    \leq & {\delta_{t,1}}\mathbb{E}\big\| y^{(t+1)} - \widetilde{y}^*{(x^{(t)})}\big\|^2 + \frac{4(\rho^{(t)}\gamma_x)^2L_y}{\delta_{t,1}}\mathbb{E}\Big\| \sum_{i=1}^n w_i \widetilde{h}^{(t)}_{x,i}\Big\|^2 \nonumber \\
    & + \frac{L_{yx}}{2}\mathbb{E}\big\| y^{(t+1)} - \widetilde{y}^*{(x^{(t)})}\big\|^2\cdot\mathbb{E}\big\| x^{(t+1)}-x^{(t)} \big\|^2 + \frac{L_{yx}}{2}\mathbb{E}\big\| x^{(t+1)}-x^{(t)} \big\|^2\nonumber \\
    \leq & {\delta_{t,1}}\mathbb{E}\big\| y^{(t+1)} - \widetilde{y}^*{(x^{(t)})}\big\|^2 + \frac{4(\rho^{(t)}\gamma_x)^2L_y}{\delta_{t,1}}\mathbb{E}\Big\| \sum_{i=1}^n w_i \widetilde{h}^{(t)}_{x,i}\Big\|^2 \nonumber \\
    & + {L_{yx}(r_{\max}^2L_1^2 + L_f^2)(\rho^{(t)}\gamma_x)^2}\mathbb{E}\big\| y^{(t+1)} - \widetilde{y}^*{(x^{(t)})}\big\|^2 + \frac{L_{yx}(\rho^{(t)}\gamma_x)^2}{2}\mathbb{E}\Big\| \sum_{i \in C^{(t)}}^n \widetilde{w}_i h^{(t)}_{x,i} \Big\|^2  \nonumber \\
    = & {\delta_{t}}\mathbb{E}\big\| y^{(t+1)} - \widetilde{y}^*{(x^{(t)})}\big\|^2 + \frac{4(\rho^{(t)}\gamma_x)^2L_y}{\delta_{t,1}}\mathbb{E}\Big\| \sum_{i=1}^n w_i \widetilde{h}^{(t)}_{x,i}\Big\|^2 + \frac{L_{yx}(\rho^{(t)}\gamma_x)^2}{2}\mathbb{E}\Big\| \sum_{i \in C^{(t)}}^n \widetilde{w}_i h^{(t)}_{x,i} \Big\|^2  
\end{align}
where (a) follows from \cref{lm:3ieq}; (b) define $\delta_t := \delta_{t,1} + {L_{yx}(r_{\max}^2L_1^2 + L_f^2)(\rho^{(t)}\gamma_x)^2}/2$. 
By incorporating \cref{eq:servergap11} into \cref{eq:servergap01}, we have
\begin{align}\label{eq:servergap1}
    \mathbb{E}\big \|y^{(t+1)} - \widetilde{y}^*{(x^{(t+1)})}\big\|^2 &\leq (1+\delta_t)\mathbb{E}\big\|y^{(t+1)} - \widetilde{y}^*{(x^{(t)})}\big\|^2 + \mathbb{E}\big \|\widetilde{y}^*{(x^{(t)})} - \widetilde{y}^*{(x^{(t+1)})}\big\|^2 \nonumber \\
    & \ \ \ + (\rho^{(t)}\gamma_x)^2\frac{4L_y}{\delta_{t,1}}\mathbb{E}\Big\| \sum_{i=1}^n w_i \widetilde{h}^{(t)}_{x,i}\Big\|^2 + (\rho^{(t)}\gamma_x)^2\frac{L_{yx}}{2}\mathbb{E}\Big\| \sum_{i \in C^{(t)}}^n \widetilde{w}_i h^{(t)}_{x,i} \Big\|^2. 
\end{align}
Similarly, we have 
\begin{align}
    \mathbb{E}\big \|v^{(t+1)} - \widetilde{v}^*{(x^{(t+1)})}\big\|^2 &\leq (1+\delta_t')\mathbb{E}\big\|v^{(t+1)} - \widetilde{v}^*{(x^{(t)})}\big\|^2 + \mathbb{E}\big \|\widetilde{v}^*{(x^{(t)})} - \widetilde{v}^*{(x^{(t+1)})}\big\|^2 \nonumber \\
    & \ \ \ + (\rho^{(t)}\gamma_x)^2\frac{4L_v}{\delta_{t,1}'}\mathbb{E}\Big\| \sum_{i=1}^n w_i \widetilde{h}^{(t)}_{x,i}\Big\|^2 + (\rho^{(t)}\gamma_x)^2\frac{L_{vx}}{2}\mathbb{E}\Big\| \sum_{i \in C^{(t)}}^n \widetilde{w}_i h^{(t)}_{x,i} \Big\|^2,
\end{align}
where $\delta_t' := \delta_{t,1}' + {L_{vx}(r_{\max}^2L_1^2 + L_f^2)(\rho^{(t)}\gamma_x)^2}/2$. 
For the first part in \cref{eq:servergap01}, we have 
\begingroup
\allowdisplaybreaks
\begin{align}\label{eq:servergap2}
    \mathbb{E}\big\|&y^{(t+1)} - \widetilde{y}^*{(x^{(t)})}\big\|^2 \nonumber \\
    &=  \mathbb{E}\bigg\|y^{(t)} - \widetilde{y}^*{(x^{(t)})} - \rho^{(t)}\gamma_y\sum_{i \in C{(t)}}\widetilde{w}_ih_{y,i}^{(t)}\bigg\|^2 \nonumber \\
    &=  \mathbb{E}\bigg\|y^{(t)} - \widetilde{y}^*{(x^{(t)})}\bigg\|^2 + \big(\rho^{(t)}\gamma_y\big)^2\mathbb{E}\bigg\|\sum_{i \in C{(t)}}\widetilde{w}_ih_{y,i}^{(t)}\bigg\|^2 - 2\rho^{(t)}\gamma_y \mathbb{E}\bigg\langle y^{(t)}-\widetilde{y}^*(x^{(t)}), \sum_{i \in C{(t)}}\widetilde{w}_ih_{y,i}^{(t)}\bigg\rangle.
\end{align}
\endgroup
For the last term in \cref{eq:servergap2}, we have 
\begingroup
\allowdisplaybreaks
\begin{align}\label{eq:servergap3}
     - \mathbb{E}\bigg\langle & y^{(t)} - \widetilde{y}^*(x^{(t)}), \sum_{i \in C{(t)}}\widetilde{w}_ih_{y,i}^{(t)}\bigg\rangle \nonumber \\
     & =  - \mathbb{E}\bigg\langle y^{(t)}-\widetilde{y}^*(x^{(t)}), \sum_{i = 1}^nw_i\widetilde{h}_{y,i}^{(t)} - \nabla_y \widetilde{G}(x^{(t)}, y^{(t)}) + \nabla_y \widetilde{G}(x^{(t)}, y^{(t)}) - \nabla_y \widetilde{G}\big(x^{(t)}, \widetilde{y}^*{(x^{(t)})}\big)\bigg\rangle \nonumber \\
     % & = - \sum_{i = 1}^nw_i \sum_{k=0}^{\tau_i-1}\frac{a_i^{(t,k)}}{\|a_i^{(t)}\|_1}\mathbb{E}\bigg\langle y^{(t)}-\widetilde{y}^*(x^{(t)}), \bigg\rangle \nonumber \\
     & = - \sum_{i = 1}^nw_i \sum_{k=0}^{\tau_i-1}\frac{a_i^{(t,k)}}{\|a_i^{(t)}\|_1}\mathbb{E}\Big\langle y^{(t)}-\widetilde{y}^*(x^{(t)}), \nabla_y g_i \big(x_i^{(t,k)}, y_i^{(t,k)}\big) - \nabla_y g_i \big(x^{(t)}, y^{(t)}\big)\Big\rangle \nonumber \\
     & \ \ \ \ - \sum_{i = 1}^nw_i \sum_{k=0}^{\tau_i-1}\frac{a_i^{(t,k)}}{\|a_i^{(t)}\|_1}\mathbb{E}\Big\langle y^{(t)}-\widetilde{y}^*(x^{(t)}), \nabla_y g_i \big(x^{(t)}, y^{(t)}\big) - \nabla_y g_i \big(x^{(t)}, y^*{(x^{(t)})}\big)\Big\rangle \nonumber \\
     & \overset{(a)}{\leq} \sum_{i = 1}^nw_i \sum_{k=0}^{\tau_i-1}\frac{a_i^{(t,k)}}{\|a_i^{(t)}\|_1} \mathbb{E}\bigg[\frac{1}{\mu_g}\Big\| \nabla_y g_i \big(x^{(t)}, y^{(t)}\big) - \nabla_y g_i \big(x^{(t)}, y^*{(x^{(t)})}\big)\Big\|^2 + \frac{\mu_g}{2}\big\|y^{(t)}-\widetilde{y}^*(x^{(t)})\big\|^2\bigg]\nonumber \\
     & \ \ \ \ -\sum_{i = 1}^nw_i \sum_{k=0}^{\tau_i-1}\frac{a_i^{(t,k)}}{\|a_i^{(t)}\|_1}{\mu_g}\mathbb{E}\big\|y^{(t)}-\widetilde{y}^*(x^{(t)})\big\|^2 \nonumber \\
     & \overset{(b)}{\leq} \frac{L_1^2}{\mu_g}\sum_{i = 1}^nw_i \sum_{k=0}^{\tau_i-1}\frac{a_i^{(t,k)}}{\|a_i^{(t)}\|_1}
     \mathbb{E}\Big[\big\|x^{(t)} - x_i^{(t,k)}\big\|^2 + \big\|y^{(t)} - y_i^{(t,k)}\big\|^2\Big] - \frac{\mu_g}{2}\mathbb{E}\big\|y^{(t)} - \widetilde{y}^*{(x^{(t)})}\big\|^2,
\end{align}
% \begin{align}\label{eq:servergap3}
%      - \mathbb{E}\bigg\langle y^{(t)}&-\widetilde{y}^*(t), \sum_{i \in C{(t)}}\widetilde{w}_ih_{y,i}^{(t)}\bigg\rangle \nonumber \\
%      & =  - \mathbb{E}\bigg\langle y^{(t)}-\widetilde{y}^*(t), \sum_{i = 1}^nw_i\widetilde{h}_{y,i}^{(t)}\bigg\rangle \nonumber \\
%      & =  - \sum_{i = 1}^nw_i \sum_{k=0}^{\tau_i-1}\frac{a_i^{(t,k)}}{\|a_i^{(t)}\|_1}\mathbb{E}\Big\langle y^{(t)}-\widetilde{y}^*(t), \nabla_y g_i \big(x_i^{(t,k)}, y_i^{(t,k)}\big)\Big\rangle \nonumber \\
%      & \overset{(a)}{\leq} \sum_{i = 1}^nw_i \sum_{k=0}^{\tau_i-1}\frac{a_i^{(t,k)}}{\|a_i^{(t)}\|_1}\mathbb{E}
%      \Big[g_i(x_i^{(t,k)}, \widetilde{y}^*(t)) - g_i(x_i^{(t,k)}, y^{(t)}) \Big]\nonumber \\
%      &\quad+ \sum_{i = 1}^nw_i \sum_{k=0}^{\tau_i-1}\frac{a_i^{(t,k)}}{\|a_i^{(t)}\|_1}\mathbb{E}
%      \Big[L\big\|y^{(t)} - y_i^{(t,k)}\big\|^2 - \frac{\mu_g}{4}\big\|y^{(t)} - \widetilde{y}^*{(t)}\big\|^2\Big]\nonumber \\
%      & \overset{(b)}{\leq}\sum_{i = 1}^nw_i \sum_{k=0}^{\tau_i-1}\frac{a_i^{(t,k)}}{\|a_i^{(t)}\|_1}\mathbb{E}
%      \Big[L\big\|y^{(t)} - y_i^{(t,k)}\big\|^2 - \frac{\mu_g}{4}\big\|y^{(t)} - \widetilde{y}^*{(t)}\big\|^2\Big] \nonumber \\
%      & = L\sum_{i = 1}^nw_i \sum_{k=0}^{\tau_i-1}\frac{a_i^{(t,k)}}{\|a_i^{(t)}\|_1}
%      \mathbb{E}\big\|y^{(t)} - y_i^{(t,k)}\big\|^2 - \frac{\mu_g}{4}\big\|y^{(t)} - \widetilde{y}^*{(t)}\big\|^2
% \end{align}
where (a) follows from the strong convexity of $g_i$;
(b) uses Assumption \ref{as:Lipschitz}. 
\endgroup
Incorporate \cref{eq:servergap3} into \cref{eq:servergap2} and we have
\begin{align}\label{eq:servergap4}
    \mathbb{E}\big\|y^{(t+1)} - \widetilde{y}^*{(x^{(t)})}\big\|^2
    &=  \bigg(1-\rho^{(t)}\gamma_y \mu_g \bigg)\mathbb{E}\Big\|y^{(t)} - \widetilde{y}^*{(x^{(t)})}\Big\|^2 + \big(\rho^{(t)}\gamma_y\big)^2\mathbb{E}\bigg\|\sum_{i \in C{(t)}}\widetilde{w}_ih_{y,i}^{(t)}\bigg\|^2 \nonumber \\
    &\quad + 2\rho^{(t)}\gamma_y \frac{L_1^2}{\mu_g} \sum_{i = 1}^nw_i \sum_{k=0}^{\tau_i-1}\frac{a_i^{(t,k)}}{\|a_i^{(t)}\|_1} \mathbb{E}\Big[\big\|x^{(t)} - x_i^{(t,k)}\big\|^2 + \big\|y^{(t)} - y_i^{(t,k)}\big\|^2\Big]. 
\end{align}
For the first part in \ref{eq:servergap01}, using \Cref{lm:3ieq}, we have
\begin{align}\label{eq:servergap5}
    \mathbb{E}\big\|\widetilde{y}^*(x^{(t)}) - \widetilde{y}^*(x^{(t+1)})\big\|^2 \leq L_y^2\big\|x^{(t)} - x^{(t+1)}\big\|^2 = L_y^2\big(\rho^{(t)}\gamma_x\big)^2 \mathbb{E}\bigg\|\sum_{i \in C^{(t)}}\widetilde{w}_i h_{x,i}^{(t)}\bigg\|^2.
\end{align}
By incorporating \cref{eq:servergap4} and \cref{eq:servergap5} into \cref{eq:servergap01}, we have
\begingroup
\allowdisplaybreaks
\begin{align}
    \mathbb{E}\|&y^{(t+1)} - \widetilde{y}^*(x^{(t+1)})\|^2 -  \mathbb{E}\|y^{(t)} - \widetilde{y}^*(x^{(t)})\|^2 \nonumber\\
    \leq& (\delta_t - \rho^{(t)}\gamma_y \mu_g - \delta_t\rho^{(t)}\gamma_y \mu_g) \mathbb{E}\|y^{(t)} - \widetilde{y}^*(x^{(t)})\|^2 + (1+\delta_t)(\rho^{(t)}\gamma_y)^2\mathbb{E}\Big\|\sum_{i \in C^{(t)}} \widetilde{w}_i h^{(t)}_{y,i}\Big\|^2 \nonumber\\
    & + (1+\delta_t)\rho^{(t)}\gamma_y \frac{2L_1^2}{\mu_g}  \sum_{i = 1}^nw_i \sum_{k=0}^{\tau_i-1}\frac{a^{(t,k)}_i}{\|a_i^{(t)}\|_1} \mathbb{E}\Big[\big\|x^{(t)} - x_i^{(t,k)}\big\|^2 + \big\|y^{(t)} - y_i^{(t,k)}\big\|^2\Big] \nonumber\\
    & + \big(\rho^{(t)}\gamma_x\big)^2\bigg(L_y^2+\frac{L_{yx}}{2}\bigg) \mathbb{E}\bigg\|\sum_{i \in C^{(t)}}\widetilde{w}_i h_{x,i}^{(t)}\bigg\|^2 + (\rho^{(t)}\gamma_x)^2\frac{4L_y}{\delta_{t,1}}\mathbb{E}\Big\| \sum_{i=1}^n w_i \widetilde{h}^{(t)}_{x,i}\Big\|^2 .\nonumber
\end{align}
% Note that 
% \begingroup
% \allowdisplaybreaks
% \begin{align}\label{eq:smoothnessofR}
%     \big\|\nabla_v R(x,y,v_1) - \nabla_v R(x,y,v_2)\big\| &= \big\|\nabla^2_{yy}g(x,y)(v1-v2)\big\| \nonumber \\
%     &\leq \big\|\nabla^2_{yy}g(x,y)\big\|\cdot\big\|v1-v2\big\| \nonumber \\
%     &\leq L_1\big\|v1-v2\big\|, 
% \end{align}
% which represents that $R(x,y,v)$ is $L_1$-smooth w.r.t. $v$. 
% \endgroup
\endgroup
\noindent
Following similar steps of \cref{eq:servergap3} by replacing $h^{(t)}_{y,i}$ and $\nabla_y g_i$ with $h^{(t)}_{v,i}$ and $\nabla_v R_i$, we can easily have 
\begingroup
\allowdisplaybreaks
\begin{align}
     &- \mathbb{E}\bigg\langle v^{(t)} - \widetilde{v}^*(x^{(t)}), \sum_{i \in C{(t)}}\widetilde{w}_ih_{v,i}^{(t)}\bigg\rangle \nonumber \\
     & = - \sum_{i = 1}^nw_i \sum_{k=0}^{\tau_i-1}\frac{a_i^{(t,k)}}{\|a_i^{(t)}\|_1}\mathbb{E}\Big\langle v^{(t)}-\widetilde{v}^*(x^{(t)}), \nabla_v R_i \big(x_i^{(t,k)}, y_i^{(t,k)}, v_i^{(t,k)}\big) - \nabla_v R_i \big(x^{(t)}, y^{(t)}, v^{(t)}\big)\Big\rangle \nonumber \\
     & \ \ \ \ - \sum_{i = 1}^nw_i \sum_{k=0}^{\tau_i-1}\frac{a_i^{(t,k)}}{\|a_i^{(t)}\|_1}\mathbb{E}\Big\langle v^{(t)}-\widetilde{v}^*(x^{(t)}), \nabla_v R_i \big(x^{(t)}, y^{(t)}, v^{(t)}\big) - \nabla_v R_i \big(x^{(t)}, y^*{(x^{(t)})}, v^{(t)}\big)\Big\rangle \nonumber \\
     & \ \ \ \ - \sum_{i = 1}^nw_i \sum_{k=0}^{\tau_i-1}\frac{a_i^{(t,k)}}{\|a_i^{(t)}\|_1}\mathbb{E}\Big\langle v^{(t)}-\widetilde{v}^*(x^{(t)}), \nabla_v R_i \big(x^{(t)}, y^*{(x^{(t)})}, v^{(t)}\big) - \nabla_v R_i \big(x^{(t)}, y^*{(x^{(t)})}, v^*{(x^{(t)})}\big)\Big\rangle \nonumber \\
     & \overset{(a)}{\leq} \frac{2}{\mu_g}\sum_{i = 1}^nw_i \sum_{k=0}^{\tau_i-1}\frac{a_i^{(t,k)}}{\|a_i^{(t)}\|_1} \mathbb{E}\Big\| \nabla_v R_i \big(x_i^{(t,k)}, y_i^{(t,k)}, v_i^{(t,k)}\big) - \nabla_v R_i \big(x^{(t)}, y^{(t)}, v^{(t)}\big)\Big\|^2 \nonumber \\
     & \ \ \ \ + \frac{2}{\mu_g}\sum_{i = 1}^nw_i \sum_{k=0}^{\tau_i-1}\frac{a_i^{(t,k)}}{\|a_i^{(t)}\|_1} \mathbb{E}\Big\| \nabla_v R_i \big(x^{(t)}, y^{(t)}, v^{(t)}\big) - \nabla_v R_i \big(x^{(t)}, y^*{(x^{(t)})}, v^{(t)}\big)\Big\|^2 \nonumber \\
     & \ \ \ \ -\frac{\mu_g}{2}\sum_{i = 1}^nw_i \sum_{k=0}^{\tau_i-1}\frac{a_i^{(t,k)}}{\|a_i^{(t)}\|_1}\mathbb{E}\big\|v^{(t)}-\widetilde{v}^*(x^{(t)})\big\|^2 \nonumber \\
     & \overset{(b)}{\leq} \frac{2L_R^2}{\mu_g}\sum_{i = 1}^nw_i \sum_{k=0}^{\tau_i-1}\frac{a_i^{(t,k)}}{\|a_i^{(t)}\|_1}
     \mathbb{E}\Big[\big\|x^{(t)} - x_i^{(t,k)}\big\|^2 + \big\|y^{(t)} - y_i^{(t,k)}\big\|^2 + \big\|v^{(t)} - v_i^{(t,k)}\big\|^2\Big] \nonumber \\
     & \ \ \ \ + \frac{2L_R^2}{\mu_g}\mathbb{E}\big\|y^{(t)} - \widetilde{y}^*{(x^{(t)})}\big\|^2 - \frac{\mu_g}{2}\mathbb{E}\big\|v^{(t)} - \widetilde{v}^*{(x^{(t)})}\big\|^2, \nonumber
\end{align}
where (a) uses the strong convexity of $R_i$; (b) follows from \Cref{lm:propertiesofR}. 
Since the projection is non-expansive, we can easily have that 
\begin{align}
    \mathbb{E}\|v^{(t+1)} - v^*(x^{(t)})\|^2 &= \mathbb{E} \| \mathcal{P}_r(v^{(t)} - \rho^{(t)} \gamma_v \sum_{i \in C^{(t)}} \widetilde{w}_i h_{v,i}^{(t)}) - v^*(x^{(t)}) \|^2 \nonumber \\
    &\leq \mathbb{E}||v^{(t)} - v^*(x^{(t)}) - \rho^{(t)}\gamma_v\sum_{i \in C^{(t)}}\widetilde{w}_i h_{v,i}^{(t)}||^2 \nonumber
\end{align}
Thus, we can have the similar result 
\begin{align}
    \mathbb{E}\|&v^{(t+1)} - \widetilde{v}^*(x^{(t+1)})\|^2 -  \mathbb{E}\|v^{(t)} - \widetilde{v}^*(x^{(t)})\|^2 \nonumber\\
    \leq& (\delta_t' - \rho^{(t)}\gamma_v \mu_g - \delta_t'\rho^{(t)}\gamma_v \mu_g) \mathbb{E}\|v^{(t)} - \widetilde{v}^*(x^{(t)})\|^2 + (1+\delta_t')(\rho^{(t)}\gamma_v)^2\mathbb{E}\Big\|\sum_{i \in C^{(t)}} \widetilde{w}_i h^{(t)}_{v,i}\Big\|^2 \nonumber\\
    & + (1+\delta_t')\rho^{(t)}\gamma_v \frac{4L_R^2}{\mu_g}\sum_{i = 1}^n w_i \sum_{k=0}^{\tau_i-1}\frac{a^{(t,k)}_i}{\|a_i^{(t)}\|_1} \mathbb{E}\Big[\big\|x^{(t)} - x_i^{(t,k)}\big\|^2 + \big\|y^{(t)} - y_i^{(t,k)}\big\|^2 + \big\|v^{(t)} - v_i^{(t,k)}\big\|^2\Big] \nonumber\\
    & + (1+\delta_t')\rho^{(t)}\gamma_v \frac{4L_R^2}{\mu_g}\mathbb{E}\big\|y^{(t)} - \widetilde{y}^*{(x^{(t)})}\big\|^2
    + \big(\rho^{(t)}\gamma_x\big)^2\bigg(L_v^2+\frac{L_{vx}}{2}\bigg) \mathbb{E}\bigg\|\sum_{i \in C^{(t)}}\widetilde{w}_i h_{x,i}^{(t)}\bigg\|^2 \nonumber \\
    &  + (\rho^{(t)}\gamma_x)^2\frac{4L_v}{\delta_{t,1}'}\mathbb{E}\Big\| \sum_{i=1}^n w_i \widetilde{h}^{(t)}_{x,i}\Big\|^2 .\nonumber
\end{align}
\endgroup
Then, the proof is complete. 
\end{proof}

\subsection{Descent in the Lyapunov Function}\label{eq:setting}
% First, we define the server learning rate as $\gamma_x$, $\gamma_y$ and $\gamma_v$ as 
% \begin{align}
%     \gamma_y = c_{\gamma_y}\gamma_x, \gamma_v = c_{\gamma_v}\gamma_x
% \end{align}
% for some positive constants $c_{\gamma_y}$ and  $c_{\gamma_v}$. 
We define the Lyapunov function as 
\begin{align}\label{def:lyapunov}
    \Psi(x^{(t)}) := \mathbb{E}\Big[\widetilde{\Phi}(x^{(t)})\Big] + K_1\mathbb{E}\|y^{(t)} - \widetilde{y}^*{(x^{(t)})}\|^2 + K_2\mathbb{E}\|v^{(t)} - \widetilde{v}^*{(x^{(t)})}\|^2,
\end{align}
where the coefficients are given by
\begin{align}
    K_1 = \bigg[\frac{40(L_1^2+r^2L_2^2)}{\mu_g} + \frac{384L_R^2L_1^2}{\mu_g^3}\bigg]\frac{1}{c_{\gamma_y}} ,\quad 
    K_2 = \frac{6L_1^2}{\mu_g c_{\gamma_v}};\quad  \delta_t = \frac{\rho^{(t)}\gamma_y\mu_g}{4}, \quad
    \delta_t' = \frac{\rho^{(t)}\gamma_v\mu_g}{4}. \nonumber
\end{align}
% \begin{align} date:7/17
%     K_1  = K_{1}'\frac{\gamma_x}{\gamma_y} = \bigg[\frac{12(L_1^2+r^2L_2^2)}{\mu_g} + \frac{144L_R^2L_1^2}{\mu_g^3}\bigg]\frac{\gamma_x}{\gamma_y} ,
%     K_2  = K_{2}'\frac{\gamma_x}{\gamma_v} = \frac{6L_1^2\gamma_x}{\mu_g \gamma_v}; 
%     \delta_t = \frac{\rho^{(t)}\gamma_y\mu_g}{4},
%     \delta_t' = \frac{\rho^{(t)}\gamma_v\mu_g}{4}. \nonumber
% \end{align}
% \begin{align}
%     K_1  = K_{1}'\frac{\gamma_x}{\gamma_y}& = \frac{24(L_1^2+r_i^2L_2^2)\gamma_x}{\mu_g\gamma_y} ,\quad
%     K_2  = K_{2}'\frac{\gamma_x}{\gamma_v} = \frac{12L_2^2\gamma_x}{\mu_g \gamma_v};  \nonumber \\
%     \delta_t = &\frac{\rho^{(t)}\gamma_y\mu_g}{8-2\rho^{(t)}\gamma_y\mu_g},\quad \delta_t' = \frac{\rho^{(t)}\gamma_v\mu_g}{8-2\rho^{(t)}\gamma_v\mu_g}. 
% \end{align}
% \begin{align}
%     K_1 &= K_{11}\frac{\gamma_x}{\gamma_y} + K_{12}\frac{\gamma_x}{\gamma_v} = \frac{48(L_1^2+r_i^2L_262)\gamma_x}{\mu_g\gamma_y} + \frac{18432L_1^2L_2^2\gamma_x}{\mu_g^3\gamma_v}\bigg(\frac{2L_1^2}{\mu_g^2} + \frac{2L^2L_2^2}{\mu_g^4}\bigg), \nonumber \\
%     K_2 &= K_{22}\frac{\gamma_x}{\gamma_v} = \frac{24L_2^2\gamma_x}{\mu_g \gamma_v}; \quad \delta_t = \frac{\rho^{(t)}\gamma_y\mu_g}{16-2\rho^{(t)}\gamma_y\mu_g},\quad \delta_t' = \frac{\rho^{(t)}\gamma_v\mu_g}{16-2\rho^{(t)}\gamma_v\mu_g}. 
% \end{align}
For server and local stepsizes, we choose 
\begin{align}\label{eq:stepsize}
    \gamma_x = \mathcal{O}\Big(\sqrt{\frac{P}{\bar{\tau}T}}\Big), \quad \gamma_y &= c_{\gamma_y}\gamma_x, \quad \gamma_v = c_{\gamma_v}\gamma_x, \nonumber \\
    \eta_x = \mathcal{O}\Big(\frac{1}{\bar{\tau}\sqrt{T}}\Big),\  \eta_y = \mathcal{O}&\Big(\frac{1}{\bar{\tau}\sqrt{T}}\Big),\  \eta_v = \mathcal{O}\Big(\frac{1}{\bar{\tau}\sqrt{T}}\Big).
\end{align}
where $c_{\gamma_y} \geq \frac{256K_1L_y}{\mu_g}$ and $c_{\gamma_v} \geq \frac{256K_2L_v}{\mu_g}$. 
% where $c_{\gamma_y} \geq \frac{96K_1L^2_1}{\mu_g^3}$ and $c_{\gamma_v} \geq \frac{96K_2}{\mu_g}\big(\frac{2L^2_1}{\mu_g^2}+\frac{2L^2L_2^2}{\mu_g^4}\big)$. 
We also set 
\begin{align}\label{eq:restriction}
    \rho^{(t)}\gamma_x \leq \min\bigg\{&\frac{\mu_g}{12L_1^2c_{\gamma_v}}, \frac{\mu_g}{36L_1^2c_{\gamma_v}}\Big(\frac{2\beta_{\max}}{P}+3\Big)^{-1}, \frac{P\mu_gc_{\gamma_v}}{36L_1^2\beta_{\max}}\Big(L_v^2+\frac{L_{vx}}{2}\Big)^{-1}, \frac{P}{6L_{\Phi}^2\beta_{\max}}, \frac{4}{L_{\Phi}^2},\frac{4}{c_{\gamma_y}}, \frac{4}{c_{\gamma_v}}, \nonumber \\
    & \frac{\mu_g}{12c_{\gamma_y}}\Big[\big(\frac{2\beta_{\max}}{P}+3\big)L_1^2 + \frac{2\beta_{\max}\beta_{gh}^2}{P}+3L_1^2\Big]^{-1}, \frac{Pc_{\gamma_y}\mu_g}{36\beta_{\max}(L_1^2+r^2L_2^2)}(L_v^2+\frac{L_{yx}}{2})^{-1}, \nonumber \\
    &\frac{c_{\gamma_y}\mu_g}{4L_{yx}(L_f^2+r_{\max}^2L_1^2)}, \frac{c_{\gamma_v}\mu_g}{4L_{vx}(L_f^2+r_{\max}^2L_1^2)}, \frac{1}{8}\Big(K_1(L_y^2+\frac{L_{yx}}{2})+K_2(L_v^2+\frac{L_{vx}}{2})\Big)^{-1} \bigg\};\nonumber \\
    &\eta_x^2 \bar{\tau} \leq \frac{1}{2c_a\alpha_{\max}L_1^2},\quad \eta_y^2 \bar{\tau} \leq \min\bigg\{\frac{1}{4c_a\alpha_{\max}L_1^2}, \frac{\mu_g^2}{96c_a\alpha_{\max}L_1^4}\bigg\}. 
\end{align}
To simplify our proof, we define constants:
\begin{align}\label{eq:constants}
    \bar{\tau} :&= \frac{1}{n}\sum_{i=1}^n \tau_i ,\quad \bar{\rho} := \frac{1}{T}\sum_{t=0}^{T-1}\rho^{(t)},  \nonumber \\
    % \bar{\tau} :=& \sum_{i=1}^n w_i\tau_i \approx \sum_{i \in C^{(t)}} \widetilde{w}_i\tau_i,\quad \bar{\bar{\tau}} := \sum_{i=1}^n w_i\tau_i^2 \approx \sum_{i \in C^{(t)}} \widetilde{w}_i\tau_i^2, \quad  \bar{\rho} := \frac{1}{T}\sum_{t=0}^{T-1}\rho^{(t)},  \nonumber \\
    M_1 :&= 2\big[\frac{L_{\Phi}^2}{2}+K_1\big(L_y^2+\frac{L_{yx}}{2}\big)+K_2\big(L_v^2+\frac{L_{vx}}{2}\big) + 2K_2c_{\gamma_v}^2\big]\beta_{\max}(L_f^2+r_{\max}^2L_1^2) + K_1c_{\gamma_y}^2\beta_{\max}\sigma_{gh}^2 \nonumber \\
    M_2 :&= 2\big[\frac{L_{\Phi}^2}{2}+K_1\big(L_y^2+\frac{L_{yx}}{2}\big)+K_2\big(L_v^2+\frac{L_{vx}}{2}\big) + 2K_2c_{\gamma_v}^2\big](\sigma^2_{f} + r_{\max}^2\sigma^2_{gg}) + K_1c_{\gamma_y}^2\sigma_{g}^2, \nonumber \\
    M_{3} :&= 3\bigg[\frac{3}{2}+ 24K_2c_{\gamma_v} + \frac{\beta_{\max}}{P}\Big(\frac{17}{4} + 16K_2c_{\gamma_v}\Big) \bigg](L_1^2+r^2L_2^2) \nonumber \\
    &\quad + K_1\bigg[4c_{\gamma_y}\frac{L_1^2}{\mu_g} + 8c_{\gamma_v}\Big(\frac{2\beta_{\max}}{P}+3\Big)L_1^2\bigg] + K_2c_{\gamma_v} \frac{4L_R^2}{\mu_g}. 
\end{align}
% After taking $\big\{2\eta_x^2\bar{\tau}c_a\alpha_{\max}L_1^2 \leq 1,\  4\eta_y^2\bar{\tau}c_a\alpha_{\max}L_1^2 \leq 1,\  16\eta_y^2\bar{\tau}c_a\alpha_{\max}L_1^4 \leq 1, \ 384\rho^{(t)}\gamma_x\eta_y^2\bar{\tau}(L_y^2+L_{yx}/4)\alpha_{\max}^2(L_1^2+r^2L_2^2)c_aL_1^2 \leq Pc_{\gamma_{y}}\mu_g, \ 64\rho^{(t)}\gamma_y^2\eta_y^2\Bar{\tau}(2\alpha_{\max}/P+3)c_a\alpha_{\max}L_1^4 \leq \mu_g, \ \rho^{(t)}\gamma_v \leq \mu_g/6L_1^2, \big\}$, 
We apply \Cref{lm:obj_function} and \Cref{lm:servergap} to \cref{def:lyapunov}, and incorporate \Cref{lm:boundofAE}, then we have 
\begingroup
\allowdisplaybreaks
\begin{align}\label{eq:Lyapunov1}
    &\Psi(x^{(t+1)})-\Psi(x^{(t)}) \nonumber \\
    &= \mathbb{E}\Big[\widetilde{\Phi}(x^{(t+1)}) - \widetilde{\Phi}(x^{(t)})\Big] + K_1\mathbb{E}\Big[\|y^{(t+1)} - \widetilde{y}^*{(x^{(t+1)})}\|^2 - \|y^{(t)} - \widetilde{y}^*{(x^{(t)})}\|^2\Big] \nonumber \\ 
    &\quad + K_2\mathbb{E}\Big[\|v^{(t+1)} - \widetilde{v}^*{(x^{(t+1)})}\|^2 - \|v^{(t)} - \widetilde{v}^*{(x^{(t)})}\|^2\Big] \nonumber \\
    & \overset{(a)}{=} -\frac{\rho^{(t)}\gamma_x}{2} \mathbb{E}\Big\|\nabla \widetilde{\Phi}(x^{(t)})\Big\|^2  \nonumber \\
    &\quad + 3(\rho^{(t)}\gamma_x)\bigg[\frac{3}{2}+ 24K_2c_{\gamma_v} + \frac{\beta_{\max}}{P}\Big(\frac{17}{4} + 16K_2c_{\gamma_v}\Big) \bigg](L_1^2+r^2L_2^2)  \Big((\eta_x^2+\eta_v^2)\bar{\tau}\sigma_{M1}^2 + 2\eta_y^2\bar{\tau}\sigma_{M2}^2\Big) \nonumber \\
    &\quad + \rho^{(t)}\gamma_xK_1\bigg[4c_{\gamma_y}\frac{L_1^2}{\mu_g} + 8c_{\gamma_v}\Big(\frac{2\beta_{\max}}{P}+3\Big)L_1^2\bigg]\Big(\eta_x^2\bar{\tau}\sigma_{M1}^2 +2\eta_y^2\bar{\tau}\sigma_{M2}^2\Big) \nonumber \\
    &\quad + \rho^{(t)}\gamma_xK_2c_{\gamma_v} \frac{4L_R^2}{\mu_g}\Big((\eta_x^2+\eta_v^2)\bar{\tau}\sigma_{M1}^2 + 2\eta_y^2\bar{\tau}\sigma_{M2}^2\Big) \nonumber \\
    &\quad + (\rho^{(t)}\gamma_x)^2\bigg[\frac{L_{\Phi}^2}{2}+K_1\big(L_y^2+\frac{L_{yx}}{2}\big)+K_2\big(L_v^2+\frac{L_{vx}}{2}\big) + 2K_2c_{\gamma_v}^2\bigg]\frac{2n}{P}\sum_{i=1}^n\frac{w_i^2\|a_i^{(t)}\|^2_2}{\|a_i^{(t)}\|^2_1} (\sigma^2_{f} + r_i^2\sigma^2_{gg}) \nonumber \\
    &\quad + (\rho^{(t)}\gamma_x)^2 K_1c_{\gamma_y}^2 \frac{n}{P}\sum_{i=1}^n\frac{w_i^2\|a_i^{(t)}\|^2_2}{\|a_i^{(t)}\|^2_1} \sigma^2_{g} \nonumber \\
    &\quad + (\rho^{(t)}\gamma_x)^2\bigg[\frac{L_{\Phi}^2}{2}+K_1\big(L_y^2+\frac{L_{yx}}{2}\big)+K_2\big(L_v^2+\frac{L_{vx}}{2}\big) + 2K_2c_{\gamma_v}^2\bigg]\frac{4(n-P)\beta_{\max}}{P(n-1)}(L_f^2+r_{\max}^2L_1^2) \nonumber \\
    &\quad + (\rho^{(t)}\gamma_x)^2 K_1c_{\gamma_y}^2\frac{2(n-P)\beta_{\max}}{P(n-1)}\sigma_{gh}^2
\end{align}
where (a) holds when we set $K_1$, $K_2$, $\delta_t$ and $\delta_t'$ as \cref{eq:constants}. 
\endgroup
We rearrange \cref{eq:Lyapunov1} and separate it into three error terms. Here, we define the error from full synchronization as 
\begingroup
\allowdisplaybreaks
\begin{align}\label{eq:errorsync}
    \epsilon_{sync}^{(t)} :&= (\rho^{(t)}\gamma_x)^2\bigg[\frac{L_{\Phi}^2}{2}+K_1\big(L_y^2+\frac{L_{yx}}{2}\big)+K_2\big(L_v^2+\frac{L_{vx}}{2}\big) + 2K_2c_{\gamma_v}^2\bigg]\frac{2n}{P}\sum_{i=1}^n\frac{w_i^2\|a_i^{(t)}\|^2_2}{\|a_i^{(t)}\|^2_1} (\sigma^2_{f} + r_i^2\sigma^2_{gg}) \nonumber \\
    &\quad + (\rho^{(t)}\gamma_x)^2 K_1c_{\gamma_y}^2 \frac{n}{P}\sum_{i=1}^n\frac{w_i^2\|a_i^{(t)}\|^2_2}{\|a_i^{(t)}\|^2_1} \sigma^2_{g} 
\end{align}
% \begin{align}\label{eq:errorsync}
% \epsilon_{part}^{(t)} :&= K_1(1+\delta_t)\big(\rho^{(t)}\gamma_y\big)^2 \frac{2(n-P)\beta_{\max}}{P(n-1)}L^2 \nonumber\\
% &\quad + K_2(1+\delta_t')\big(\rho^{(t)}\gamma_v\big)^2\frac{4(n-P)\beta_{\max}}{P(n-1)}(1+r_i^2)L^2 ,
% \end{align}
and the error from partial participation as
\begin{align}\label{eq:errorpart}
    \epsilon_{part}^{(t)} :&= (\rho^{(t)}\gamma_x)^2\bigg[\frac{L_{\Phi}^2}{2}+K_1\big(L_y^2+\frac{L_{yx}}{2}\big)+K_2\big(L_v^2+\frac{L_{vx}}{2}\big) + 2K_2c_{\gamma_v}^2\bigg]\frac{4(n-P)\beta_{\max}}{P(n-1)}(L_f^2+r_{\max}^2L_1^2) \nonumber \\
    &\quad + (\rho^{(t)}\gamma_x)^2 K_1c_{\gamma_y}^2\frac{2(n-P)\beta_{\max}}{P(n-1)}\sigma_{gh}^2
\end{align}
% \begin{align}\label{eq:errorpart}
% \epsilon_{sync}^{(t)} :&= K_1(1+\delta_t)\big(\rho^{(t)}\gamma_y\big)^2 \frac{1}{P}\sum_{i=1}^n\frac{w_i\beta_{\max}}{\|a_i^{(t)}\|^2_1} \sum_{k=0}^{\tau_i-1}\big(a^{(t,k)}_i\big)^2 \sigma^2_{g} \nonumber \\
% &\quad + K_2(1+\delta_t')\big(\rho^{(t)}\gamma_v\big)^2 \frac{2}{P}\sum_{i=1}^n\frac{w_i\beta_{\max}}{\|a_i^{(t)}\|^2_1} \sum_{k=0}^{\tau_i-1}\big(a^{(t,k)}_i\big)^2 \big(\sigma^2_{f} + r_i^2\sigma^2_{gg}\big) .
% \end{align}
\endgroup
Next, we define the error due to client drifts as 
\begingroup
\allowdisplaybreaks
\begin{align}\label{eq:errorclientdrift}
    \epsilon_{cd}^{(t)} :&= 3(\rho^{(t)}\gamma_x)\bigg[\frac{3}{2}+ 24K_2c_{\gamma_v} + \frac{\beta_{\max}}{P}\Big(\frac{17}{4} + 16K_2c_{\gamma_v}\Big) \bigg](L_1^2+r^2L_2^2)  \Big((\eta_x^2+\eta_v^2)\bar{\tau}\sigma_{M1}^2 + 2\eta_y^2\bar{\tau}\sigma_{M2}^2\Big) \nonumber \\
    % \epsilon_{cd}^{(t)} :&= (\rho^{(t)}\gamma_x)\bigg[\frac{3}{2}+ 24K_2c_{\gamma_v} + \frac{\beta_{\max}}{P}\Big(\frac{17}{4} + 16K_2c_{\gamma_v}\Big) \bigg]  \nonumber \\
    % & \quad\quad \times \bigg(3\eta_x^2\bar{\tau}(L_1^2+r^2L_2^2)\sigma_{M1}^2+3\eta_v^2\bar{\tau}L_1^2\sigma_{M1}^2 + 6\eta_y^2\bar{\tau}(L_1^2+r^2L_2^2)\sigma_{M2}^2\bigg) \nonumber \\
    &\quad + \rho^{(t)}\gamma_xK_1\bigg[4c_{\gamma_y}\frac{L_1^2}{\mu_g} + 8c_{\gamma_v}\Big(\frac{2\beta_{\max}}{P}+3\Big)L_1^2\bigg]\Big(\eta_x^2\bar{\tau}\sigma_{M1}^2 +2\eta_y^2\bar{\tau}\sigma_{M2}^2\Big) \nonumber \\
    &\quad + \rho^{(t)}\gamma_xK_2c_{\gamma_v} \frac{4L_R^2}{\mu_g}\Big((\eta_x^2+\eta_v^2)\bar{\tau}\sigma_{M1}^2 + 2\eta_y^2\bar{\tau}\sigma_{M2}^2\Big)
\end{align}
% \begin{align}\label{eq:errorclientdrift}
% &\epsilon_{cd}^{(t)} := \frac{3\rho^{(t)}\gamma_x}{2}\sum_{i=1}^n w_i \sum_{k=0}^{\tau_i - 1}\frac{a_i^{(t,k)}}{\|a_i^{(t)}\|_1} \Delta_{f,i}^{(t,k)} \nonumber \\
% & \ \ + K_1(1+\delta_t)\big(\rho^{(t)}\gamma_y\big)^2 \bigg(\frac{2n(n-P)}{P(n-1)}\sum_{i=1}^n w_i^2 \sum_{k=0}^{\tau_i-1} \frac{a_i^{(t,k)}}{\|a_i^{(t)}\|_1} + 3\sum_{i=1}^n w_i \sum_{k=0}^{\tau_i-1} \frac{a_i^{(t,k)}}{\|a_i^{(t)}\|_1} \bigg)\Delta_{g,i}^{(t,k)}\nonumber \\
% & \ \ + K_1(1+\delta_t)\rho^{(t)}\gamma_y \frac{2L_1^2}{\mu_g}  \sum_{i = 1}^nw_i \sum_{k=0}^{\tau_i-1}\frac{a^{(t,k)}_i}{\|a_i^{(t)}\|_1} \mathbb{E}\Big[\big\|x^{(t)} - x_i^{(t,k)}\big\|^2 + \big\|y^{(t)} - y_i^{(t,k)}\big\|^2\Big] \nonumber\\
% & \ \ + K_2(1+\delta_t')\big(\rho^{(t)}\gamma_v\big)^2 \bigg(\frac{2n(n-P)}{P(n-1)}\sum_{i=1}^n w_i^2 \sum_{k=0}^{\tau_i-1} \frac{a_i^{(t,k)}}{\|a_i^{(t)}\|_1} + 3\sum_{i=1}^n w_i \sum_{k=0}^{\tau_i-1} \frac{a_i^{(t,k)}}{\|a_i^{(t)}\|_1}\bigg)\Delta_{R,i}^{(t,k)} \nonumber\\
% & \ \ + K_2(1+\delta_t')\rho^{(t)}\gamma_v \frac{4L_R^2}{\mu_g}\sum_{i = 1}^n w_i \sum_{k=0}^{\tau_i-1}\frac{a^{(t,k)}_i}{\|a_i^{(t)}\|_1} \mathbb{E}\Big[\big\|x^{(t)} - x_i^{(t,k)}\big\|^2 + \big\|y^{(t)} - y_i^{(t,k)}\big\|^2 + \big\|v^{(t)} - v_i^{(t,k)}\big\|^2\Big] .
% \end{align}
\endgroup
Then we have the Descent in the Lyapunov function as 
\begin{align}\label{eq:Lyapunov2}
    \Psi(x^{(t+1)})-\Psi(x^{(t)}) \leq -\frac{\rho^{(t)}\gamma_x}{2} \mathbb{E}\Big\|\nabla \widetilde{\Phi}(x^{(t)})\Big\|^2 + \epsilon_{part}^{(t)} + \epsilon_{sync}^{(t)} + \epsilon_{cd}^{(t)}. 
\end{align}

\subsection{Proof of \Cref{th:theorem1}}
\begin{proof}
Summing \cref{eq:Lyapunov2}, we have 
\begin{align}\label{eq:finalmin1}
    \min_t \mathbb{E}\Big\|\nabla \widetilde{\Phi}(x^{(t)})\Big\|^2 
    &\leq \frac{1}{T}\sum_{t=0}^{T-1} \mathbb{E}\Big\|\nabla \widetilde{\Phi}(x^{(t)})\Big\|^2 \nonumber \\
    &\overset{(a)}{\leq} 2\times\frac{1}{T}\sum_{t=0}^{T-1} \frac{\rho^{(t)}}{\bar{\rho}}\mathbb{E}\Big\|\nabla \widetilde{\Phi}(x^{(t)})\Big\|^2 \nonumber \\
    & \leq 2\times\frac{2}{T}\Big(\frac{\Psi(x^{(0)})}{\bar{\rho}\gamma_x} - \frac{\Psi(x^{(T)})}{\bar{\rho}\gamma_x}\Big) + 2\times\frac{1}{T}\sum_{t=0}^{T-1}\frac{2}{\bar{\rho}\gamma_x}\Big(\epsilon_{part}^{(t)} + \epsilon_{sync}^{(t)} + \epsilon_{cd}^{(t)}\Big),
\end{align}
where (a) uses $\rho^{(t)} \in [\frac{1}{2}\bar{\rho}, \frac{3}{2}\bar{\rho}]$. 
For the error with partial participation in \cref{eq:errorsync}, we have 
\begingroup
\allowdisplaybreaks
\begin{align}
    \frac{1}{T}\sum_{t=0}^{T-1}\frac{2}{\bar{\rho}\gamma_x}\epsilon_{part}^{(t)} 
    \overset{(a)}{\leq} \frac{1}{T}\sum_{t=0}^{T-1}\frac{4}{\bar{\rho}\gamma_x}(\rho^{(t)}\gamma_x)^2\frac{(n-P)}{P(n-1)} M_1 
    \overset{(b)}{\leq} 6\bar{\rho}\gamma_x\frac{(n-P)}{P(n-1)} M_1
\end{align}
% \begin{align}
%     \frac{1}{T}\sum_{t=0}^{T-1}\frac{2}{\bar{\rho}\gamma_x}\epsilon_{part}^{(t)} 
%     &\leq \frac{3}{2}\cdot\frac{1}{T}\sum_{t=0}^{T-1}\bigg[2K_1(1+\delta_t)\rho^{(t)}\frac{\gamma_y^2}{\gamma_x} \frac{2(n-P)\beta_{\max}}{P(n-1)}L^2 \nonumber\\
%     &\quad + 2K_2(1+\delta_t')\rho^{(t)}\frac{\gamma_v^2}{\gamma_x} \frac{4(n-P)\beta_{\max}}{P(n-1)}(1+r_i^2)L^2 \bigg] \nonumber \\
%    &\overset{(a)}{\leq} \frac{3}{2}\cdot\frac{1}{T}\sum_{t=0}^{T-1} \bigg[K_1\rho^{(t)} \frac{\gamma_y^2}{\gamma_x} \frac{6(n-P)\beta_{\max}}{P(n-1)}L^2 \nonumber\\
%     &\quad + K_2\rho^{(t)} \frac{\gamma_v^2}{\gamma_x} \frac{12(n-P)\beta_{\max}}{P(n-1)}(1+r_i^2)L^2 \bigg] \nonumber \\
%     &\overset{(b)}{\leq} \frac{3}{2}\cdot K_{1}' c_{\gamma_y}^2 \bar{\rho} \gamma_x \frac{6(n-P)\beta_{\max}}{P(n-1)}L^2  + \frac{3}{2}\cdot K_{2}' c_{\gamma_v}^2 \bar{\rho} \gamma_x \frac{12(n-P)\beta_{\max}}{P(n-1)}(1+r_i^2)L^2 \nonumber
% \end{align}
where (a) simplifies the problem by defining 
\begin{align}
    M_1 :=& 2\big[\frac{L_{\Phi}^2}{2}+K_1\big(L_y^2+\frac{L_{yx}}{2}\big)+K_2\big(L_v^2+\frac{L_{vx}}{2}\big) + 2K_2c_{\gamma_v}^2\big]\beta_{\max}(L_f^2+r_{\max}^2L_1^2) + K_1c_{\gamma_y}^2\beta_{\max}\sigma_{gh}^2 \nonumber
\end{align}
and (b) holds due to $\rho^{(t)} \in [\frac{1}{2}\bar{\rho}, \frac{3}{2}\bar{\rho}]$. 
By the definition of $\rho^{(t)}$ in \cref{eq:localAEofq}, we can easily see that $\bar{\rho} = \mathcal{O}(\bar{\tau})$. Then we have
\begin{align}\label{eq:orderofsync}
    \frac{1}{T}\sum_{t=0}^{T-1}\frac{2}{\bar{\rho}\gamma_x}\epsilon_{part}^{(t)} 
    \leq 6M_1 \frac{n-P}{P(n-1)}\bar{\rho}\gamma_x 
    = \mathcal{O}\Big(\frac{M_1(n-P)}{n}\sqrt{\frac{\bar{\tau}}{PT}}\Big)
\end{align}
by taking $\gamma_x = \mathcal{O}\Big(\sqrt{\frac{P}{\bar{\tau}T}}\Big)$. 
\endgroup
Similarly, for the error with full synchronization in \cref{eq:errorpart}, we have 
\begingroup
\allowdisplaybreaks
\begin{align}
    \frac{1}{T}\sum_{t=0}^{T-1}\frac{2}{\bar{\rho}\gamma_x}\epsilon_{sync}^{(t)} 
    &\overset{(a)}{\leq} \frac{1}{T}\sum_{t=0}^{T-1}\frac{2}{\bar{\rho}\gamma_x}(\rho^{(t)}\gamma_x)^2\frac{n}{P}\sum_{i=1}^n\frac{w_i^2\|a_i^{(t)}\|^2_2}{\|a_i^{(t)}\|^2_1}M_2 \nonumber \\
    &\overset{(b)}{\leq} \frac{3}{T}\sum_{t=0}^{T-1}\rho^{(t)}\gamma_x\frac{\beta_{\max}}{P}\sum_{i=1}^nw_i\frac{\alpha_{\max}}{c_{a}'\bar{\tau}\alpha_{\min}}M_2 \nonumber \\
    & \leq \frac{3\alpha_{\max}\beta_{\max}}{c_a'\alpha_{\min}} \frac{\bar{\rho}}{P\bar{\tau}}\gamma_xM_2
\end{align}
% \begin{align}
%     \frac{1}{T}\sum_{t=0}^{T-1}\frac{2}{\bar{\rho}\gamma_x}\epsilon_{sync}^{(t)} 
%     &\leq \frac{3}{2}\cdot \frac{1}{T}\sum_{t=0}^{T-1}\bigg[3K_{1}'\rho^{(t)}c_{\gamma_y}\gamma_y \frac{1}{P}\sum_{i=1}^n\frac{w_i\beta_{\max}}{\|a_i^{(t)}\|^2_1} \sum_{k=0}^{\tau_i-1}\big(a^{(t,k)}_i\big)^2 \sigma^2_{g} \nonumber \\
%     &\quad + 3K_{2}'\rho^{(t)}c_{\gamma_v}\gamma_v \frac{1}{P}\sum_{i=1}^n\frac{w_i\beta_{\max}}{\|a_i^{(t)}\|^2_1} \sum_{k=0}^{\tau_i-1}\big(a^{(t,k)}_i\big)^2 \big(\sigma^2_{f} + r_i^2\sigma^2_{gg}\big)\bigg] \nonumber \\
%     &\leq \frac{3}{2}\cdot 3K_1'c_{\gamma_y}\gamma_y \frac{1}{P}\sum_{i=1}^nw_i\beta_{\max}\bigg(\frac{1}{T}\sum_{t=0}^{T-1}\frac{\|a_i^{(t)}\|^2_2}{\|a_i^{(t)}\|^2_1}\bar{\rho}\bigg) \sigma^2_{g} \nonumber \\
%     &\quad + \frac{3}{2}\cdot 3K_{2}'c_{\gamma_v}\gamma_v \frac{1}{P}\sum_{i=1}^nw_i\beta_{\max}\bigg(\frac{1}{T}\sum_{t=0}^{T-1}\frac{\|a_i^{(t)}\|^2_2}{\|a_i^{(t)}\|^2_1}\bar{\rho}\bigg) \big(\sigma^2_{f} + r_i^2\sigma^2_{gg}\big). \nonumber
% \end{align}
where (a) simplifies the problem by defining 
\begin{align}
    M_2 := 2\big[\frac{L_{\Phi}^2}{2}+K_1\big(L_y^2+\frac{L_{yx}}{2}\big)+K_2\big(L_v^2+\frac{L_{vx}}{2F}\big) + 2K_2c_{\gamma_v}^2\big](\sigma^2_{f} + r_{\max}^2\sigma^2_{gg}) + K_1c_{\gamma_y}^2\sigma_{g}^2, \nonumber
\end{align}
and (b) holds since $c_{a}'\bar{\tau}\alpha_{\min} \leq \|a_i^{(t)}\|_1 \leq c_{a}\bar{\tau}\alpha_{\max}$. 
% \begin{align}
%     M_2 := \max \bigg\{3K_1'c_{\gamma_y}^2\sigma^2_g ,\  3K_2' c_{\gamma_v}^2\big(\sigma^2_{f} + r_i^2\sigma^2_{gg}\big)\bigg\}\times \bigg(\sum_{i=1}^nw_i\frac{1}{T}\sum_{t=0}^{T-1}\frac{\|a_i^{(t)}\|^2_2}{\|a_i^{(t)}\|^2_1}\bar{\rho}\bigg)\beta_{\max}, \nonumber
% \end{align}
% \begin{align}
%     M_2 := \bar{\rho}\sum_{i=1}^n\frac{w_i\beta_{\max}}{\|a_i^{(t)}\|^2_1} \sum_{k=0}^{\tau_i-1}\big(a^{(t,k)}_i\big)^2 \max \bigg\{&3K_{11} \sigma^2_{g}, \ 3K_{21} \big(\sigma^2_{f} + r_i^2\sigma^2_{gg}\big),\nonumber \\ 
%     &\frac{32K_{21}}{\mu_g}\bigg(\frac{2L^2_{1}}{\mu_g^2} + \frac{2L^2L^2_{2}}{\mu^4_g}\bigg)\sigma^2_{g}\bigg\}, 
% \end{align}
% where we can easily see that $\Big(\sum_{i=1}^nw_i\frac{1}{T}\sum_{t=0}^{T-1}\frac{\|a_i^{(t)}\|^2_2}{\|a_i^{(t)}\|^2_1}\bar{\rho}\Big) = \mathcal{O}(1)$. 
Then we have 
\begin{align}\label{eq:orderofpart}
    \frac{1}{T}\sum_{t=0}^{T-1}\frac{2}{\bar{\rho} \gamma_x}\epsilon_{sync}^{(t)} \leq \frac{3\alpha_{\max}\beta_{\max}}{c_a'\alpha_{\min}} \frac{\bar{\rho}}{P\bar{\tau}}\gamma_xM_2 = \mathcal{O}\Big(\frac{M_2}{\sqrt{P\bar{\tau}T}}\Big)
\end{align}
by taking $\gamma_x = \mathcal{O}\Big(\sqrt{\frac{P}{\bar{\tau}T}}\Big)$. 
\endgroup
Similarly, for the error due to client drifts in \cref{eq:errorclientdrift}, we have 
\begingroup
\allowdisplaybreaks
\begin{align}\label{eq:errorclientdrift1}
\frac{1}{T}\sum_{t=0}^{T-1}\frac{2}{\bar{\rho}\gamma_x}\epsilon_{cd}^{(t)} &\leq \frac{1}{T}\sum_{t=0}^{T-1}\frac{2}{\bar{\rho}\gamma_x}{\rho^{(t)}}\gamma_x\Big((\eta_x^2+\eta_v^2)\bar{\tau}\sigma_{M1}^2 + 2\eta_y^2\bar{\tau}\sigma_{M2}^2\Big)M_3 \nonumber \\
& \leq 3\Big((\eta_x^2+\eta_v^2)\bar{\tau}\sigma_{M1}^2 + 2\eta_y^2\bar{\tau}\sigma_{M2}^2\Big)M_3
\end{align}
We define constant $M_{3}$ as
\begin{align}
    M_{3} :=& 3\bigg[\frac{3}{2}+ 24K_2c_{\gamma_v} + \frac{\beta_{\max}}{P}\Big(\frac{17}{4} + 16K_2c_{\gamma_v}\Big) \bigg](L_1^2+r^2L_2^2) \nonumber \\
    &+ K_1\bigg[4c_{\gamma_y}\frac{L_1^2}{\mu_g} + 8c_{\gamma_v}\Big(\frac{2\beta_{\max}}{P}+3\Big)L_1^2\bigg] + K_2c_{\gamma_v} \frac{4L_R^2}{\mu_g} \nonumber
\end{align}
Then we have 
\begingroup
\allowdisplaybreaks
\begin{align}\label{eq:orderofcd}
    \frac{1}{T}\sum_{t=0}^{T-1}\frac{2}{\bar{\rho}\gamma_x}\epsilon_{cd}^{(t)} \leq 3\Big((\eta_x^2+\eta_v^2)\bar{\tau}\sigma_{M1}^2 + 2\eta_y^2\bar{\tau}\sigma_{M2}^2\Big)M_3 = \mathcal{O}\bigg(\frac{M_3}{\bar{\tau}T}\bigg),
\end{align}
\endgroup
by 
% defining $\bar{\tau} := \sum_{i=1}^n w_i\tau_i \approx \sum_{i \in C^{(t)}} \widetilde{w}_i\tau_i$ and 
setting  $\eta_x = \mathcal{O}\Big(\frac{1}{\bar{\tau}\sqrt{T}}\Big)$, $\eta_y = \mathcal{O}\Big(\frac{1}{\bar{\tau}\sqrt{T}}\Big)$ and $\eta_v = \mathcal{O}\Big(\frac{1}{\bar{\tau}\sqrt{T}}\Big)$. 
\endgroup
Last but not least, for the first tern on the right-hand side of \cref{eq:finalmin1}, we have
\begin{align}\label{eq:finalmin2}
    \frac{2}{T}\Big(\frac{\Psi(x^{(0)})}{\bar{\rho}\gamma_x} - \frac{\Psi(x^{(T)})}{\bar{\rho}\gamma_x}\Big) = \mathcal{O}\Big(\sqrt{\frac{1}{P\bar{\tau}T}}\Big)
\end{align}
when we take $\gamma_x = \mathcal{O}\Big(\sqrt{\frac{P}{\bar{\tau}T}}\Big)$. Finally, by combining \cref{eq:orderofsync}, \cref{eq:orderofpart}, \cref{eq:orderofcd} and \cref{eq:finalmin2}, we have
\begin{align}\label{eq:finalorder}
    \min_t \mathbb{E}\Big\|\nabla \widetilde{\Phi}(x^{(t)})\Big\|^2 = \mathcal{O}\Big(\frac{M_1(n-P)}{n}\sqrt{\frac{\bar{\tau}}{PT}}\Big) + \mathcal{O}\Big(M_2\sqrt{\frac{1}{P\bar{\tau}T}}\Big) + \mathcal{O}\Big(\frac{M_3}{\bar{\tau}T}\Big) . 
\end{align}
Then, the first part of the proof of \Cref{th:theorem1} is complete. Next, we provide the detail of complexity analysis. 
First, for nearly full client participation, which means that $\frac{n-P}{n-1} \approx 0$, we can easily have 
\begin{align}\label{eq:finalorder1}
    \min_t \mathbb{E}\Big\|\nabla \widetilde{\Phi}(x^{(t)})\Big\|^2 =  \mathcal{O}\Big(M_2\sqrt{\frac{1}{n\bar{\tau}T}}\Big) + \mathcal{O}\Big(\frac{M_3}{\bar{\tau}T}\Big) \leq \epsilon. 
\end{align}
As a result, we can see that the per-client sample complexity $\bar{\tau}T = \mathcal{O}(n^{-1}\epsilon^{-2})$. Since the local update rounds contribute to saving communication rounds, 
% by the restriction in \cref{eq:tauandT}, 
we take $\bar{\tau} = \mathcal{O}(\frac{T}{n})$, then we have $T = \mathcal{O}(\epsilon^{-1})$.
% since $T \gg n$. 
Second, for partial client participation, we have 
\begin{align}\label{eq:finalorder2}
    \min_t \mathbb{E}\Big\|\nabla \widetilde{\Phi}(x^{(t)})\Big\|^2 = \mathcal{O}\Big(\frac{M_1(n-P)}{n}\sqrt{\frac{\bar{\tau}}{PT}}\Big) + \mathcal{O}\Big(M_2\sqrt{\frac{1}{P\bar{\tau}T}}\Big) + \mathcal{O}\Big(\frac{M_3}{\bar{\tau}T}\Big) \leq \epsilon. 
\end{align}
when participating client number $P$ is not close to full client number $n$, we can find that the local update date rounds will increase the partial participation error, which may affect the whole convergence rate. As a consequence, taking $\bar{\tau} = \mathcal{O}(\frac{n}{n-P})$ will result in the best performance. We can see that 
\begin{align}\label{eq:finalorder3}
    \min_t \mathbb{E}\Big\|\nabla \widetilde{\Phi}(x^{(t)})\Big\|^2 = \mathcal{O}\Big((M_1+M_2)\sqrt{\frac{n-P}{nPT}}\Big) + \mathcal{O}\Big(\frac{M_3}{T}\Big) \leq \epsilon.
\end{align}
% \begin{align}\label{eq:finalorder3}
%     \min_t \mathbb{E}\Big\|\nabla \widetilde{\Phi}(x^{(t)})\Big\|^2 = \mathcal{O}\Big(\frac{M_1(n-P)}{n}\sqrt{\frac{1}{PT}}\Big) + \mathcal{O}\Big(M_2\sqrt{\frac{1}{PT}}\Big) + \mathcal{O}\Big(\frac{M_4}{T}\Big) \leq \epsilon.
% \end{align}
Since $T \gg P$, we have the per-client sample complexity $\bar{\tau}T = \mathcal{O}(P^{-1}\epsilon^{-2})$ and communication rounds $T = \mathcal{O}(P^{-1}\epsilon^{-2})$. Then, we finish the proof of \Cref{th:theorem1}. 
\end{proof}

\section{Proof of \Cref{th:theorem2}}
\begin{proof}
Recall the definitions: 
\begin{align}
    &\Phi(x^{(t)}) = F(x^{(t)}, y^*(x^{(t)})) = \sum_{i=1}^n p_i f_i(x^{(t)}, y^*(x^{(t)})), \nonumber \\
    &\widetilde{\Phi}(x^{(t)}) = \widetilde{F}(x^{(t)},  \widetilde{y}^*(x^{(t)})) = \sum_{i=1}^n w_i f_i(x^{(t)}, \widetilde{y}^*(x^{(t)})). \nonumber
\end{align}
Then we have 
\begingroup
\allowdisplaybreaks
\begin{align}\label{eq:corollary1}
    \nabla &\Phi(x^{(t)}) -  \nabla \widetilde{\Phi}(x^{(t)}) \nonumber \\
    & = \sum_{i=1}^n \Big[p_i \bar{\nabla}f_i\big(x^{(t)}, y^*(x^{(t)}), v^*(x^{(t)})\big) - w_i \bar{\nabla}f_i\big(x^{(t)}, \widetilde{y}^*(x^{(t)}), \widetilde{v}^*(x^{(t)})\big)\Big] \nonumber \\
    & = \sum_{i=1}^n p_i\Big[\bar{\nabla}f_i\big(x^{(t)}, y^*(x^{(t)}), v^*(x^{(t)})\big) -  \bar{\nabla}f_i\big(x^{(t)}, \widetilde{y}^*(x^{(t)}), \widetilde{v}^*(x^{(t)})\big)\Big] \nonumber \\
    & \quad + \sum_{i=1}^n (p_i-w_i)\bar{\nabla}f_i\big(x^{(t)}, \widetilde{y}^*(x^{(t)}), \widetilde{v}^*(x^{(t)})\big) \nonumber \\
    & = \bar{\nabla}F\big(x^{(t)}, y^*(x^{(t)}), v^*(x^{(t)})\big) - \bar{\nabla}F\big(x^{(t)}, \widetilde{y}^*(x^{(t)}), \widetilde{v}^*(x^{(t)})\big) + \sum_{i=1}^n \frac{p_i-w_i}{w_i}w_i \bar{\nabla}f_i\big(x^{(t)}, \widetilde{y}^*(x^{(t)}), \widetilde{v}^*(x^{(t)})\big) .
    % & = \bar{\nabla}F\big(x^{(t)}, y^*(x^{(t)}), v^*(x^{(t)})\big) - \bar{\nabla}F\big(x^{(t)}, \widetilde{y}^*(x^{(t)}), \widetilde{v}^*(t)\big) + \sum_{i=1}^n \frac{p_i-w_i}{w_i}w_i \bar{\nabla}f_i\big(x^{(t)}, \widetilde{y}^*(t), \widetilde{v}^*(t)\big) .
    % & \overset{(a)}{\leq} 3\Big(L_1^2 + r^2L_2^2 \big\|y^*(t) - \widetilde{y}^*(t)\big\|^2\Big) + 3 L_1^2\big\|v^*(t) - \widetilde{v}^*(t)\big\|^2 \nonumber \\
    % & \quad + \sum_{i=1}^n \frac{p_i-w_i}{\sqrt{w_i}}\sqrt{w_i} \bar{\nabla}f_i\big(x^{(t)}, \widetilde{y}^*(t), \widetilde{v}^*(t)\big),
\end{align}
\endgroup
By taking the norm of \cref{eq:corollary1} and using Assumption \ref{as:Lipschitz}, we have 
\begin{align}
    \Big\| \nabla &\Phi(x^{(t)}) -  \nabla \widetilde{\Phi}(x^{(t)})\Big\|^2 \nonumber \\
    &\overset{(a)}{\leq} 6\big(L_1^2 + r^2L_2^2\big) \big\|y^*(x^{(t)}) - \widetilde{y}^*(x^{(t)})\big\|^2 + 6 L_1^2\big\|v^*(x^{(t)}) - \widetilde{v}^*(x^{(t)})\big\|^2 \nonumber \\
    & \quad + 2\bigg\|\sum_{i=1}^n \frac{p_i-w_i}{w_i}w_i \bar{\nabla}f_i\big(x^{(t)}, \widetilde{y}^*(x^{(t)}), \widetilde{v}^*(x^{(t)})\big) \bigg\|^2 \nonumber \\
    &\overset{(b)}{\leq} 6\big(L_1^2 + r^2L_2^2\big) \big\|y^*(x^{(t)}) - \widetilde{y}^*(x^{(t)})\big\|^2 + 6 L_1^2\big\|v^*(x^{(t)}) - \widetilde{v}^*(x^{(t)})\big\|^2 \nonumber \\
    & \quad + 2\bigg\| \frac{\beta_{\max}'-\beta_{\min}}{\beta_{\min}} \sum_{i=1}^n w_i \bar{\nabla}f_i\big(x^{(t)}, \widetilde{y}^*(x^{(t)}), \widetilde{v}^*(x^{(t)})\big) \bigg\|^2 \nonumber \\
    & = 6\Big(L_1^2 + r^2L_2^2 \Big)\big\|y^*(x^{(t)}) - \widetilde{y}^*(x^{(t)})\big\|^2 + 6 L_1^2\big\|v^*(x^{(t)}) - \widetilde{v}^*(x^{(t)})\big\|^2 + 2\Big(\frac{\beta_{\max}'-\beta_{\min}}{\beta_{\min}}\Big)^2\Big\|\nabla \widetilde{\Phi}(x)\Big\|^2,
\end{align}
where (a) uses Assumption \ref{as:Lipschitz}; (b) uses the setting $\frac{\beta_{\min}}{n} \leq w_i \leq \frac{\beta_{\max}}{n}$ and $\frac{\beta_{\min}'}{n} \leq p_i \leq \frac{\beta_{\max}'}{n}$ for all $i = 1,2, ... n$. Then we can see that 
\begingroup
\allowdisplaybreaks
\begin{align}\label{eq:th2}
    \min_t \Big\|\nabla \Phi(x^{(t)})\Big\|^2 &\leq \frac{1}{T}\sum_{t=0}^{T-1}\Big\|\nabla \Phi(x^{(t)})\Big\|^2 \nonumber \\
    &\leq \frac{2}{T}\sum_{t=0}^{T-1}\bigg( \Big\| \nabla \Phi(x^{(t)}) -  \nabla \widetilde{\Phi}(x^{(t)})\Big\|^2 + \Big\|\nabla \widetilde{\Phi}(x^{(t)})\Big\|^2\bigg) \nonumber \\
    & \leq \frac{2}{T}\sum_{t=0}^{T-1}\bigg[1+2\Big(\frac{\beta_{\max}'-\beta_{\min}}{\beta_{\min}}\Big)^2\bigg]\Big\|\nabla \widetilde{\Phi}(x^{(t)})\Big\|^2  \nonumber \\
    & \quad + 12\Big(L_1^2 + r^2L_2^2\Big)\frac{1}{T}\sum_{t=0}^{T-1}\big\|y^*(x^{(t)}) - \widetilde{y}^*(x^{(t)})\big\|^2 + 12 L_1^2\frac{1}{T}\sum_{t=0}^{T-1}\big\|v^*(x^{(t)}) - \widetilde{v}^*(x^{(t)})\big\|^2.
\end{align}
By taking $w_i = p_i$ in \cref{eq:th2} for all $i$, we have $y^*(x^{(t)}) = \widetilde{y}^*(x^{(t)})$ and $v^*(x^{(t)}) = \widetilde{v}^*(x^{(t)})$, which results in 
\begin{align}
    \min_t \Big\|\nabla \Phi(x^{(t)})\Big\|^2 &\leq \frac{2}{T}\bigg[1+2\Big(\frac{\beta_{\max}'-\beta_{\min}}{\beta_{\min}}\Big)^2\bigg]\sum_{t=0}^{T-1}\Big\|\nabla \widetilde{\Phi}(x^{(t)})\Big\|^2 \nonumber \\
    & = \mathcal{O}\Big(\frac{M_1(n-P)}{(n-1)}\sqrt{\frac{\bar{\tau}}{PT}}\Big) + \mathcal{O}\Big(M_2\sqrt{\frac{1}{P\bar{\tau}T}}\Big) + \mathcal{O}\Big(\frac{M_4}{\bar{\tau}T}\Big).\nonumber
\end{align}
Since we have the convergence rate of SimFBO and ShroFBO, the complexity analysis is the same as the complexity analysis \Cref{th:theorem1}. Thus, we finish the proof. 
\endgroup
\end{proof}

% \subsection{Proof of Corollary 2}
% \begin{proof}
%     \red{xxx}
% \end{proof}

\vspace*{1cm}

\end{document}